\documentclass[10pt]{article}

\usepackage[utf8]{inputenc}
\usepackage[T1]{fontenc}

\usepackage{epsf}
\usepackage{amsmath}

\allowdisplaybreaks

\usepackage[showframe=false]{geometry}
\usepackage{changepage}

\usepackage{epsfig}
\usepackage{amssymb}

\usepackage{amsthm}
\usepackage{setspace}
\usepackage{cite}
\usepackage{mcite}

\usepackage{algorithmic}  
\usepackage{algorithm}

\usepackage{shadow}
\usepackage{fancybox}
\usepackage{fancyhdr}

\usepackage{color}
\usepackage[usenames,dvipsnames,svgnames,table]{xcolor}
\newcommand{\bl}[1]{\textcolor{blue}{#1}}
\newcommand{\red}[1]{\textcolor{red}{#1}}

\definecolor{mypurple}{rgb}{.4,.0,.5}

\usepackage[hyphens]{url}

\usepackage[colorlinks=true,
            linkcolor=black,
            urlcolor=blue,
            citecolor=purple]{hyperref}

\usepackage{breakurl}

\def\w{{\bf w}}

\def\y{{\bf y}}

\def\x{{\bf x}}

\def\x{{\mathbf x}}

\def\w{{\bf w}}

\def\u{{\bf u}}

\def\x{{\bf x}}
\def\y{{\bf y}}
\def\z{{\bf z}}
\def\q{{\bf q}}
\def\m{{\bf m}}

\def\b{{\bf b}}
\def\c{{\bf c}}
\def\d{{\bf d}}
\def\f{{\bf f}}
\def\h{{\bf h}}

\def\tr{\mbox{Tr}}

\def\tr{{\rm tr}\,}

\def\cH{{\mathcal H}}
\def\cS{{\mathcal S}}
\def\cL{{\mathcal L}}

\def\be{\begin{equation}}
\def\ee{\end{equation}}
\def\ba{\left[\begin{array}}
\def\ea{\end{array}\right]}

\def\w{{\bf w}}

\def\u{{\bf u}}

\def\x{{\bf x}}
\def\y{{\bf y}}
\def\z{{\bf z}}
\def\q{{\bf q}}

\def\b{{\bf b}}
\def\c{{\bf c}}
\def\d{{\bf d}}
\def\f{{\bf f}}
\def\p{{\bf p}}

\def\1{{\bf 1}}

\def\g{{\bf g}}
\def\0{{\bf 0}}

\def\erf{\mbox{erf}}
\def\erfc{\mbox{erfc}}

\def\tanh{\mbox{tanh}}
\def\asin{\mbox{asin}}

\def\calX{{\cal X}}
\def\calY{{\cal Y}}

\def\mR{{\mathbb R}}
\def\mN{{\mathbb N}}
\def\mE{{\mathbb E}}
\def\mS{{\mathbb S}}
\def\mP{{\mathbb P}}

\def\lp{\left (}
\def\rp{\right )}

\def\w{{\bf w}}

\def\y{{\bf y}}

\def\x{{\bf x}}

\def\x{{\mathbf x}}

\def\w{{\bf w}}

\def\u{{\bf u}}

\def\x{{\bf x}}
\def\y{{\bf y}}
\def\z{{\bf z}}
\def\q{{\bf q}}

\def\b{{\bf b}}
\def\c{{\bf c}}
\def\d{{\bf d}}
\def\f{{\bf f}}
\def\h{{\bf h}}

\def\tr{\mbox{Tr}}

\def\tr{{\rm tr}\,}

\def\cH{{\cal H}}

\def\be{\begin{equation}}
\def\ee{\end{equation}}
\def\ba{\left[\begin{array}}
\def\ea{\end{array}\right]}

\def\w{{\bf w}}

\def\u{{\bf u}}

\def\x{{\bf x}}
\def\y{{\bf y}}
\def\z{{\bf z}}
\def\q{{\bf q}}

\def\b{{\bf b}}
\def\c{{\bf c}}
\def\d{{\bf d}}
\def\f{{\bf f}}
\def\p{{\bf p}}

\def\({\left (}
\def\){\right )}

\def\1{{\bf 1}}
\def\m{{\bf m}}
\def\q{{\bf q}}

\def\g{{\bf g}}
\def\0{{\bf 0}}

\def\cX{{\mathcal X}}
\def\cY{{\mathcal Y}}

\usepackage{xcolor}
\usepackage{color}

\definecolor{darkgreen}{rgb}{0, 0.4,0}

\definecolor{purplebrown}{rgb}{0.5,0.1,0.6}

\definecolor{ultclupcol}{rgb}{0.1,0.5,0.5}

\definecolor{mytrycolor}{rgb}{0.5,0.7,0.2}

\definecolor{ultclupcola}{rgb}{.5,0,.5}

\definecolor{shadebrown}{rgb}{0.1,0.1,0.9}
\definecolor{lightblue}{rgb}{0.2,0,1}

\usepackage{fancybox}
\usepackage{graphicx}
\usepackage{epstopdf}
\usepackage{epsfig}
\usepackage{wrapfig}
\usepackage{subfigure}

\usepackage{xcolor}
\usepackage{tcolorbox}
\tcbuselibrary{skins}

\newtcbox{\xmybox}{on line,
arc=7pt,
before upper={\rule[-3pt]{0pt}{10pt}},boxrule=0pt,
boxsep=0pt,left=6pt,right=6pt,top=0pt,bottom=0pt,enhanced, coltext=blue, colback=white!10!yellow}

\newtcbox{\xmyboxa}{on line,
arc=7pt,
before upper={\rule[-3pt]{0pt}{10pt}},boxrule=0pt,
boxsep=0pt,left=6pt,right=6pt,top=0pt,bottom=0pt,enhanced, colback=white!10!yellow}

\newtcbox{\xmyboxb}{on line,
arc=7pt,
before upper={\rule[-3pt]{0pt}{10pt}},boxrule=1pt,colframe=darkgreen!100!blue,
boxsep=0pt,left=6pt,right=6pt,top=0pt,bottom=0pt,enhanced, colback=white!10!yellow}

\newtcbox{\xmyboxc}{on line,
arc=7pt,
before upper={\rule[-3pt]{0pt}{10pt}},boxrule=.7pt,colframe=blue!100!blue,
boxsep=0pt,left=6pt,right=6pt,top=0pt,bottom=0pt,enhanced, coltext=blue, colback=white!10!yellow}

\newtcbox{\xmytboxa}{on line,
arc=7pt,
before upper={\rule[-3pt]{0pt}{10pt}},boxrule=.0pt,colframe=pink!50!yellow,
boxsep=0pt,left=6pt,right=6pt,top=0pt,bottom=0pt,enhanced, coltext=white, colback=blue!40!red}

\newtcbox{\xmytboxb}{on line,
arc=7pt,
before upper={\rule[-3pt]{0pt}{10pt}},boxrule=.0pt,colframe=pink!50!yellow,
boxsep=0pt,left=6pt,right=6pt,top=0pt,bottom=0pt,enhanced, coltext=white, colback=white!40!green}

\makeatletter
\newcommand\subsubsubsection{\@startsection{paragraph}{4}{\z@}{-2.5ex\@plus -1ex \@minus -.25ex}{1.25ex \@plus .25ex}{\normalfont\normalsize\bfseries}}
\newcommand\subsubsubsubsection{\@startsection{subparagraph}{5}{\z@}{-2.5ex\@plus -1ex \@minus -.25ex}{1.25ex \@plus .25ex}{\normalfont\normalsize\bfseries}}
\makeatother

\newtheorem{theorem}{Theorem}

\newtheorem{corollary}{Corollary}

\newtheorem{lemma}{Lemma}

\begin{document}

\begin{singlespace}

\title{Exact capacity of the \emph{wide} hidden layer treelike neural networks with generic activations
}
\author{
\textsc{Mihailo Stojnic
\footnote{e-mail: {\tt flatoyer@gmail.com}} }}
\date{}
\maketitle

\centerline{{\bf Abstract}} \vspace*{0.1in}

 Recent progress in studying \emph{treelike committee machines} (TCM) neural networks (NN) in \cite{Stojnictcmspnncaprdt23,Stojnictcmspnncapliftedrdt23} showed that the Random Duality Theory (RDT) and its a \emph{partially lifted}(pl RDT) variant are powerful tools that can be used for very precise networks capacity analysis. The initial considerations from \cite{Stojnictcmspnncaprdt23,Stojnictcmspnncapliftedrdt23}, related to the famous  \emph{sign} activations, were then extended  to more general activations in \cite{Stojnictcmspnncapdiffactrdt23}, where particularly elegant results were obtained for \emph{any} even number of the \emph{quadratically} and \emph{ReLU} activated hidden layer neurons, $d$. While the results of  \cite{Stojnictcmspnncapdiffactrdt23} are in principle applicable to any type of activations a significant amount of numerical work is often needed to make them practically usable. Here, we consider \emph{wide} hidden layer networks and uncover that certain aspects of such difficulties miraculously disappear. In particular, we employ recently developed \emph{fully lifted} (fl) RDT to characterize the \emph{wide} ($d\rightarrow \infty$) TCM nets capacity. We obtain explicit, closed form, capacity characterizations for a very generic class of the hidden layer activations. While the utilized approach significantly lowers the amount of the needed numerical evaluations, the ultimate fl RDT usefulness and success still require a solid portion of the residual numerical work. To get the concrete capacity values, we take four very famous activations examples: \emph{\textbf{ReLU}}, \textbf{\emph{quadratic}}, \textbf{\emph{erf}}, and \textbf{\emph{tanh}}. After successfully conducting all the residual numerical work for all of them, we uncover that the whole lifting mechanism exhibits a remarkably rapid convergence with the relative improvements no better than $\sim 0.1\%$ happening already on the 3-rd level of lifting. As a convenient bonus, we also uncover that the capacity characterizations obtained on the first and second level of lifting precisely match those obtained through the statistical physics replica theory methods in \cite{ZavPeh21} for the generic and in \cite{BalMalZech19} for the ReLU activations.

\vspace*{0.25in} \noindent {\bf Index Terms: Wide TCM neural networks; Capacity; Fully lifted random duality theory; ReLU, quadratic, erf, tanh}.

\end{singlespace}

\section{Introduction}
\label{sec:intro}

Development of machine learning (ML) and neural networks (NN) concepts experienced a rapid progress over the last 15-20 years. Larger than ever need for efficient handling and interpretation of huge data sets stimulated the invention of many fundamental algorithmic NN breakthroughs. Such an algorithmic progress necessarily  dictated advancement of the accompanying analytical/theoretical justification methodologies. Along the paths of both algorithmic and theoretical advancements, many new concepts have been developed and many old ones have been revisited and brought to practical usability. We are here interested in, possibly, the most important of them all, namely the so-called, network's \emph{memory capacity} (see, e.g., \cite{Schlafli,Cover65,Winder,Winder61,Wendel62,Cameron60,Joseph60,Gar88,Ven86,BalVen87}). As is well known, studying the network capacity has two key components: \textbf{\emph{(i)}} the \emph{theoretical} one which attempts to provide the engineering practitioners with the mathematically precise description of the ultimate underlying network architecture usefulness; and \textbf{\emph{(ii)}} the \emph{practical/algorithmic} one which attempts to provide the users with the concrete (hopefully efficiently implementable) computational methodologies to indeed utilize the network architectures to their ultimate potential. Here we continue the trend established in the recent literature, focus on the first one, and provide a strong theoretical progress on several important capacity related questions. To be able to properly present the technical contributions and to adequately contextualize them within the relevant prior work, we find it convenient to first introduce the needed mathematical formalisms that best describe the underlying NN models.

\subsection{Architecture of the \emph{wide} hidden layer generically activated NNs}
\label{sec:model}

We start with a generic architecture description of the multi-input single-output feed-forward neural networks with $L-2$ hidden layers and $d_i$ ($i\in\{1,2,\dots,L\}$) nodes (neurons) in the $i$-th layer. For the notational convenience, we add two artificial layers, indexed by  $i=1$ and $i=L$, which correspond to the network's input and output, respectively. Although they are artificial, we  refer to them as networks layers to ensure the consistency of the overall indexation. The network operates through the  specification of the activation functions vectors, $\f^{(i)}(\cdot)=[\f_1^{(i)}(\cdot),\f_2^{(i)}(\cdot),\dots,\f_{d_{i+1}}^{(i)}(\cdot)]^T$, where each activation function $\f_{j}^{(i)}(\cdot):\mR^{d_i}\rightarrow \mR$ describes how the $j$-th neuron in layer $i$ operates. One effectively has that the outputs of the nodes in layer $i$ are taken as the inputs of the nodes in layer $i+1$ and then transformed into the new outputs (of the nodes in layer $i+1$) via $\f_{j}^{(i)}(\cdot)$ and matrix of weights, $W^{(i)}\in\mR^{d_{i}\times d_{i+1}}$. Setting $\d=[d_1,d_2,\dots,d_L]$ (with $d_1=n$ and $d_L=1$) and denoting by $\b^{(i)}\in\mR^{d_{i+1}}$ the so-called activation thresholds vectors and by $\x^{(i)}\in\mR^{d_i}$ and $\x^{(i+1)}\in\mR^{d_{i+1}}$ the inputs and outputs of the neurons in layer $i$, one
 has the following:
\begin{center}
     	\tcbset{beamer,nobeforeafter,lower separated=false, fonttitle=\bfseries, coltext=black,
		interior style={top color=yellow!20!white, bottom color=yellow!60!white},title style={left color=black, right color=red!50!blue!60!white},
		before=,after=\hfill,fonttitle=\bfseries,equal height group=AT}
     	\tcbset{beamer,nobeforeafter,lower separated=false, fonttitle=\bfseries, coltext=black,
		interior style={top color=yellow!20!white, bottom color=cyan!60!white},title style={left color=black, right color=red!50!blue!60!white},
		before=,after=\hfill,fonttitle=\bfseries,equal height group=AT, width=.9\linewidth}
\begin{tcolorbox}[title=Mathematical formalism of NN with architecture $A(\d\text{; }\f^{(i)})$:]
\vspace{-.03in}$\mbox{\textbf{input:}} \triangleq \x^{(1)}\quad \longrightarrow$ \hfill
\tcbox[coltext=black,colback=white!65!red!30!yellow,interior style={top color=cyan!20!white, bottom color=yellow!60!white},nobeforeafter,box align=base]{$ \x^{(i+1)}=\f^{(i)}(W^{(i)}\x^{(i)}-\b^{(i)}) $ }\hfill $\longrightarrow \quad \mbox{\textbf{output:}} \triangleq \x^{(L+1)}$.
\vspace{-.2in}\begin{equation}\label{eq:model0}
\vspace{-.2in}\end{equation}
\vspace{-.4in}\end{tcolorbox}
\end{center}
Clearly, the network's architecture, $A(\d;\f^{(i)})$,  is fully specified by the vectors $\d$ and $\f^{(i)}$ (we may on occasion write $A(\d;\f)$ instead of $A(\d;\f^{(i)})$ when $\f^{(i)}$ are identical).

As is the case for the single neurons, one of the most fundamental features of any neural net is their ability to properly store/memorize a large amount of data. To see how the above formalism works in that regard, one can assume, for example, the existence of $m$ data pairs $(\x^{(0,k)},\y^{(0,k)})$, $k\in\{1,2,\dots,m\}$, with $\x^{(0,k)}\in \mR^{n}$ being the $n$-dimensional data vectors and $\y^{(0,k)}\in\mR$ being their corresponding labels. Finding weight matrices $W^{(i)}$  such that
\begin{equation}\label{eq:model3}
\x^{(1)}=\x^{(0,k)}\quad \Longrightarrow \quad \x^{(L+1)}=\y^{(0,k)} \qquad \forall k,
\end{equation}
is then sufficient to relate given data vectors, $\x^{(0,k)}$, to their associated labels, $\y^{(0,k)}$. The \emph{memory capacity, $C(A(\d,\f^{(i)}))$} of the given architecture $A(\d,\f^{(i)})$ is then defined as the largest sample size, $m$, such that (\ref{eq:model3}) holds for any collection of data pairs $(\x^{(0,k)},\y^{(0,k)})$, $k\in\{1,2,\dots,m\}$ with certain prescribed properties. Given the relevance of the capacity in understanding the limits of NNs functioning, determining its precise theoretical value (together with the development of the corresponding computationally efficient algorithmic procedures that achieve it), is of utmost importance. We below provide a collection of results that enable full capacity characterization for many well known architectures $A(\d;\f^{(i)})$.

We state below several technical and structural assumptions that facilitate the presentation. As many of them are aligned with the ones discussed in \cite{Stojnictcmspnncaprdt23,Stojnictcmspnncapliftedrdt23,Stojnictcmspnncapdiffactrdt23}, we avoid unnecessarily repeating them and instead opt for  briefly recalling on the most important ones. We, however, do place a particular emphasis on those that substantially differ and are of particular relevance for the considerations of main interest in this paper.

\subsection{Assumptions related to architecture and data}
\label{sec:assumpt}

As the assumptions that we rely on are rather common and prevalent in the existing literature, we avoid discussing them in deep details and instead focus on precisely stating them.

\noindent \textbf{\emph{Architecture assumptions:}} We consider \textbf{\emph{generic}} zero-threshold activation functions in the hidden layer with the following properties: \emph{\textbf{(i)}} We take $L=3$, $\b^{(i)}=0$, $W^{(1)}=I_{n\times n}$, and $W^{(3)}=\w^T$, where $\w\in\mR^{d_2\times 1}$ (in other words, $W^{(3)}$ is a $d_2$-dimensional row vector). \emph{\textbf{(ii)}}  We set $d\triangleq d_2$ and $\delta\triangleq \delta_1=\frac{d_1}{d_2}=\frac{n}{d}$. While we ultimately consider $d\rightarrow\infty$ scenario, many of the presented mathematical concepts hold for any even $d$. Along the same lines, whenever $d$ is not emphasized as infinite, we assume that it is any given even natural number.
\textbf{\emph{(iii)}} \emph{Identity} neuronal functions, $\f^{(1)}(\x^{(1)})=\x^{(1)}$, are considered in the first (artificial) layer. In the hidden layer, we take \emph{generic} zero-threshold activations $\f_j^{(2)}(\cdot): \mR^{d_2}\rightarrow R$ with $\f_j^{(2)}(\cdot)=\f_k^{(2)}(\cdot)$ for any $j\neq k$. As typical for the TCMs, at the output, the zero-threshold sign activations, $\f^{(3)}(W^{(3)}\x^{(3)}-\b^{(3)})=\mbox{sign} \left ( W^{(3)}\x^{(3)}\right )$, are assumed. We denote this architecture by $A(\d;[\f^{(1)},\f^{(2)},\f^{(3)}])=A(\d;[I,\f^{(2)},\mbox{sign}])
\triangleq A(\d;\f^{(2)})$. \textbf{\emph{(iv)}}
 The matrix of the hidden layer weights, $W^{(2)}$, can be full or with a particular sparse structure. Both options have been considered previously throughout the literature. For example, a particular sparse structuring with the support of $W^{(2)}$'s $j$-th row, $\mbox{supp}\lp W_{j,:}^{(2)}\rp$, satisfying
$\mbox{supp}\lp W_{j,:}^{(2)}\rp=\cS^{(j)}$, with $\cS^{(j)}\triangleq\{(j-1)\delta+1,(j-1)\delta+2,\dots,j\delta\}$, makes the above architecture correspond to the \emph{treelike committee machines} (TCM) which are of our main interest in this paper (alternatively, full $W^{(2)}$ makes the architecture correspond to the \emph{fully connected committee machines} (FCM)).

\noindent \textbf{\emph{Data assumptions:}} \emph{\textbf{(i)}} We assume the typical \emph{binary} labeling $\y_i^{(0,k)}\in\{-1,1\}$ (in addition to being the most standard type of labeling, it is also nicely complemented by the \emph{sign} neuronal choice at the network's output). \emph{\textbf{(ii)}} Data sets that are inseparable are not allowed. This, for example, means that indistinguishable or contradictory pairs, such as $(\x^{(0,k)},\y^{(0,k)})$ and $(\x^{(0,k)},-\y^{(0,k)})$, can not appear. \emph{\textbf{(iii)}} We focus on statistical datasets and particularly focus on $\x^{(0,k)}$ being comprised of iid standard normals. This follows the trend established in the classical single perceptron references (see, e.g., \cite{DTbern,Gar88,StojnicGardGen13,Cover65,Winder,Winder61,Wendel62}) and is expected to allow for a fairly universal statistical treatment. To provide universal capacity upper bounds, it is, however, perfectly sufficient to consider any type of data set (including even nonstatistical ones).


\subsection{Contextualization within relevant prior work}
\label{sec:priorcont}

Given that the problems of our interest are well known and have been studied for longer than a half of century, the underlying related literature is rather vast. As surveying all of it here is infeasible (and way better suited for general review papers), we below focus on the results, which, in our view, are the most relevant and closest to our own.

The first memory capacity considerations started with the spherical perceptrons in the early sixties of the last century. A close connection to several fundamental integral geometry problems was observed and the early results were directly related to some of the classical geometrical/probabilistic works (see, e.g., \cite{Schlafli,Cover65,Wendel62,Joseph60}). The most famous of them establishes that $C(A(1;\mbox{sign}))\rightarrow 2n$ as $n\rightarrow\infty$, which effectively means that, in a large dimensional statistical context, the spherical \emph{sign} perceptron  capacity basically \emph{doubles the dimension} of the data ambient space, $n$. After being initially proven as a remarkable combinatorial geometry fact in \cite{Schlafli,Cover65,Winder,Winder61,Wendel62,Cameron60,Joseph60}, decades later, it was reproved in various forms in a host of different fields ranging from machine learning and pattern recognition to probability and information theory (see, e.g., \cite{BalVen87,Ven86,DT,StojnicISIT2010binary,DonTan09Univ,DTbern,Gar88,StojnicGardGen13,StojnicGardSphErr13}).

\underline{\emph{Sign perceptrons hidden layer activations:}} Extending the single neuron capacity results to the corresponding network ones is not easy. Particularly scarce are results related to TCM NNs. On the other hand, a bit more is known about the FCM ones, but a direct connection between the two is not very apparent. For example, the FCM capacities trivially upper-bound the corresponding TCM ones, but a way more appropriate appears to be viewing the TCM capacities as roughly the FCM ones divided by $d$. An overwhelming majority of the known capacity results are of the \textbf{\emph{scaling}} type and indicate an unavoidable relation to the total number of the network weights, $w=\sum_{i=1}^{L-1} d_id_{i+1}$. For example, the famous VC-dimension provides the upper-bounding scaling $O(w\log(w))$. For the NNs with 1-hidden layer, one has  $w=d_1d_{2}+d_2=(n+1)d$ for the FCM and $w=d_1+d_{2}=n+d$ for the TCM, which, for large $d_i$'s and huge $n$, gives the ``\emph{division by $d$}'' FCM -- TCM capacity relation. When it comes to the corresponding lower bounding, \cite{Baum88} argued that the capacity of a shallow 3-layer net scales as $O(nd)$. A bit stronger version was obtained recently in \cite{Vershynin20}, where, for the networks with more than three layers, the capacity was shown to scale at least as $O(w)$.

\underline{\emph{Scaling versus non-scaling:}}  Obtaining the capacity results that are precise and of the \textbf{\emph{non-scaling}} type is a much harder challenge. Given the simplicity and elegance of the single spherical perceptron capacity, this initially might seem as a bit surprising. However, after recognizing that several decades of a strong effort did not produce much of the analytical progress, the level of difficulty becomes clearer. In fact, after \cite{MitchDurb89} provided simple multi-perceptron extension of the combinatorial considerations of \cite{Cover65,Winder,Winder61,Wendel62,Cameron60,Joseph60}, the progresses completely stalled  until the very recent appearance of \cite{Stojnictcmspnncapliftedrdt23,Stojnictcmspnncaprdt23,Stojnictcmspnncapdiffactrdt23}.  Relying on the Random duality theory (RDT), \cite{Stojnictcmspnncaprdt23}  developed a generic framework for the analysis of TCM networks and obtained strong capacity upper bounds for any given (odd) number of the hidden layer neurons. \cite{Stojnictcmspnncapliftedrdt23} went then a bit further, utilized a partially lifted (pl) RDT variant and substantially lowered the bounds proven in \cite{Stojnictcmspnncaprdt23}.

\underline{\emph{Different hidden layer activations:}} The above discussion highlighted the analytical hardness as the first of the two key obstacles one typically faces when trying to transition from the single to multi neuron architectures.  The second one is of the algorithmic type and relates to the \emph{sign} perceptrons being noncontinuous functions. Namely, it is usually not very easy to design computationally provably efficient network training strategies for such objects. This amplifies the need for potentially less simple but easier to use activation factions. Since the discreteness is usually perceived as the main source of the  sign perceptrons analytical and algorithmic hardness, consideration of continuous activations positions itself as a promising alternative. Many of such activations have already found their place in various NN architectures, include the ReLU, quadratic, tanh, erf and so on. As things are, at least, algorithmically a bit more favorable when such activations are in place, a little bit more is known regarding their capacities as well. For example, \cite{Yama93} first suggested for deep nets, and  \cite{GBHuang03} later on proved for 4-layer nets, that the capacity is at least $O(w)$ for sigmoids. \cite{ZBHRV17,HardrtMa16} then showed similar results for ReLU while keeping an additional number of nodes restriction which was later on removed in \cite{YunSuJad19} for both tanh and ReLU. After \cite{Stojnictcmspnncapliftedrdt23,Stojnictcmspnncaprdt23} first introduced a generic framework for precise  \textbf{\emph{non-scaling}} capacity analysis of the \emph{sign} hidden layer activations, \cite{Stojnictcmspnncapdiffactrdt23} extended the framework so that it can handle various different activations. It then particularly focused on three types of activations, linear, quadratic, and ReLU. For the linear, \cite{Stojnictcmspnncapdiffactrdt23} determined the exact value of the capacity and showed that, for \emph{any} width (the number of the hidden layer neurons),  it matches the corresponding one of the single spherical perceptron. On the other hand, for the quadratic and ReLU, it obtained the plain and the pl RDT capacity upper bounds. All considerations from \cite{Stojnictcmspnncapliftedrdt23,Stojnictcmspnncaprdt23,Stojnictcmspnncapdiffactrdt23} were done for the networks with any given (even) number of hidden layer neurons, $d$.

\underline{\emph{Statistical physics -- Replica methods:}} Notorious difficulty that moving from the
\emph{scaling} (say, of the $O(\cdot)$ type) to the \emph{precise} capacity descriptions  imposes was already recognized in the early eighties of the last century. As at that time there were no available powerful mathematical techniques that could handle such a move, statistical physics replica methods positioned themselves as an excellent (and basically only known) alternative. Despite their analytical non-rigorousness, they produce expectedly  \emph{precise} final results. The foundational replica concepts within the analysis of the NN capacities were laid out in the pioneering works \cite{GarDer88,Gar88}, where various single perceptron forms were discussed. Utilizing those concepts, a few
years later, \cite{EKTVZ92,BHS92} studied the same TCM architecture that we study here  (as well as the related FCM one). Considering the sign activations, and the so-called replica symmetry formalism, they obtained  the capacity predictions for any number of the neurons in the hidden layer, $d$. They also established the corresponding large $d$ scaling behavior. Each of these results was proven as mathematically rigorous capacity upper bound, the first one in \cite{Stojnictcmspnncaprdt23} and the second one in \cite{Stojnictcmspnncapliftedrdt23}. \cite{EKTVZ92,BHS92}, however, went a step further and showed that their replica symmetry large $d$ predictions violate the mathematically rigorous ones of \cite{MitchDurb89}. To remedy such a contradiction, \cite{EKTVZ92,BHS92} then proceeded by studying the first level of replica symmetry breaking (rsb) and showing that it lowers the capacity. For both the committee and the so-called parity machines (PM), relevant large $d$ scaling rsb considerations were presented in \cite{MonZech95} as well (for earlier PM related replica considerations see also, e.g., \cite{BarKan91,BHK90}). On the other hand, for the FCM architecture, \cite{Urban97,XiongKwonOh97} obtained  a bit later the large $d$ scaling that matches the upper-bounding one of \cite{MitchDurb89}. Of particular relevance, however, to our work are two very recent lines of work \cite{BalMalZech19,ZavPeh21} where the \emph{wide} hidden layer TCM architectures were considered. In \cite{BalMalZech19} the 1rsb capacity predictions for the ReLU activations were obtained. A bit later, \cite{ZavPeh21} moved things further and obtained analogous 1rsb predictions for several other activations, including the linear, ReLU, quadratic, erf, and tanh among others.

\underline{\emph{Practical achievability:}} We also mention another line of work that attracted a strong interest over the last several years. It relates to the design and analysis of the efficient network training algorithms that could potentially approach the capacity. After it was empirically observed that the simple gradient based methods perform reasonably well in this context while requiring only the so-called mild over-parametrization (moderately larger number of free parameters, $w$, compared to the data set size, $m$), a lot of effort was put in providing theoretical justifications of such a phenomenon. More on a solid progress made in these directions in recent years can be found in, e.g., \cite{DuZhaiPoc18,GeWangZhao19,ADHLW19,JiTel19,LiLiang18,OymSol19,RuoyuSun19,SongYang19,ZCZG18}. While these results are mostly oriented towards the FCMs, they are also extendable to the TCMs as well.

\subsubsection{Contributions}
\label{sec:contrib}

The \emph{precise} analytical characterization of the so-called $n$-\emph{scaled} memory capacity of the TCM NNs  with generic neuronal activations, $\f^{(2)}$, in the \emph{wide} hidden layer is the main object of our study. In other words, we are interested in determining
\begin{equation}\label{eq:model4}
\alpha_c(\infty;\f^{(2)})\triangleq \lim_{d\rightarrow\infty} \alpha_c(d;\f^{(2)})\triangleq \lim_{d\rightarrow\infty} \lim_{n\rightarrow\infty} \frac{C(A([n,d,1];\f^{(2)}))}{n},
\end{equation}
where we often for brevity instead of $\alpha_c(\infty;\f^{(2)})$ write just $\alpha_c(\infty)$. The generic framework for the analysis of sign activations, established in\cite{Stojnictcmspnncaprdt23,Stojnictcmspnncapliftedrdt23} (relying on the RDT and pl RDt principles), was extended to various different activations in
\cite{Stojnictcmspnncapdiffactrdt23}. While the framework works for any given even number of hidden layer neurons, it also heavily relies on the underlying numerical evaluations. We, here consider \emph{wide} hidden layers ($d\rightarrow\infty$), and uncover that a significant portions of numerical difficulties miraculously disappears.

\underline{\emph{A summary of the key results:}}  \emph{\textbf{(i)}}  Relying on \cite{Stojnicflrdt23}, we establish \emph{fully lifted} (fl) RDT, based framework for the capacity analysis of \emph{wide} hidden layer TCM NNs with generic activations. \emph{\textbf{(ii)}} For several particular activations, ReLU, quadratic, erf, and tanh, we then obtain explicit $d\rightarrow\infty$ closed form capacity characterizations. \emph{\textbf{(iii)}} While we uncover that, compared to \cite{Stojnictcmspnncapdiffactrdt23}, a significant amount of required numerical work disappears, to have the obtained capacity characterizations become practically operational, a substantial amount of the residual numerical work is still needed. For all the considered activations, we successfully conduct the needed numerical evaluations and obtain the concrete capacity values as well.  \emph{\textbf{(iv)}} We observe a rather rapid convergence of the lifting mechanism with the relative improvement no better than $\sim 0.1\%$ achieved already on the third level of full lifting. Moreover, we uncover that the first and second level of lifting characterizations match the ones obtained through the replica symmetry and 1rsb analyses in \cite{ZavPeh21} for a spectrum of different activations and in \cite{BalMalZech19} for the ReLU activation. Some of the concrete estimates are also shown in Table \ref{tab:tab0}. The results from Table \ref{tab:tab0} are additionally complemented with their visual representations in  Figure \ref{fig:fig0}. Both, the strong effect/benefit of the fl lifted RDT as well as its a rapid convergence are rather obvious.

\begin{table}[h]
\caption{\emph{Wide} treelike net capacity -- lifting mechanism progress for different activations;   $d\rightarrow\infty$; $n\rightarrow\infty$}\vspace{.1in}
\centering
\def\arraystretch{1.2}
{\small
\begin{tabular}{c||c|c|c|c}\hline\hline
  \hspace{.2in}   \textbf{Memory capacity} \hspace{.2in}   &  \multicolumn{4}{c}{\textbf{Activation}}                   \\
    \cline{2-5}
 \hspace{.0in} ($r$ -- lifting level)  &  ReLU & Quad & $\erf$ & $\tanh$  \\  \hline\hline
\hspace{.0in}$\alpha_c^{(1,f)}(\infty)$  ($1$-sfl RDT) \hspace{.0in}   & $\quad \mathbf{2.9339}\quad $ & $\quad \mathbf{4}\quad $ &  $\quad \mathbf{ 2.4514}\quad $  &  $\quad \mathbf{ 2.3556}\quad $    \\ \hline\hline
\hspace{.0in}$\alpha_c^{(2,f)}(\infty)$  ($2$-sfl RDT)    & $\mathbf{2.6643}$  & $\mathbf{3.3750}$   &  $\mathbf{2.3750}$ & $\mathbf{2.3063}$   \\ \hline\hline
\hspace{.0in}$\alpha_c^{(3,f)}(\infty)$  ($3$-sfl RDT)    & $\mathbf{2.6534}$  & $\mathbf{3.3669}$ &  $\mathbf{2.3744}$ & $\mathbf{2.3058}$   \\ \hline\hline
  \end{tabular}
  }
\label{tab:tab0}
\end{table}

\begin{figure}[h]
\centering
\centerline{\includegraphics[width=1\linewidth]{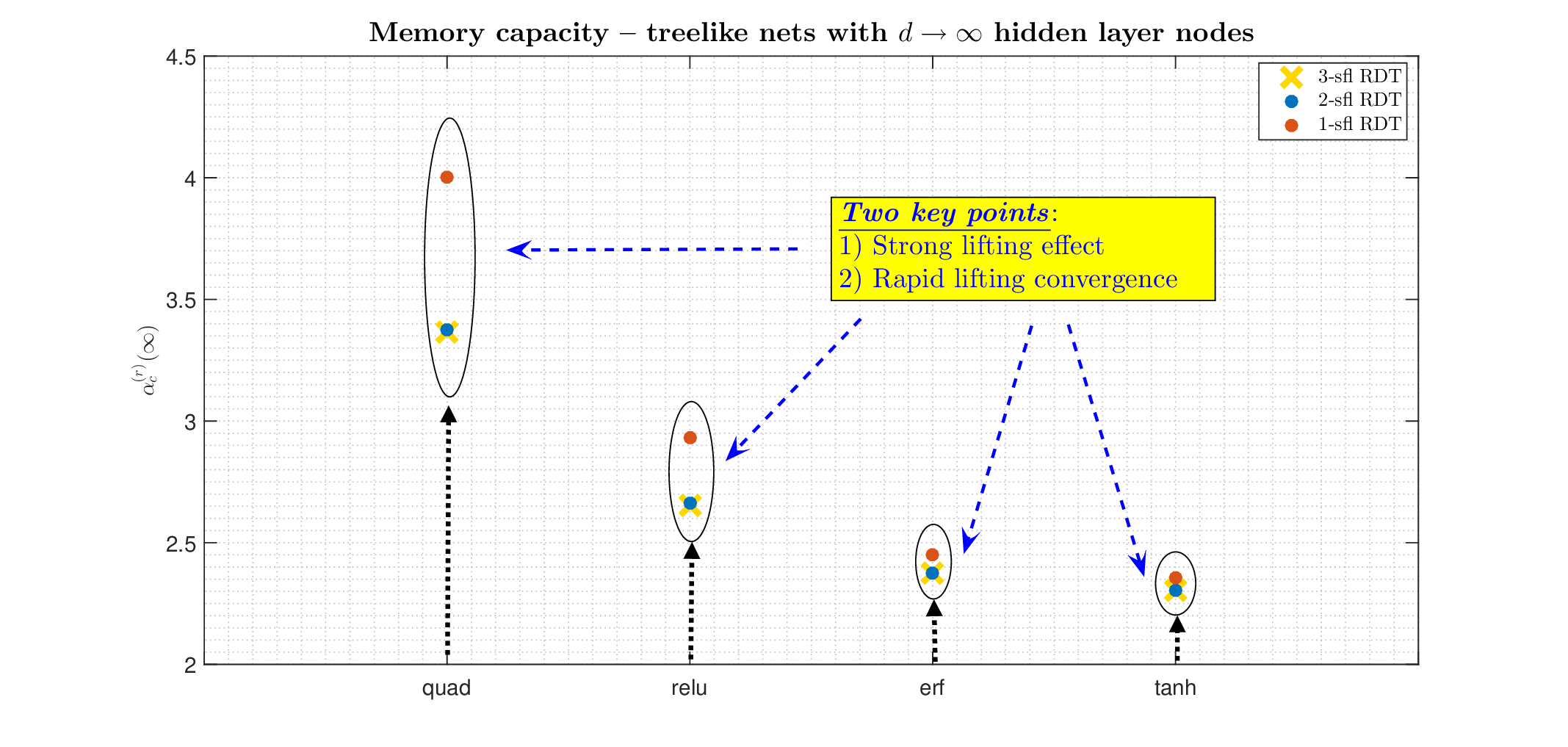}}
\caption{Memory capacity -- treelike nets with $d\rightarrow\infty$ hidden layer neurons; different activations}
\label{fig:fig0}
\end{figure}

\section{Mathematical formalism of network functioning}
\label{sec:amathdataproc}

To ensure the easiness of writing and overall presentation,  we set $W\triangleq W^{(2)}$, recall on $W^{(1)}=I$ and $W^{(3)}=\w^T$, and write for any $k\in\{1,2,\dots,m\}$
\begin{equation}\label{eq:ta1}
\x^{(1)}=\x^{(0,k)}  \quad \Longrightarrow \quad  \x^{(2)}=\f^{(1)}(W^{(1)}\x^{(1)})=\f^{(1)}(\x^{(1)})=\x^{(1)}=\x^{(0,k)},
\end{equation}
and
\begin{equation}\label{eq:ta2}
\x^{(2)}=\x^{(0,k)}  \quad \Longrightarrow \quad  \x^{(3)}=\f^{(2)}(W^{(2)}\x^{(2)})=\f^{(2)}(W\x^{(0,k)}),
\end{equation}
and
\begin{equation}\label{eq:ta3}
 \x^{(3)}=\f^{(2)}(W\x^{(0,k)}) \quad \Longrightarrow \quad  \x^{(4)}=\f^{(3)}(W^{(3)}\x^{(3)})=\f^{(3)}(\w^T\x^{(3)})
 =\mbox{sign}(\w^T  \f^{(2)}(W\x^{(0,k)})).
\end{equation}
After connecting beginning in (\ref{eq:ta1}) and end in (\ref{eq:ta3}), one obtains the following explicit relation between the network's input and output
\begin{equation}\label{eq:ta4}
\x^{(1)}=\x^{(0,k)} \quad \Longrightarrow \quad  \x^{(4)} =\mbox{sign}(\w^T  \f^{(2)}(W\x^{(0,k)})).
\end{equation}
The necessary and sufficient condition for network to operate properly, is then the following
\begin{equation}\label{eq:ta5}
\y^{(0,k)}=\mbox{sign}(\w^T  \f^{(2)}(W\x^{(0,k)})).
\end{equation}
Moreover, after setting
\begin{equation}\label{eq:ta6}
\y\triangleq \begin{bmatrix}
               \y^{(0,1)} & \y^{(0,2)} & \dots & \y^{(0,m)}
            \end{bmatrix}^T   \qquad \mbox{and} \qquad X\triangleq \begin{bmatrix}
               \x^{(0,1)} & \x^{(0,2)} & \dots & \x^{(0,m)}
            \end{bmatrix}^T,
\end{equation}
it is not that difficult to see that (\ref{eq:ta5}) can be rewritten in generic matrix form as
\begin{equation}\label{eq:ta7}
\left (\exists W\in\mR^{d\times n}| \|\y^T-\mbox{sign}(\w^T  \f^{(2)}(WX^T))\|_2=0 \right )  \quad \Longleftrightarrow \quad \left ( \left ( X,\y \right ) \mbox{is memorized} \right ),
\end{equation}
 where the $k$-th data pair, $\left ( \x^{(0,k)},\y^{(0,k)} \right )$, are the $k$-th row of $m\times n$ matrix $X$ and the $k$-th element of $m\times 1$ column vector $\y$. One then has the following (alternative to (\ref{eq:ta7})):
\begin{center}
 	\tcbset{beamer,sidebyside,lower separated=false, fonttitle=\bfseries, coltext=black,
		interior style={top color=yellow!20!white, bottom color=cyan!60!white},title style={left color=black, right color=red!50!blue!60!white},
		width=(\linewidth-4pt)/4,before=,after=\hfill,fonttitle=\bfseries,equal height group=AT}
 	\begin{tcolorbox}[title=Algebraic memorization characterization of the $\f^{(2)}$-activated hidden layer TCMs:,sidebyside,width=1\linewidth]
\vspace{-.15in}\begin{eqnarray}\label{eq:ta8}
\hspace{-.3in} 0=\xi\triangleq \min_{W,Q} & & \hspace{-.1in}\|\y-\mbox{sign}(\f^{(2)}(Q) \w)\|_2 \nonumber \\
\hspace{-.5in} \mbox{subject to} & & \hspace{-.1in}XW^T=Q
\end{eqnarray}
 \tcblower
 \hspace{-.2in}$\Longleftrightarrow$ \hspace{.1in} Data set $\left (X,\y \right )$ is properly memorized.
 		\vspace{-.0in}
 	\end{tcolorbox}
\end{center}
As it will soon be clear, the above is the key mathematical problem on the path towards characterizing the network memorization capabilities. We find it useful for what follows to slightly reformulate the above optimization. To that end we first observe that it can be rewritten as
\begin{eqnarray}\label{eq:ta9}
\xi=\min_{Z,Q} & & \|\y-\mbox{sign}(\f^{(2)}(Q) \w)\|_2 \nonumber \\
 \mbox{subject to} & & XZ=Q,
\end{eqnarray}
where cosmetic change, $Z=W^T$, is only for facilitating overall writing. We continue the trend of \cite{Stojnictcmspnncapliftedrdt23,Stojnictcmspnncaprdt23,Stojnictcmspnncapdiffactrdt23} and consider the TCM architecture, with a sparse $Z$ that ensures the \emph{treelike} architecture. In other words, we consider $Z$ with the only nonzero elements in the $j$-th column are in the rows from the following set $\cS^{(j)}\triangleq\{(j-1)\delta+1,(j-1)\delta+2,\dots,j\delta\}$. Keeping in mind this $Z$ specialization and the insensitiveness of (\ref{eq:ta9}) with respect to the
 $Z$ or $Q$ scalings, one can further write
\begin{eqnarray}\label{eq:ta9a}
\xi=\min_{Z,Q} & & \|\1-\mbox{sign}(\f^{(2)}(Q) \w)\|_2 \nonumber \\
 \mbox{subject to} & & XZ=Q\nonumber \\
  & & \|Z_{:,j}\|_2=1 \nonumber \\
  & & \mbox{supp}(Z_{:,j})=\cS^{(j)}, 1\leq j\leq d,
\end{eqnarray}
with $Z_{:,j}$ being the $j$-th column of $Z$ and $\|Z_{:,j}\|_2$ its Euclidean norm. A bit of additional cosmetic rewriting of (\ref{eq:ta9a}) gives
\begin{eqnarray}\label{eq:ta9aa0}
\xi=\min_{\z^{(j)},Q} & & \|\1-\mbox{sign}(\f^{(2)}(Q) \w)\|_2 \nonumber \\
 \mbox{subject to} & & X^{(j)}\z^{(j)}=Q_{:,j}, 1\leq j\leq d, \nonumber \\
  & & \|\z^{(j)}\|_2=1 \nonumber \\
  & & \z^{(j)}\in\mR^{\delta}, Q\in\mR^{m\times d},
\end{eqnarray}
where $X^{(j)}=X_{:,\cS^{(j)}}\in\mR^{m\times \delta}$. As emphasized earlier, we consider statistical data sets with elements of $X$ being iid standard normals, which, due to rotational symmetry and without loss of generality, allows to take all the elements of $\y$ equal to 1. For short, in what follows, we therefore take $\y=\1$ (where $\1$ is the all ones column vector of appropriate dimensions). Keeping this in mind, the following lemma, taken from \cite{Stojnictcmspnncapdiffactrdt23}, provides a precise resulting optimization representation of the network memorization property. It is in fact structurally a mirrored analogue to Lemma 1 from \cite{Stojnictcmspnncaprdt23}.
\begin{lemma}(\cite{Stojnictcmspnncapdiffactrdt23} Algebraic optimization representation)
Assume a 1-hidden layer TCM with architecture $A([n,d,1];\f^{(2)})$. Any given data set $\left (\x^{(0,k)},1\right )_{k=1:m}$ can not be properly memorized by the network if
\begin{equation}\label{eq:ta10}
  f_{rp}(X)>0,
\end{equation}
where
\begin{equation}\label{eq:ta11}
f_{rp}(X)\triangleq \frac{1}{\sqrt{n}}\min_{\|\z^{(j)}\|_2=1,Q} \max_{\Lambda\in\mR^{m\times d}} \|\1-\mbox{\emph{sign}}(\f^{(2)}(Q) \w)\|_2 +\sum_{j=1}^{d}(\Lambda_{:,j})^TX^{(j)}\z^{(j)} -\tr(\Lambda^TQ),
\end{equation}
and $X\triangleq \begin{bmatrix}
               \x^{(0,1)} & \x^{(0,2)} & \dots & \x^{(0,m)}
            \end{bmatrix}^T$.
  \label{lemma:lemma1}
\end{lemma}
\begin{proof}
Immediate consequence of Lemma 1 in \cite{Stojnictcmspnncaprdt23}.
\end{proof}

Throughout the rest of the paper we consider mathematically the most challenging, so-called \emph{linear}, regime with
\begin{equation}\label{eq:ta14}
  \alpha\triangleq \lim_{n\rightarrow\infty}\frac{m}{n}.
\end{equation}

\subsection{Connecting network functioning and (partially reciprocal) \emph{free energy}}
\label{secrfpsfe}

\emph{Free energies} are well known and almost unavoidable objects in many statistical physics considerations. Here, we view them as purely mathematical objects. We below give a bit of a preview related to the importance of these objects in studying neural networks capacities. To introduce their mathematical representation relevant to the problems of our interest here, we start by defining the following, so-called, \emph{bilinearly summed Hamiltonian}
\begin{equation}
\cH_{sq}(G^{(j_w)})= \sum_{j_w=1}^{d} \lp\y^{(j_w)}\rp^TG^{(j_w)}\x^{(j_w)},\label{eq:ham1}
\end{equation}
and its, so to say, \emph{partially reciprocal} partition function
\begin{equation}
Z_{sq}(\beta,G^{(j_w)})=\sum_{\x^{(j_w)}\in\cX^{(j_w)}} \lp \sum_{\y^{(j_w)}\in\cY^{(j_w)}}e^{\beta\cH_{sq}(G^{(j_w)})}\rp^{-1}.  \label{eq:partfun}
\end{equation}
It is important to note right here at the beginning that for $d=1$ one gets the usual \emph{bilinear} Hamiltonian. Indexing and summing over $j_w$ will effectively correspond to the extension of the width of the hidden layer -- the key component of the considered network architecture. For starters, we take $\cX^{(j_w)}$ and $\cY^{(j_w)}$ in (\ref{eq:partfun}) as general sets. Later on, throughout the presentation, we make the necessary specializations. One also notes that the inner summation factors in in a reciprocal fashion making the overall partition function appear as seemingly different from the typically seen counterparts in the statistical physics literature. The thermodynamic limit of the average of such ``\emph{partially reciprocal}'' free energy is then
\begin{eqnarray}
f_{sq}(\beta) & = & \lim_{n\rightarrow\infty}\frac{\mE_{G^{(j_w)}}\log{(Z_{sq}(\beta,G^{(j_w)})})}{\beta \sqrt{n}} \nonumber \\
& = &\lim_{n\rightarrow\infty} \frac{\mE_{G^{(j_w)}}\log\lp \sum_{\x^{(j_w)}\in\cX^{(j_w)}} \lp \sum_{\y^{(j_w)}\in\cY^{(j_w)}}e^{\beta\cH_{sq}(G^{(j_w)})}\rp^{-1}\rp}{\beta \sqrt{n}} \nonumber \\
& = &\lim_{n\rightarrow\infty} \frac{\mE_{G^{(j_w)}}\log\lp \sum_{\x^{(j_w)}\in\cX^{(j_w)}} \lp \sum_{\y^{(j_w)}\in\cY^{(j_w)}}e^{\beta \sum_{j_w=1}^{d} \lp\y^{(j_w)}\rp^TG^{(j_w)}\x^{(j_w)})}\rp^{-1}\rp}{\beta \sqrt{n}}.\label{eq:logpartfunsqrt}
\end{eqnarray}
The so-called ``zero-temperature'' ($T\rightarrow 0$ or  $\beta=\frac{1}{T}\rightarrow\infty$) regime gives the ground state special case
\begin{eqnarray}
f_{sq}(\infty)   \triangleq    \lim_{\beta\rightarrow\infty}f_{sq}(\beta) & = &
\lim_{\beta,n\rightarrow\infty}\frac{\mE_{G^{(j_w)}}\log{(Z_{sq}(\beta,G^{(j_w)})})}{\beta \sqrt{n}} \nonumber \\
& = &
 \lim_{n\rightarrow\infty}\frac{\mE_{G^{(j_w)}} \max_{\x^{(j_w)}\in\cX^{(j_w)}}  -  \max_{\y^{(j_w)}\in\cY^{(j_w)}} \sum_{j_w=1}^{d}  \lp\y^{(j_w)}\rp^TG^{(j_w)}\x^{(j_w)}}{\sqrt{n}} \nonumber \\
& = & - \lim_{n\rightarrow\infty}\frac{\mE_{G^{(j_w)}} \min_{\x^{(j_w)}\in\cX^{(j_w)}}  \max_{\y^{(j_w)}\in\cY^{(j_w)}} \sum_{j_w=1}^{d} \lp\y^{(j_w)}\rp^TG^{(j_w)}\x^{(j_w)}}{\sqrt{n}}. \nonumber \\
  \label{eq:limlogpartfunsqrta0}
\end{eqnarray}
One can then also trivially rewrite (\ref{eq:limlogpartfunsqrta0}) as
\begin{eqnarray}
-f_{sq}(\infty)
& = &  \lim_{n\rightarrow\infty}\frac{\mE_{G^{(j_w)}} \min_{\x^{(j_w)}\in\cX^{(j_w)}}  \max_{\y^{(j_w)}\in\cY^{(j_w)}} \sum_{j_w=1}^{d} \lp\y^{(j_w)}\rp^TG^{(j_w)}\x^{(j_w)}}{\sqrt{n}}.
  \label{eq:limlogpartfunsqrt}
\end{eqnarray}
Connecting $j$ to $j_w$, $X^{(j)}$ to $G^{(j_w)}$,  $\Lambda_{:,j}$ to $\y^{(j_w)}$, and $\z^{(j)}$ to $\x^{(j_w)}$ hints that the network functioning may indeed be related to the above introduced partially reciprocal free energy. However, quite a few of additional preliminaries need to be addressed before one can definitely make such a connection. One particular thing from (\ref{eq:limlogpartfunsqrt}) should be kept in mind for later on though. Namely, while the connection between $f_{rp}(X)$ and $f_{sq}(\infty)$ seems apparent, direct studying of $f_{sq}(\infty)$ might not be very easy. Somewhat paradoxically, we may instead find it as more beneficial to first study $f_{sq}(\beta)$ for a general $\beta$  and then to eventually specialize the obtained results to the above mentioned ground state, $\beta\rightarrow\infty$, regime. As the analysis will be rather heavy, we may also, in the interest of easing the exposition, on occasion neglect some terms which are of no importance in the ground state considerations.

The above hints at the potential role that the free energies can play in studying the networks capacities. Still, to be able to fully exploit such a potential connection, one would need to develop a mechanism to study the free energies themselves. Such a mechanism is precisely what we discuss next.

\section{Network memorization through the prism of sfl RDT}
\label{sec:randlincons}

To ensure a smooth and  proper connection between the network functioning and the sfl RDT, we find it convenient to first revisit some of the sfl RDT basics.

\subsection{Basics of sfl RDT}
\label{sec:basicssflrdt}

To make writing easier and initial considerations a bit smoother, we, for time being, take $d=1$, which allows to ignore all $j_w$ indexing. One of the key observations that enables pretty much everything that follows is then the recognition that a slightly changed variant of the (partially reciprocal) free energy from (\ref{eq:logpartfunsqrt}),
\begin{eqnarray}
f_{sq}(\beta) & = &\lim_{n\rightarrow\infty} \frac{\mE_G\log\lp \sum_{\x\in\cX} \lp \sum_{\y\in\cY}e^{\beta\y^TG\x)}\rp^{s}\rp}{\beta \sqrt{n}},\label{eq:hmsfl1}
\end{eqnarray}
is a function of \emph{bilinearly indexed} (bli) random process $\y^TG\x$. Precisely this very recognition is exactly that allows us to establish a direct connection between $f_{sq}(\beta)$ and the bli related results from \cite{Stojnicsflgscompyx23,Stojnicnflgscompyx23,Stojnicflrdt23}. To that end, we closely follow \cite{Stojnichopflrdt23,Stojnicbinperflrdt23} and start with several technical definitions. We consider $r\in\mN$, $k\in\{1,2,\dots,r+1\}$, real scalars $s$, $x$, and $y$ that satisfy $s^2=1$, $x>0$, and $y>0$, sets $\cX\subseteq \mR^n$ and $\cY\subseteq \mR^m$, function $f_S(\cdot):\mR^{n+m}\rightarrow R$, vectors $\p=[\p_0,\p_1,\dots,\p_{r+1}]$, $\q=[\q_0,\q_1,\dots,\q_{r+1}]$, and $\c=[\c_0,\c_1,\dots,\c_{r+1}]$ such that
 \begin{eqnarray}\label{eq:hmsfl2}
1=\p_0\geq \p_1\geq \p_2\geq \dots \geq \p_r\geq \p_{r+1} & = & 0 \nonumber \\
1=\q_0\geq \q_1\geq \q_2\geq \dots \geq \q_r\geq \q_{r+1} & = &  0,
 \end{eqnarray}
and $\c_0=1$, $\c_{r+1}=0$. For ${\mathcal U}_k\triangleq [u^{(4,k)},\u^{(2,k)},\h^{(k)}]$  such that the components of  $u^{(4,k)}\in\mR$, $\u^{(2,k)}\in\mR^m$, and $\h^{(k)}\in\mR^n$ are i.i.d. standard normals, we set
  \begin{eqnarray}\label{eq:fl4}
\psi_{S,\infty}(f_{S},\calX,\calY,\p,\q,\c,x,y,s)  =
 \mE_{G,{\mathcal U}_{r+1}} \frac{1}{n\c_r} \log
\lp \mE_{{\mathcal U}_{r}} \lp \dots \lp \mE_{{\mathcal U}_3}\lp\lp\mE_{{\mathcal U}_2} \lp \lp Z_{S,\infty}\rp^{\c_2}\rp\rp^{\frac{\c_3}{\c_2}}\rp\rp^{\frac{\c_4}{\c_3}} \dots \rp^{\frac{\c_{r}}{\c_{r-1}}}\rp, \nonumber \\
 \end{eqnarray}
where
\begin{eqnarray}\label{eq:fl5}
Z_{S,\infty} & \triangleq & e^{D_{0,S,\infty}} \nonumber \\
 D_{0,S,\infty} & \triangleq  & \max_{\x\in\cX,\|\x\|_2=x} s \max_{\y\in\cY,\|\y\|_2=y}
 \lp \sqrt{n} f_{S}
+\sqrt{n}  y    \lp\sum_{k=2}^{r+1}c_k\h^{(k)}\rp^T\x
+ \sqrt{n} x \y^T\lp\sum_{k=2}^{r+1}b_k\u^{(2,k)}\rp \rp \nonumber  \\
 b_k & \triangleq & b_k(\p,\q)=\sqrt{\p_{k-1}-\p_k} \nonumber \\
c_k & \triangleq & c_k(\p,\q)=\sqrt{\q_{k-1}-\q_k}.
 \end{eqnarray}
Equipped with all the above, we are in position to recall on the following theorem -- clearly, one of the sfl RDT's  fundamental components.
\begin{theorem} \cite{Stojnicflrdt23}
\label{thm:thm1}  Consider large $n$ context with  $\alpha=\lim_{n\rightarrow\infty} \frac{m}{n}$, remaining constant as  $n$ grows. Let the elements of  $G\in\mR^{m\times n}$
 be i.i.d. standard normals and let $\cX\subseteq \mR^n$ and $\cY\subseteq \mR^m$ be two given sets. Assume the complete sfl RDT frame from \cite{Stojnicsflgscompyx23} and consider a given function $f(\y):\mR^m\rightarrow \mR$. Set
\begin{align}\label{eq:thmsflrdt2eq1}
   \psi_{rp} & \triangleq - \max_{\x\in\cX} s \max_{\y\in\cY} \lp f(\x,\y)+\y^TG\x \rp
   \qquad  \mbox{(\bl{\textbf{random primal}})} \nonumber \\
   \psi_{rd}(\p,\q,\c,x,y,s) & \triangleq    \frac{x^2y^2}{2}    \sum_{k=2}^{r+1}\Bigg(\Bigg.
   \p_{k-1}\q_{k-1}
   -\p_{k}\q_{k}
  \Bigg.\Bigg)
\c_k \nonumber \\
& \quad
  - \psi_{S,\infty}(f(\x,\y),\calX,\calY,\p,\q,\c,x,y,s) \hspace{.03in} \mbox{(\bl{\textbf{fl random dual}})}. \nonumber \\
 \end{align}
Let $\hat{\p_0}\rightarrow 1$, $\hat{\q_0}\rightarrow 1$, and $\hat{\c_0}\rightarrow 1$, $\hat{\p}_{r+1}=\hat{\q}_{r+1}=\hat{\c}_{r+1}=0$, and let the non-fixed parts of $\hat{\p}\triangleq \hat{\p}(x,y)$, $\hat{\q}\triangleq \hat{\q}(x,y)$, and  $\hat{\c}\triangleq \hat{\c}(x,y)$ be the solutions of the following system
\begin{eqnarray}\label{eq:thmsflrdt2eq2}
   \frac{d \psi_{rd}(\p,\q,\c,x,y,s)}{d\p} =  0,\quad
   \frac{d \psi_{rd}(\p,\q,\c,x,y,s)}{d\q} =  0,\quad
   \frac{d \psi_{rd}(\p,\q,\c,x,y,s)}{d\c} =  0.
 \end{eqnarray}
 Then,
\begin{eqnarray}\label{eq:thmsflrdt2eq3}
    \lim_{n\rightarrow\infty} \frac{\mE_G  \psi_{rp}}{\sqrt{n}}
  & = &
\min_{x>0} \max_{y>0} \lim_{n\rightarrow\infty} \psi_{rd}(\hat{\p}(x,y),\hat{\q}(x,y),\hat{\c}(x,y),x,y,s) \qquad \mbox{(\bl{\textbf{strong sfl random duality}})},\nonumber \\
 \end{eqnarray}
where $\psi_{S,\infty}(\cdot)$ is as in (\ref{eq:fl4})-(\ref{eq:fl5}).
 \end{theorem}
\begin{proof}
The $s=-1$ scenario follows immediately from the corresponding one proven in \cite{Stojnicflrdt23} after a trivial change $f(\x)\rightarrow f(\x,\y)$. On the other hand, the $s=1$ scenario, follows after trivial line-by-line repetitions of the arguments from Section 3 of \cite{Stojnicflrdt23} with $s=-1$ replaced by $s=1$ and $f(\x)$ replaced by $f(\x,\y)$.
 \end{proof}

The above theorem has a very generic character and holds for any given sets $\cX$ and $\cY$. Of our interest here is its a $d$-fold summing extension and specialization to particular sets analogues $\cX^{(j_w)}$ and $\cY^{(j_w)}$, $j_w\in\{1,2,\dots,d\}$. The following corollary contains such a fully operational extension.
\begin{corollary}
\label{cor:cor1}  Assume the setup of Theorem \ref{thm:thm1}. For $d>0$ and $j_w\in\{1,2,\dots,d\}$, let the elements of  $G^{(j_w)}\in\mR^{m\times n}$
 be i.i.d. standard normals. Let $\cX^{(j_w)}\subseteq \mR^n$, $\cY^{(j_w)}\subseteq \mR^m$ be $2d$ given sets and let ${\mathcal Q}\in\mR^l$ be another given set. Also, let $f(\x^{(j_w)},\y^{(j_w)},Q):\mR^{d(n+m)+l}\rightarrow \mR$ and ${\mathcal U}_k^{(j_w)}\triangleq [u^{(4,k,j_w)},\u^{(2,k,j_w)},\h^{(k,j_w)}]$  with the components of  $u^{(4,k,j_w)}\in\mR$, $\u^{(2,k,j_w)}\in\mR^m$, and $\h^{(k,j_w)}\in\mR^n$ being i.i.d. standard normals. Set
\begin{multline}\label{eq:cor1fl4}
\psi_{S,\infty}^{(d)}(f_{S},\calX^{(j_w)},\calY^{(j_w)},\p,\q,\c,x^{(j_w)},y^{(j_w)},s)  =
 \mE_{G^{(j_w)},{\mathcal U}_{r+1}^{(j_w)}} \frac{1}{n\c_r} \\
  \times \log
\lp \mE_{{\mathcal U}_{r}^{(j_w)}} \lp \dots \lp \mE_{{\mathcal U}_3^{(j_w)}}\lp\lp\mE_{{\mathcal U}_2^{(j_w)}} \lp \lp Z_{S,\infty}^{(d)}\rp^{\c_2}\rp\rp^{\frac{\c_3}{\c_2}}\rp\rp^{\frac{\c_4}{\c_3}} \dots \rp^{\frac{\c_{r}}{\c_{r-1}}}\rp,
 \end{multline}
where
\begin{eqnarray}\label{eq:cor1fl5}
Z_{S,\infty}^{(d)} & \triangleq & e^{D_{0,S,\infty}^{(d)}} \nonumber \\
 D_{0,S,\infty}^{(d)} & \triangleq  & \max_{\x^{(j_w)}\in\cX^{(j_w)},\|\x^{(j_w)}\|_2=x^{(j_w)},Q\in {\mathcal Q}} s \max_{\y^{(j_w)}\in\cY^{(j_w)},\|\y^{(j_w)}\|_2=y^{(j_w)}}
 \lp \sqrt{n} f_{S}
+{\mathcal D}_1
+{\mathcal D}_2
 \rp \nonumber  \\
{\mathcal D}_1  & \triangleq & \sqrt{n} \sum_{j_w=1}^{d} y^{(j_w)}    \lp\sum_{k=2}^{r+1}c_k\h^{(k,j_w)}\rp^T\x^{(j_w)} \nonumber \\
{\mathcal D}_2  & \triangleq  & \sqrt{n} \sum_{j_w=1}^{d}x^{(j_w)}  \lp\y^{(j_w)}\rp^T\lp\sum_{k=2}^{r+1}b_k\u^{(2,k,j_w)}\rp   \nonumber \\
 b_k & \triangleq & b_k(\p,\q)=\sqrt{\p_{k-1}-\p_k} \nonumber \\
c_k & \triangleq & c_k(\p,\q)=\sqrt{\q_{k-1}-\q_k}.
 \end{eqnarray}
  Set
\begin{equation}\label{eq:thmsflrdt2eq1a0a1}
   \psi_{rp}^{(d)}  \triangleq - \max_{\x^{(j_w)}\in\cX^{(j_w)},Q\in{\mathcal Q}} s \max_{\y^{(j_w)}\in\cY^{(j_w)}} \lp f(\x^{(j_w)},\y^{(j_w)},Q) + \sum_{j_w=1}^{d}\lp\y^{(j_w)}\rp^TG^{(j_w)}\x^{(j_w)} \rp,
   \quad  \mbox{(\bl{\textbf{random primal}})}
   \end{equation}
and
\begin{multline}\label{eq:thmsflrdt2eq1a0a1rd}
   \psi_{rd}^{(d)}(\p,\q,\c,x^{(j_w)},y^{(j_w)},s)  \triangleq   \frac{\sum_{j_w=1}^{d} \lp x^{(j_w)}\rp^2\lp y^{(j_w)}\rp^2}{2}    \sum_{k=2}^{r+1}\Bigg(\Bigg.
   \p_{k-1}\q_{k-1}
   -\p_{k}\q_{k}
  \Bigg.\Bigg)
\c_k \\
 - \psi_{S,\infty}^{(d)}(f(\x^{(j_w)},\y^{(j_w)},Q) ,\calX^{(j_w)},\calY^{(j_w)},\p,\q,\c,x^{(j_w)},y^{(j_w)},s) \qquad \mbox{(\bl{\textbf{fl random dual}})}. \\
 \end{multline}
Let $\hat{\p_0}\rightarrow 1$, $\hat{\q_0}\rightarrow 1$, and $\hat{\c_0}\rightarrow 1$, $\hat{\p}_{r+1}=\hat{\q}_{r+1}=\hat{\c}_{r+1}=0$, and let the non-fixed parts of $\hat{\p}\triangleq \hat{\p}(x^{(j_w)},y^{(j_w)})$, $\hat{\q}\triangleq \hat{\q}(x^{(j_w)},y^{(j_w)})$, and  $\hat{\c}\triangleq \hat{\c}(x^{(j_w)},y^{(j_w)})$ be the solutions of the following system
\begin{eqnarray}\label{eq:thmsflrdt2eq2a0}
   \frac{d  \psi_{rd}^{(d)}(\p,\q,\c,x^{(j_w)},y^{(j_w)},s)}{d\p} =  0,\quad
   \frac{d  \psi_{rd}^{(d)}(\p,\q,\c,x^{(j_w)},y^{(j_w)},s)}{d\q} =  0,\quad
   \frac{d  \psi_{rd}^{(d)}(\p,\q,\c,x^{(j_w)},y^{(j_w)},s)}{d\c} =  0.
 \end{eqnarray}
 Then,
\begin{eqnarray}\label{eq:thmsflrdt2eq3a0}
    \lim_{n\rightarrow\infty} \frac{\mE_{G^{(j_w)}}  \psi_{rp}^{(d)}}{\sqrt{n}}
  & = &
\min_{x^{(j_w)}>0} \max_{y^{(j_w)}>0} \lim_{n\rightarrow\infty} \psi_{rd}(\hat{\p}(x^{(j_w)},y^{(j_w)}),\hat{\q}(x^{(j_w)},y^{(j_w)}),\hat{\c}(x^{(j_w)},y^{(j_w)}),x^{(j_w)},y^{(j_w)},s) \nonumber \\
& & \hspace{2in} \mbox{(\bl{\textbf{strong sfl random duality}})}. \nonumber \\
 \end{eqnarray}
  \end{corollary}
\begin{proof}
Follows as a $d$-fold application of Theorem \ref{thm:thm1} in the same manner Theorem 1 of \cite{Stojnictcmspnncaprdt23} follows in the plain RDT case.
 \end{proof}

\subsection{Fitting memorization into the sfl RDT machinery}
\label{sec:fitmemsflrdt}

The following corollary enables fitting the analysis of the network memorization problem into the sfl RDT framework.

\begin{corollary}
\label{cor:cor2}  Assume the setup of Theorem \ref{thm:thm1}, Corollary \ref{cor:cor1}, and Lemma \ref{lemma:lemma1}. Set
\begin{multline}\label{eq:cor2fl4}
\psi_{S,\infty}^{(d)}(f_{S},\lp \mS^n \rp^{\times d},\mS^{dm},\p,\q,\c,1,\lambda^{(j_w)},s)  =
 \mE_{X^{(j_w)},{\mathcal U}_{r+1}^{(j_w)}} \frac{1}{n\c_r} \\
  \times \log
\lp \mE_{{\mathcal U}_{r}^{(j_w)}} \lp \dots \lp \mE_{{\mathcal U}_3^{(j_w)}}\lp\lp\mE_{{\mathcal U}_2^{(j_w)}} \lp \lp Z_{S,\infty}^{(d)}\rp^{\c_2}\rp\rp^{\frac{\c_3}{\c_2}}\rp\rp^{\frac{\c_4}{\c_3}} \dots \rp^{\frac{\c_{r}}{\c_{r-1}}}\rp,
 \end{multline}
where
\begin{eqnarray}\label{eq:cor2fl5}
Z_{S,\infty}^{(d)} & \triangleq & e^{D_{0,S,\infty}^{(d)}} \nonumber \\
 D_{0,S,\infty}^{(d)} & \triangleq  & \max_{\|\z^{(j_w)}\|_2=1,\phi(Q)=1} s \max_{\|\Lambda_{:,j_w}\|_2=\lambda^{(j_w)},\|\lambda^{(j_w)}\|_2=1}
 \lp \sqrt{n} f_{S}
+{\mathcal D}_1
+{\mathcal D}_2
 \rp \nonumber  \\
{\mathcal D}_1  & \triangleq & \sqrt{n} \sum_{j_w=1}^{d} \lambda^{(j_w)}    \lp\sum_{k=2}^{r+1}c_k\h^{(k,j_w)}\rp^T\z^{(j_w)} \nonumber \\
{\mathcal D}_2  & \triangleq  & \sqrt{n} \sum_{j_w=1}^{d}  \lp\Lambda_{:,j_w}\rp^T\lp\sum_{k=2}^{r+1}b_k\u^{(2,k,j_w)}\rp   \nonumber \\
\phi(Q) & \triangleq &  \|\1-\mbox{\emph{sign}}(\f^{(2)}(Q) \w)\|_2\nonumber \\
 b_k & \triangleq & b_k(\p,\q)=\sqrt{\p_{k-1}-\p_k} \nonumber \\
c_k & \triangleq & c_k(\p,\q)=\sqrt{\q_{k-1}-\q_k}.
 \end{eqnarray}
  Set
\begin{eqnarray}\label{eq:cor2sflrdt2eq1a0a1}
   \psi_{rp}^{(d)} & \triangleq  & - \max_{\|\z^{(j_w)}\|_2=1,\phi(Q)=1} s \max_{\|\Lambda_{:,j_w}\|_2=\lambda^{(j_w)},\|\lambda^{(j_w)}\|_2=1} \lp f(\Lambda,Q) + \sum_{j_w=1}^{d}\lp\Lambda_{:,j_w}\rp^TX^{(j_w)}\z^{(j_w)} \rp,\nonumber \\
 & & \hspace{4in}   \mbox{(\bl{\textbf{random primal}})}
   \end{eqnarray}
and
\begin{eqnarray}\label{eq:cor2sflrdt2eq1a0a1rd}
   \psi_{rd}^{(d)}(\p,\q,\c,1,\lambda^{(j_w)},s) & \triangleq &  \frac{1}{2}    \sum_{k=2}^{r+1}\Bigg(\Bigg.
   \p_{k-1}\q_{k-1}
   -\p_{k}\q_{k}
  \Bigg.\Bigg)
\c_k \nonumber \\
& &  - \psi_{S,\infty}^{(d)}(f(\Lambda,Q),\lp \mS^n \rp^{\times d},\mS^{dm},\p,\q,\c,1,\lambda^{(j_w)},s) \qquad \mbox{(\bl{\textbf{fl random dual}})}. \nonumber \\
 \end{eqnarray}
Let $\hat{\p_0}\rightarrow 1$, $\hat{\q_0}\rightarrow 1$, and $\hat{\c_0}\rightarrow 1$, $\hat{\p}_{r+1}=\hat{\q}_{r+1}=\hat{\c}_{r+1}=0$, and let the non-fixed parts of $\hat{\p}\triangleq \hat{\p}(1,\lambda^{(j_w)})$, $\hat{\q}\triangleq \hat{\q}(1,\lambda^{(j_w)})$, and  $\hat{\c}\triangleq \hat{\c}(1,\lambda^{(j_w)})$ be the solutions of the following system
\begin{eqnarray}\label{eq:cor2sflrdt2eq2a0}
   \frac{d  \psi_{rd}^{(d)}(\p,\q,\c,1,\lambda^{(j_w)},s)}{d\p} =  0,\quad
   \frac{d  \psi_{rd}^{(d)}(\p,\q,\c,1,\lambda^{(j_w)},s)}{d\q} =  0,\quad
   \frac{d  \psi_{rd}^{(d)}(\p,\q,\c,1,\lambda^{(j_w)},s)}{d\c} =  0.
 \end{eqnarray}
 Then,
\begin{eqnarray}\label{eq:cor2sflrdt2eq3a0}
    \lim_{n\rightarrow\infty} \frac{\mE_{X^{(j_w)}}  \psi_{rp}^{(d)}}{\sqrt{n}}
  & = &
\max_{\lambda^{(j_w)}>0} \lim_{n\rightarrow\infty} \psi_{rd}^{(d)}(\hat{\p}(1,\lambda^{(j_w)}),\hat{\q}(1,\lambda^{(j_w)}),\hat{\c}(1,\lambda^{(j_w)}),1,\lambda^{(j_w)},s) \nonumber \\
& & \hspace{2in} \mbox{(\bl{\textbf{strong sfl random duality}})}. \nonumber \\
 \end{eqnarray}
Moreover, for $s=-1$ and $f(\Lambda,Q)=-\tr(\Lambda^TQ)$ one has
\begin{eqnarray}\label{eq:cor2sflrdt2eq3a0a0}
   \psi_{rp}^{(d)} &  \triangleq &  \min_{\|\z^{(j_w)}\|_2=1,\phi(Q)=1} \max_{\|\Lambda_{:,j_w}\|_2=\lambda^{(j_w)},\|\lambda^{(j_w)}\|_2=1} \lp -\tr(\Lambda^TQ) + \sum_{j_w=1}^{d}\lp\Lambda_{:,j_w}\rp^TX^{(j_w)}\z^{(j_w)} \rp, \nonumber \\
    & & \hspace{4in}   \mbox{(\bl{\textbf{random primal}})}
   \end{eqnarray}
 and
\begin{eqnarray}\label{eq:cor2sflrdt2eq3a0a1}
\phi_0 =    \lim_{n\rightarrow\infty} \frac{\mE_{X^{(j_w)}}  \psi_{rp}^{(d)}}{\sqrt{n}}
  & = &
\max_{\lambda^{(j_w)}>0} \lim_{n\rightarrow\infty} \psi_{rd}^{(d)}(\hat{\p}(1,\lambda^{(j_w)}),\hat{\q}(1,\lambda^{(j_w)}),\hat{\c}(1,\lambda^{(j_w)}),1,\lambda^{(j_w)},-1) \nonumber \\
& & \hspace{2in} \mbox{(\bl{\textbf{strong sfl random duality}})}. \nonumber \\
 \end{eqnarray}
 One then has \vspace{-.0in}
\begin{eqnarray}
\hspace{-.3in}(\phi_0  > 0)   & \Longleftrightarrow & \lp \lim_{n\rightarrow\infty}\mP_{X}(f_{rp}(X)>0)\longrightarrow 1 \rp
\Longleftrightarrow  \lp \lim_{n\rightarrow\infty}\mP_{X}(\psi_{rp}^{(d)}>0)\longrightarrow 1 \rp \nonumber \\
& \Longleftrightarrow & \lp \lim_{n\rightarrow\infty}\mP_{X}(A([n,d,1];\f^{(2)}) \quad \mbox{fails to memorize data set} \quad (X,\1))\longrightarrow 1\rp.
\label{eq:cor2ta17}
\end{eqnarray}
  \end{corollary}
\begin{proof}
Follows immediately from Corollary \ref{cor:cor1} and Lemma \ref{lemma:lemma1} after first recognizing the following connections:
\begin{eqnarray}\label{eq:prcor2eq1}
j \leftrightarrow j_w, \quad  X^{(j)} \leftrightarrow G^{(j_w)}, \quad   \Lambda_{:,j} \leftrightarrow \y^{(j_w)}, \quad  \z^{(j)}\leftrightarrow \x^{(j_w)}, \quad \cX \leftrightarrow \mS^n,
\end{eqnarray}
and then observing that
\begin{eqnarray}\label{eq:prcor2eq2}
\|\Lambda\|_F=1 \quad \Longleftrightarrow \quad \|\lambda^{(j_w)}\|_2=\sum_{j_w=1}^{d} \lp\lambda^{(j_w)}\rp^2=1 \quad \Longleftrightarrow \quad \cY\times\cY\dots\times \cY=\mS^{dn},
\end{eqnarray}
and that for $s=-1$, $\psi_{rp}^{(d)}$ from (\ref{eq:cor2sflrdt2eq3a0a0}), and $f_{rp}(X)$ from (\ref{eq:ta11})
\begin{eqnarray}\label{eq:prcor2eq3}
\sqrt{n} f_{rp}(X)=\psi_{rp}^{(d)}.
 \end{eqnarray}
 \end{proof}

From the above corollary one then easily recognizes the relevance and importance of $\psi_{rd}^{(d)}(\p,\q,\c,1,\lambda^{(j_w)},-1)$ from (\ref{eq:cor2sflrdt2eq1a0a1rd}) and consequently of
$\psi_{S,\infty}^{(d)}(f(\Lambda,Q),\lp \mS^n \rp^{\times d},\mS^{dm},\p,\q,\c,1,\lambda^{(j_w)},-1)$ from (\ref{eq:cor2fl4}), and $ D_{0,S,\infty}^{(d)}$ from (\ref{eq:cor2fl5}). Rewriting $ D_{0,S,\infty}^{(d)}$  for $s=-1$ and $f_S=f(\Lambda,Q)=-\tr(\Lambda^TQ)$ one obtains
\begin{eqnarray}\label{eq:fiteq1}
  D_{0,S,\infty}^{(d)} & \triangleq  & - \min_{\|\z^{(j_w)}\|_2=1,\phi(Q)=1} \max_{\|\Lambda_{:,j_w}\|_2=\lambda^{(j_w)},\|\lambda^{(j_w)}\|_2=1}
 \lp -\sqrt{n} \tr(\Lambda^TQ)
+{\mathcal D}_1
+{\mathcal D}_2
 \rp,
  \end{eqnarray}
with
\begin{eqnarray}\label{eq:fiteq2}
{\mathcal D}_1  & \triangleq & \sqrt{n} \sum_{j_w=1}^{d} \lambda^{(j_w)}    \lp\sum_{k=2}^{r+1}c_k\h^{(k,j_w)}\rp^T\z^{(j_w)} \nonumber \\
{\mathcal D}_2  & \triangleq  & \sqrt{n} \sum_{j_w=1}^{d}  \lp\Lambda_{:,j_w}\rp^T\lp\sum_{k=2}^{r+1}b_k\u^{(2,k,j_w)}\rp.
\end{eqnarray}
After optimizing over $\Lambda$ and $\z$, we find
\begin{eqnarray}\label{eq:fiteq3}
  D_{0,S,\infty}^{(d)} & \triangleq  & - \sqrt{n} \min_{\phi(Q)=1}
  \sqrt{a^{(f)}_1-2a^{(f)}_2+a^{(f)}_3},
  \end{eqnarray}
where
\begin{eqnarray}\label{eq:fiteq4}
a^{(f)}_1 & = & \sum_{j_w=1}^{d}   \left \|  \lp\sum_{k=2}^{r+1}b_k\u^{(2,k,j_w)}\rp-Q_{:,j_w} \right \|_2^2 \nonumber \\
a^{(f)}_2 & = &  \sum_{j_w=1}^{d} \left \|  \lp\sum_{k=2}^{r+1}b_k\u^{(2,k,j_w)}\rp-Q_{:,j_w} \right \|_2  \left \| \sum_{k=2}^{r+1}c_k\h^{(k,j_w)}\right \|_2 \nonumber \\
a^{(f)}_3 & = & \sum_{j_w=1}^{d} \left \| \sum_{k=2}^{r+1}c_k\h^{(k,j_w)}\right \|_2^2.
  \end{eqnarray}
Since by Cauchy-Schwartz
\begin{eqnarray}\label{eq:fiteq4}
 a^{(f)}_2  \leq a^{(f)}_1a^{(f)}_3,
  \end{eqnarray}
one can then rewrite (\ref{eq:fiteq3}) as
\begin{eqnarray}\label{eq:fiteq5}
  D_{0,S,\infty}^{(d)} & \leq  & - \sqrt{n} \min_{\phi(Q)=1}
  \sqrt{a^{(f)}_1}-\sqrt{a^{(f)}_3} =D^{(net)}(d)+D^{(per)}(d),
  \end{eqnarray}
where
\begin{eqnarray}\label{eq:fiteq6}
 D^{(net)}(d) & \triangleq & - \sqrt{n} \min_{\phi(Q)=1} \sqrt{\sum_{j_w=1}^{d}   \left \|  \lp\sum_{k=2}^{r+1}b_k\u^{(2,k,j_w)}\rp-Q_{:,j_w} \right \|_2^2} \nonumber \\
  D^{(per)}(d) & \triangleq & \sqrt{n}\sqrt{\sum_{j_w=1}^{d} \left \| \sum_{k=2}^{r+1}c_k\h^{(k,j_w)}\right \|_2^2}.
   \end{eqnarray}
Utilizing the  \emph{square root trick} introduced on numerous occasions in \cite{StojnicMoreSophHopBnds10,StojnicLiftStrSec13,StojnicGardSphErr13,StojnicGardSphNeg13}, we further find
\begin{eqnarray}\label{eq:prac4a00}
D^{(per)}(d) & = &
\sqrt{n} \sqrt{\sum_{j_w=1}^{d} \left \| \sum_{k=2}^{r+1}c_k\h^{(k,j_w)}\right \|_2^2}
=\sqrt{n} \min_{\gamma^{(p)}} \lp \frac{\sum_{j_w=1}^{d} \left \| \sum_{k=2}^{r+1}c_k\h^{(k,j_w)}\right \|_2^2}{4\gamma^{(p)}} +\gamma^{(p)} \rp \nonumber \\
& = & \sqrt{n} \min_{\gamma^{(p)}} \lp \frac{\sum_{i=1}^{\frac{n}{d}}\sum_{j_w=1}^{d} \lp \sum_{k=2}^{r+1}c_k\h_i^{(k,j_w)}\rp^2}{4\gamma^{(p)}} +\gamma^{(p)} \rp.
 \end{eqnarray}
After introducing scaling $\gamma^{(p)}=\gamma^{(p)}_{sq}\sqrt{n}$, one can rewrite (\ref{eq:prac4a00}) as
\begin{eqnarray}\label{eq:prac4a01}
D^{(per)}(s) & = &
  \sqrt{n} \min_{\gamma_{sq}^{(p)}} \lp \frac{\sum_{i=1}^{\frac{n}{d}}\sum_{j_w=1}^{d} \lp \sum_{k=2}^{r+1}c_k\h_i^{(k,j_w)}\rp^2}{4\gamma_{sq}^{(p)}\sqrt{n}} +\gamma_{sq}^{(p)}\sqrt{n} \rp \nonumber \\
&  = &
  \min_{\gamma_{sq}^{(p)}} \lp \frac{\sum_{i=1}^{\frac{n}{d}}\sum_{j_w=1}^{d} \lp \sum_{k=2}^{r+1}c_k\h_i^{(k,j_w)}\rp^2}{4\gamma_{sq}^{(p)}} +\gamma_{sq}^{(p)}n \rp \nonumber \\
 & = &
  \min_{\gamma_{sq}^{(p)}} \lp \sum_{i=1}^{\frac{n}{d}}\sum_{j_w=1}^{d}  D^{(per)}_{i,j_w}(c_k) +\gamma_{sq}^{(p)}n \rp, \nonumber \\  .
 \end{eqnarray}
where
\begin{eqnarray}\label{eq:prac5}
D^{(per)}_{i,j_w}(c_k)=\frac{\lp\sum_{k=2}^{r+1}c_k\h_i^{(k,j_w)} \rp^2}{4\gamma_{sq}^{(p)}}.
\end{eqnarray}
After setting
\begin{equation}\label{eq:fitta19}
  \phi_i(Q_{i,1:d})\triangleq  \mbox{sign}(\f^{(2)}(Q_{i,1:d}) \w),
\end{equation}
another utilization of the \emph{square root trick} gives
\begin{eqnarray}\label{eq:prac8}
 D^{(net)}(d) & \triangleq & - \sqrt{n} \min_{\phi(Q)=1} \sqrt{\sum_{j_w=1}^{d}   \left \|  \lp\sum_{k=2}^{r+1}b_k\u^{(2,k,j_w)}\rp-Q_{:,j_w} \right \|_2^2} \nonumber \\
& = & - \sqrt{n}  \sqrt{\min_{\phi(Q)=1}  \sum_{j_w=1}^{d} \sum_{i=1}^{m} \left \|  \lp\sum_{k=2}^{r+1}b_k\u_i^{(2,k,j_w)}\rp-Q_{i,j_w} \right \|_2^2} \nonumber \\
& = & - \sqrt{n}  \sqrt{\min_{\phi(Q)=1}  \sum_{i=1}^{m}  \sum_{j_w=1}^{d} \left \|  \lp\sum_{k=2}^{r+1}b_k\u_i^{(2,k,j_w)}\rp-Q_{i,j_w} \right \|_2^2} \nonumber \\
 & = & - \sqrt{n} \sqrt{\sum_{i=1}^{m} \min_{\phi_i(Q_{i,1:d})=1}  \sum_{j_w=1}^{d}   \left \|  \lp\sum_{k=2}^{r+1}b_k\u_i^{(2,k,j_w)}\rp-Q_{i,j_w} \right \|_2^2} \nonumber \\
 & =  &  -\sqrt{n}  \min_{\gamma} \lp \frac{\sum_{i=1}^{m} \min_{\phi_i(Q_{i,1:d})=1}  \sum_{j_w=1}^{d}   \lp  \lp\sum_{k=2}^{r+1}b_k\u_i^{(2,k,j_w)}\rp-Q_{i,j_w} \rp^2}{4\gamma}+\gamma \rp.
 \end{eqnarray}
After introducing scaling $\gamma=\gamma_{sq}\sqrt{n}$, (\ref{eq:prac8}) can further be rewritten as
\begin{eqnarray}\label{eq:prac9}
   D^{(net)}(d)
  & =  & -\sqrt{n}  \min_{\gamma_{sq}} \lp \frac{\sum_{i=1}^{m} \min_{\phi_i(Q_{i,1:d})=1}  \sum_{j_w=1}^{d}   \lp  \lp\sum_{k=2}^{r+1}b_k\u_i^{(2,k,j_w)}\rp-Q_{i,j_w} \rp^2}{4\gamma_{sq}\sqrt{n}}+\gamma_{sq}\sqrt{n} \rp   \nonumber \\
  & =  & - \min_{\gamma_{sq}} \lp \frac{\sum_{i=1}^{m} \min_{\phi_i(Q_{:,1:d})=1}  \sum_{j_w=1}^{d}   \lp \lp\sum_{k=2}^{r+1}b_k\u_i^{(2,k,j_w)}\rp-Q_{i,j_w} \rp^2}{4\gamma_{sq}}+\gamma_{sq} n \rp   \nonumber \\
   & =  &  - \min_{\gamma_{sq}} \lp \sum_{i=1}^{m}D_{i}^{(net)}(b_k)+\gamma_{sq}n \rp, \nonumber \\
 \end{eqnarray}
with
\begin{eqnarray}\label{eq:prac10}
   D_i^{(net)}(b_k)= \frac{\min_{\phi_i(Q_{:,1:d})=1}  \sum_{j_w=1}^{d}   \lp \lp\sum_{k=2}^{r+1}b_k\u_i^{(2,k,j_w)}\rp-Q_{i,j_w} \rp^2}{4\gamma_{sq}}.
 \end{eqnarray}

 We summarize the above discussion into the following theorem.

 \begin{theorem}
  \label{thm:thm2}
  Assume the setup of Lemma \ref{lemma:lemma1} and Theorem \ref{thm:thm1}. Consider large $n$ linear regime with $\alpha=\lim_{n\rightarrow\infty} \frac{m}{n}$ and set
  \begin{eqnarray}\label{eq:thm2prac2}
\varphi(D,\c) & \triangleq &
 \mE_{{\mathcal U}_{r+1}^{(j_w)}} \frac{1}{\c_r} \log
\lp \mE_{{\mathcal U}_{r}^{(j_w)}} \lp \dots \lp \mE_{{\mathcal U}_3^{(j_w)}}\lp\lp\mE_{{\mathcal U}_2^{(j_w)}} \lp
\lp    e^{D}   \rp^{\c_2}\rp\rp^{\frac{\c_3}{\c_2}}\rp\rp^{\frac{\c_4}{\c_3}} \dots \rp^{\frac{\c_{r}}{\c_{r-1}}}\rp, \nonumber \\
  \end{eqnarray}
and
\begin{eqnarray}\label{eq:thm2negprac13}
    \bar{\psi}_{rd}^{(d)}(\p,\q,\c,\gamma_{sq},\gamma_{sq}^{(p)})   & \triangleq &  \frac{1}{2}    \sum_{k=2}^{r+1}\Bigg(\Bigg.
   \p_{k-1}\q_{k-1}
   -\p_{k}\q_{k}
  \Bigg.\Bigg)
\c_k
\nonumber \\
& & -\gamma_{sq}^{(p)} - \varphi(D_{1,1}^{(per)}(c_k(\p,\q)),\c) +\gamma_{sq}- \alpha\varphi(-D_1^{(net)}(b_k(\p,\q)),\c),\nonumber \\
  \end{eqnarray}
where $D_{1,1}^{(per)}(c_k(\p,\q))$ and $D_1^{(net)}(b_k(\p,\q))$ are as in (\ref{eq:prac5}) and  (\ref{eq:prac10}), respectively. Let the ``fixed'' parts of $\hat{\p}$, $\hat{\q}$, and $\hat{\c}$ satisfy $\hat{\p}_1\rightarrow 1$, $\hat{\q}_1\rightarrow 1$, $\hat{\c}_1\rightarrow 1$, $\hat{\p}_{r+1}=\hat{\q}_{r+1}=\hat{\c}_{r+1}=0$, and let the ``non-fixed'' parts of $\hat{\p}_k$, $\hat{\q}_k$, and $\hat{\c}_k$ ($k\in\{2,3,\dots,r\}$) be the solutions of the following system of equations
  \begin{eqnarray}\label{eq:negthmprac1eq1}
   \frac{d \bar{\psi}_{rd}^{(d)}(\p,\q,\c,\gamma_{sq},\gamma_{sq}^{(p)})}{d\p} & = &  0 \nonumber \\
   \frac{d \bar{\psi}_{rd}^{(d)}(\p,\q,\c,\gamma_{sq},\gamma_{sq}^{(p)})}{d\q} & = &  0 \nonumber \\
   \frac{d \bar{\psi}_{rd}^{(d)}(\p,\q,\c,\gamma_{sq},\gamma_{sq}^{(p)})}{d\c} & = &  0 \nonumber \\
   \frac{d \bar{\psi}_{rd}^{(d)}(\p,\q,\c,\gamma_{sq},\gamma_{sq}^{(p)})}{d\gamma_{sq}} & = &  0\nonumber \\
   \frac{d \bar{\psi}_{rd}^{(d)}(\p,\q,\c,\gamma_{sq},\gamma_{sq}^{(p)})}{d\gamma_{sq}^{(p)}} & = &  0,
 \end{eqnarray}
 and, consequently, let
\begin{eqnarray}\label{eq:prac17}
c_k(\hat{\p},\hat{\q})  & = & \sqrt{\hat{\q}_{k-1}-\hat{\q}_k} \nonumber \\
b_k(\hat{\p},\hat{\q})  & = & \sqrt{\hat{\p}_{k-1}-\hat{\p}_k}.
 \end{eqnarray}
 Then
\begin{align}
\hspace{-.0in} \bar{\psi}_{rd}^{(d)}(\hat{\p},\hat{\q},\hat{\c},\hat{\gamma}_{sq},\hat{\gamma}_{sq}^{(p)}) >0 & \Longleftrightarrow  \lp \lim_{n\rightarrow\infty}\mP_{X}(f_{rp}(X)>0)\longrightarrow 1 \rp
 \nonumber \\
& \Longleftrightarrow  \lp \lim_{n\rightarrow\infty}\mP_{X}(A([n,d,1];\f^{(2)}) \quad \mbox{fails to memorize data set} \quad (X,\1))\longrightarrow 1\rp. \nonumber \\
\label{eq:thm2ta17}
\end{align}
\end{theorem}
\begin{proof}
Follows from the previous discussion, Lemma \ref{lemma:lemma1}, Theorem \ref{thm:thm1}, Corollaries \ref{cor:cor1} and \ref{cor:cor2}, the sfl RDT machinery presented in \cite{Stojnicnflgscompyx23,Stojnicsflgscompyx23,Stojnicflrdt23,Stojnichopflrdt23}, and after recognizing that for $\phi_0$ and $\psi_{rd}^{(d)}(\hat{\p}(1,\lambda^{(j_w)}),\hat{\q}(1,\lambda^{(j_w)}),\hat{\c}(1,\lambda^{(j_w)}),1,\lambda^{(j_w)},-1) $ from (\ref{eq:cor2sflrdt2eq3a0})
and $\bar{\psi}_{rd}^{(d)}(\hat{\p},\hat{\q},\hat{\c},\hat{\gamma}_{sq},\hat{\gamma}_{sq}^{(p)})$ from (\ref{eq:thm2ta17}) one, based on (\ref{eq:fiteq1})-(\ref{eq:prac10}), has
 \begin{eqnarray}
\phi_0=\max_{\lambda^{(j_w)}>0} \lim_{n\rightarrow\infty} \psi_{rd}^{(d)}(\hat{\p}(1,\lambda^{(j_w)}),\hat{\q}(1,\lambda^{(j_w)}),\hat{\c}(1,\lambda^{(j_w)}),1,\lambda^{(j_w)},-1) =  \bar{\psi}_{rd}^{(d)}(\hat{\p},\hat{\q},\hat{\c},\hat{\gamma}_{sq},\hat{\gamma}_{sq}^{(p)}).
  \label{eq:pprthm2negprac12a}
\end{eqnarray}
The inequality in (\ref{eq:fiteq5}) is sufficient for the implication in (\ref{eq:thm2ta17}). The equivalence follows since, due to symmetry and concentrations, the inequality in (\ref{eq:fiteq5}) can actually be replaced by equality.
\end{proof}

\section{Practical utilization and numerical evaluations}
\label{sec:pracutil}

For the results of Theorem \ref{thm:thm2} to become practically useful, all the underlying quantities need to be successfully evaluated. That, in general, is not an easy task. A couple of obstacles might be particulary unsurpassable: \textbf{\emph{(i)}} It is  not clear a priori what is the correct value for $r$; and \textbf{\emph{(ii)}} The residual decoupling over $\y$ is, in general, potentially highly non-convex. By a complete miracle, it however turns out, that each of them can be successfully surpassed. A majority of technical ingredients needed for the evaluations is already present in the theorem itself. However, several additional aspects will need to be addressed as well. These are, however, fairly specific and not particularly convenient to be presented in a generic form. We therefore discuss them as the presentation progresses within the context where their relevance becomes important. As is usually the case with the fl RDT considerations, the evaluations start  with $r=1$ and proceed by increasing $r$ incrementally. This enhances the clarity and enables a systematic following of the overall lifting mechanism's progression. As a bonus, this, at the same time, also allows to establish adequate connections with some of the known results. Since we consider several different hidden layer activations, $\f^{(2)}$, to ensure the easiness of the exposition and following, we try to parallel the presentation of each of them with the remaining ones. Also, since we will take as concrete examples some of the well known activations, it will be possible to obtain corresponding concrete capacity numerical values. Finally, several explicit analytical results can be obtained that substantially simplify the evaluation process. These will be stated as the presentation progresses below as well.

To facilitate writing and exposition, we set
\begin{eqnarray}\label{eq:prac10a0}
\g^{(k)} & \triangleq & \begin{bmatrix}
                 \u_i^{(2,k,1)} & \u_i^{(2,k,2)} & \dots & \u_i^{(2,k,d)}
               \end{bmatrix}^T \nonumber \\
\bar{\g}^{(r+1)} & \triangleq & \left \{\g^{(2)},\g^{(3)},\dots,\g^{(r+1)}\right \} \nonumber \\
\q^{(net)} & \triangleq &\lp Q_{i,1:d}\rp^T,
  \end{eqnarray}
and
\begin{eqnarray}\label{eq:prac10a1}
 z_i^{(r)}(\bar{\g}^{(r+1)};\f^{(2)})  & \triangleq &
 \min_{\phi_i(Q_{:,1:d})=1}  \sum_{j_w=1}^{d}   \lp \lp\sum_{k=2}^{r+1}b_k\u_i^{(2,k,j_w)}\rp-Q_{i,j_w} \rp^2 \nonumber \\
 & = &
 \min_{\f^{(2)}\lp\lp\q^{(net)}\rp^T\rp \w\geq 0}  \sum_{j_w=1}^{d}   \lp \lp\sum_{k=2}^{r+1}b_k \g_{j_w}^{(k)}\rp-\q^{(net)}_{j_w} \rp^2.
   \end{eqnarray}
   One then has
   \begin{eqnarray}\label{eq:prac10a2}
   D_i^{(net)}(b_k)
   & = & \frac{\min_{\phi_i(Q_{:,1:d})=1}  \sum_{j_w=1}^{d}   \lp \lp\sum_{k=2}^{r+1}b_k\u_i^{(2,k,j_w)}\rp-Q_{i,j_w} \rp^2}{4\gamma_{sq}} \nonumber \\
   & = & \frac{\min_{\f^{(2)}\lp\lp\q^{(net)}\rp^T\rp \w\geq 0}  \sum_{j_w=1}^{d}   \lp \lp\sum_{k=2}^{r+1}b_k \g_{j_w}^{(k)}\rp-\q^{(net)}_{j_w} \rp^2}{4\gamma_{sq}} \nonumber \\
   & = & \frac{ z_i^{(r)}(\bar{\g}^{(r+1)};\f^{(2)}) }{4\gamma_{sq}}.
 \end{eqnarray}

\subsection{ReLU activations}
\label{sec:relu}

We start by considering the well known ReLU activation. In other words, we assume that the neuronal activation functions in the hidden layer are
\begin{equation}
\hspace{-1.5in} \mbox{\textbf{\bl{\emph{ReLU}} activation:}} \hspace{1in} \f^{(2)}(\x)=\max(\x,0).
\label{eq:reluact1}
\end{equation}
As stated earlier, we begin by considering the first level of lifting.

\subsubsection{$r=1$ -- first level of lifting}
\label{sec:firstlev}

For the first level, we have $r=1$ and $\hat{\p}_1\rightarrow 1$ and $\hat{\q}_1\rightarrow 1$ which, together with $\hat{\p}_{r+1}=\hat{\p}_{2}=\hat{\q}_{r+1}=\hat{\q}_{2}=0$, and $\hat{\c}_{2}\rightarrow 0$, gives
\begin{align}\label{eq:negprac19}
    \bar{\psi}_{rd}^{(d,1)}(\hat{\p},\hat{\q},\hat{\c},\gamma_{sq},\gamma_{sq}^{(p)})   & =   \frac{1}{2}
\c_2
-\gamma_{sq}^{(p)}  - \frac{1}{\c_2}\log\lp \mE_{{\mathcal U}_2^{(j_w)}} e^{\c_2\frac{\lp\sqrt{1-0}\h_1^{(2,1)}\rp^2}{4\gamma_{sq}^{(p)}}}\rp
\nonumber \\
&   +\gamma_{sq}
- \alpha\frac{1}{\c_2}\log\lp \mE_{{\mathcal U}_2^{(j_w)}} e^{-\c_2\frac{\min_{\f^{(2)}\lp\lp\q^{(net)}\rp^T\rp \w\geq 0}  \sum_{j_w=1}^{d}   \lp \lp \sqrt{1-0} \g_{j_w}^{(2)}\rp-\q^{(net)}_{j_w} \rp^2}{4\gamma_{sq}}}\rp \nonumber \\
& \rightarrow
-\gamma_{sq}^{(p)}   - \frac{1}{\c_2}\log\lp 1+ \mE_{{\mathcal U}_2^{(j_w)}} \c_2\frac{\lp\sqrt{1-0}\h_1^{(2,1)}\rp^2}{4\gamma_{sq}^{(p)}}
\rp  +\gamma_{sq} \nonumber \\
& \qquad - \alpha\frac{1}{\c_2}\log\lp 1- \mE_{{\mathcal U}_2^{(j_w)}} \c_2\frac{\min_{\f^{(2)}\lp\lp\q^{(net)}\rp^T\rp \w\geq 0}  \sum_{j_w=1}^{d}   \lp \g_{j_w}^{(2)}-\q^{(net)}_{j_w} \rp^2}{4\gamma_{sq}} \rp \nonumber \\
& \rightarrow
 -\gamma_{sq}^{(p)}  - \frac{1}{\c_2}\log\lp 1+ \c_2\frac{1}{4\gamma_{sq}^{(p)}}\rp  \nonumber \\
& \qquad +\gamma_{sq} - \alpha\frac{1}{\c_2}\log\lp 1- \frac{\c_2}{4\gamma_{sq}} \mE_{{\mathcal U}_2^{(j_w)}} \min_{\f^{(2)}\lp\lp\q^{(net)}\rp^T\rp \w\geq 0}  \sum_{j_w=1}^{d}   \lp  \g_{j_w}^{(2)}-\q^{(net)}_{j_w} \rp^2 \rp \nonumber \\
& \rightarrow
   -\gamma_{sq}^{(p)}-\frac{1}{4\gamma_{sq}^{(p)}} +\gamma_{sq}
+  \frac{\alpha}{4\gamma_{sq}}\mE_{{\mathcal U}_2^{(j_w)}} \min_{\f^{(2)}\lp\lp\q^{(net)}\rp^T\rp \w\geq 0}  \sum_{j_w=1}^{d}   \lp \g_{j_w}^{(2)}-\q^{(net)}_{j_w} \rp^2.
  \end{align}
One then easily finds $\gamma_{sq}^{(p)}=\frac{1}{2}$ and $\hat{\gamma}_{sq}=\frac{\sqrt{\alpha}}{2}\sqrt{\mE_{{\mathcal U}_2^{(j_w)}} \min_{\f^{(2)}\lp\lp\q^{(net)}\rp^T\rp \w\geq 0}  \sum_{j_w=1}^{d}   \lp  \g_{j_w}^{(2)}-\q^{(net)}_{j_w} \rp^2}$ and
\begin{align}\label{eq:negprac20}
 \bar{\psi}_{rd}^{(d,1)}(\hat{\p},\hat{\q},\hat{\c},\hat{\gamma}_{sq},\hat{\gamma}_{sq}^{(p)})   & =
  -1+\sqrt{\alpha}\sqrt{\mE_{{\mathcal U}_2^{(j_w)}} \min_{\f^{(2)}\lp\lp\q^{(net)}\rp^T\rp \w\geq 0}  \sum_{j_w=1}^{d}   \lp \g_{j_w}^{(2)} - \q^{(net)}_{j_w} \rp^2}.
  \end{align}
To obtain the critical $\alpha_c^{(1)}$, we rely on condition $\bar{\psi}_{rd}^{(d,1)}(\hat{\p},\hat{\q},\hat{\c},\hat{\gamma}_{sq},\hat{\gamma}_{sq}^{(p)})=0$, which gives
\begin{eqnarray}\label{eq:negprac20a0}
a_c^{(1)}(d)
& = &   \frac{1}{\mE_{{\mathcal U}_2^{(j_w)}}  \lp z_i^{(1)}\lp\g^{(2)};\f^{(2)}\lp \q^{(net)} \rp\rp\rp}
\nonumber \\
& = &
  \frac{1}{\mE_{{\mathcal U}_2^{(j_w)}} \min_{\f^{(2)}\lp\lp\q^{(net)}\rp^T\rp \w\geq 0}  \sum_{j_w=1}^{d}   \lp \g_{j_w}^{(2)} - \q^{(net)}_{j_w} \rp^2} \nonumber \\
& = &
  \frac{1}{\mE_{ \g^{(2)}} \min_{\lp\max\lp\q^{(net)},0\rp\rp^T \w\geq 0}  \sum_{j_w=1}^{d}   \lp   \g_{j_w}^{(2)}-\q^{(net)}_{j_w} \rp^2} \nonumber \\  & = &   \frac{1}{\mE_{\g^{(2)}}  \lp z_i^{(1)}\lp \bar{\g}^{(2)};\max\lp\q^{(net)},0\rp\rp\rp}.
  \end{eqnarray}
As discussed in \cite{Stojnictcmspnncapdiffactrdt23}, solving the optimization in (\ref{eq:negprac20a0})  is not an easy task in general. For a couple of small (even) values of $d$ it was done analytically. For larger values of $d$ the analytical solutions required additional numerical simulations to complete the needed evaluations. Here, however, we uncover that when $d$ is large, i.e. when $d\rightarrow\infty$, the numerical evaluations miraculously sufficiently simplify so that they can ultimately be done. For the simplicity of writing, we assume large even $d$ and, due to the nonnegativity of ReLU,
\begin{equation}\label{eq:reluact2}
  \w=\begin{bmatrix}
      -\1 \\ \1
 \end{bmatrix},
 \end{equation}
 where $\1$ is the $\frac{d}{2}$-dimensional vector of all ones.

 To handle $\mE_{\g^{(2)}}  \lp z_i^{(1)}\lp  \bar{\g}^{(2)};\f^{(2)}\lp \q^{(net)} \rp \rp\rp$, we start by writing
 \begin{eqnarray}\label{eq:reluact3}
z_i^{(1)}\lp  \bar{\g}^{(2)};\f^{(2)}\lp \q^{(net)} \rp\rp  =
 \min_{q^{(net)}} & &  \sum_{j_w=1}^{d}   \lp  \g_{j_w}^{(2)}-\q^{(net)}_{j_w} \rp^2 \nonumber \\
 \mbox{subject to} & & \f^{(2)}\lp \q^{(net)} \rp^T\w\geq 0.
 \end{eqnarray}
After further writing the Lagrangian and utilizing the Lagrangian duality one obtains
 \begin{equation}\label{eq:reluact4}
z_i^{(1)}\lp  \bar{\g}^{(2)};\f^{(2)}\lp \q^{(net)} \rp   \rp
 =   \min_{\q^{(net)}}\max_{\nu\geq 0} \cL(\nu)  \geq   \max_{\nu\geq 0} \min_{\q^{(net)}} \cL(\nu),
 \end{equation}
 where
 \begin{eqnarray}\label{eq:reluact5}
\cL(\nu) & = &   \sum_{j_w=1}^{d}   \lp  \g_{j_w}^{(2)}-\q^{(net)}_{j_w} \rp^2 - 2\nu \f^{(2)}\lp \q^{(net)} \rp^T\w
= \left \| \g^{(2)}-\q^{(net)} \right \|_2^2 - 2\nu \f^{(2)}\lp \q^{(net)} \rp^T\w.
 \end{eqnarray}
Taking the derivative with respect to $\q^{(net)}$, we further find
 \begin{eqnarray}\label{eq:reluact6}
\frac{d\cL(\nu)}{d\q^{(net)}}   =
- 2  \lp  \g^{(2)}-\q^{(net)} \rp - 2\nu \frac{d\f^{(2)}\lp \q^{(net)} \rp }{d\q^{(net)}}\circ \w,
 \end{eqnarray}
 where $\circ$ stands for the component-wise multiplication. Equalling the above derivative to zero gives
 \begin{eqnarray}\label{eq:reluact7}
 \q^{(net)}=
 \g^{(2)} + \nu \frac{d\f^{(2)}\lp \q^{(net)} \rp }{d\q^{(net)}} \circ \w.
 \end{eqnarray}
Plugging this back into (\ref{eq:reluact5}), one then finds
 \begin{eqnarray}\label{eq:reluact8}
\cL(\nu)
& = &    \nu^2 \left \|  \frac{d\f^{(2)}\lp \q^{(net)} \rp }{d\q^{(net)}}\circ \w \right \|_2^2 - 2\nu \f^{(2)}\lp \q^{(net)} \rp^T\w \nonumber \\
& = &    \nu^2 \left \|  \frac{d\f^{(2)}\lp \q^{(net)} \rp }{d\q^{(net)}} \right \|_2^2 - 2\nu \f^{(2)}\lp \q^{(net)} \rp^T\w,
 \end{eqnarray}
with $ \q^{(net)}$ as given in (\ref{eq:reluact7}). For the time being we assume $\lim_{d\rightarrow\infty} \nu \rightarrow 0$, and for the $ \q^{(net)}$ from (\ref{eq:reluact7}), we write
 \begin{eqnarray}\label{eq:reluact8a0a0}
 \f^{(2)}\lp \q^{(net)} \rp  = \f^{(2)}\lp \g^{(2)} \rp + \lp \nu \frac{d\f^{(2)}\lp \q^{(net)} \rp }{d\q^{(net)}} \circ \w \rp
\circ \frac{d\f^{(2)}\lp \g^{(2)} \rp }{d\g^{(2)}} +o(\nu).
 \end{eqnarray}
Combining (\ref{eq:reluact8}) and (\ref{eq:reluact8a0a0}), one further finds
 \begin{eqnarray}\label{eq:reluact8a0a1}
\cL(\nu)
 & = &  -  \nu^2 \left \|  \frac{d\f^{(2)}\lp \g^{(2)} \rp }{d\g^{(2)}} \circ \w \right \|_2^2 - 2\nu \f^{(2)}\lp \g^{(2)} \rp^T\w +o(\nu).
 \end{eqnarray}
Taking the derivative of $\cL(\nu) $ with respect to $\nu$ gives
 \begin{eqnarray}\label{eq:reluact9}
\frac{d\cL(\nu)}{d\nu}
& \rightarrow &   - 2\nu \left \|  \frac{d\f^{(2)}\lp \g^{(2)} \rp }{d\g^{(2)}} \right \|_2^2 - 2\f^{(2)}\lp \g^{(2)} \rp^T\w.
 \end{eqnarray}
Equalling the above derivative to zero and keeping in mind that $\nu\geq 0$, one then finds
 \begin{eqnarray}\label{eq:reluact10}
     \nu^{(opt)} =\frac{\max\lp - \f^{(2)}\lp \g^{(2)} \rp^T\w,0\rp }{ \left \|  \frac{d\f^{(2)}\lp \g^{(2)} \rp }{d\g^{(2)}} \right \|_2^2}.
 \end{eqnarray}
Plugging this value of $\nu$ back into (\ref{eq:reluact8a0a1}), one obtains
 \begin{eqnarray}\label{eq:reluact11}
\cL(\nu^{(opt)})
 & = &   \frac{\lp \max\lp - \f^{(2)}\lp \g^{(2)} \rp^T\w,0\rp \rp^2}{ \left \|  \frac{d\f^{(2)}\lp \g^{(2)} \rp }{d\g^{(2)}} \right \|_2^2}.
 \end{eqnarray}
A combination of (\ref{eq:reluact3}), (\ref{eq:reluact4}), (\ref{eq:reluact5}), and (\ref{eq:reluact11}) together with concentrations gives
 \begin{eqnarray}\label{eq:reluact12}
\mE_{\g^{(2)}} z_i^{(1)}\lp   \bar{\g}^{(2)};\f^{(2)}\lp \q^{(net)} \rp   \rp
 & \geq  & \mE_{\g^{(2)}} \cL(\nu^{(opt)})  \nonumber \\
 & \rightarrow &
\frac{\mE_{\g^{(2)}} \lp \max\lp - \f^{(2)}\lp \g^{(2)} \rp^T\w,0\rp \rp^2}{\mE_{\g^{(2)}} \left \|  \frac{d\f^{(2)}\lp \g^{(2)} \rp }{d\g^{(2)}} \right \|_2^2}.
 \end{eqnarray}

\noindent \red{\textbf{\emph{(i) Handling $\mE_{\g^{(2)}} \lp \max\lp - \f^{(2)}\lp \g^{(2)} \rp^T\w,0\rp \rp^2$:}}} We first write
   \begin{eqnarray}\label{eq:reluact13}
  -\f^{(2)}\lp \g^{(2)} \rp^T\w
  & = & \sum_{j_w=1}^{\frac{d}{2}}\f^{(2)}\lp \g_{j_w}^{(2)}\rp - \sum_{j_w=\frac{d}{2}+1}^{d} \f^{(2)}\lp \g_{j_w}^{(2)}\rp.
\end{eqnarray}
Utilizing concentrations and the central limit theorem, one further has
 \begin{eqnarray}\label{eq:reluact14}
\sum_{j_w=1}^{\frac{d}{2}}\f^{(2)}\lp \g_{j_w}^{(2)}\rp & \rightarrow &  g_{c,1}^{(2)},
\end{eqnarray}
where $g_{c,1}^{(2)}$ is a Gaussian variable with mean $\mu_{2,1}$ and variance $\sigma_{2,1}^2$, i.e., $g_{c,1}^{(2)}$ is a Gaussian variable given by
 \begin{equation}\label{eq:reluact15}
g_{c,1}^{(2)} \sim {\mathcal N}(\mu_{2,1},\sigma_{2,1}^2) \quad \mbox{with} \quad \mu_{2,1}=\frac{d}{2}\mE\f^{(2)}\lp \g_{1}^{(2)} \rp, \quad \mbox{and}\quad
\sigma_{2,1}^2=\frac{d}{2} \lp \mE \lp\f^{(2)}\lp \g_{1}^{(2)} \rp \rp^2- \lp \mE\f^{(2)}\lp \g_{1}^{(2)}\rp\rp^2\rp.
\end{equation}
Analogously, one also has
 \begin{eqnarray}\label{eq:reluact16}
\sum_{j_w=\frac{d}{2}+1}^{d} \f^{(2)}\lp \g_{j_w}^{(2)}\rp  & \rightarrow &  g_{c,2}^{(2)},
\end{eqnarray}
where $g_{c,2}^{(2)}$ is a Gaussian variable with mean $\mu_{2,2}$ and variance $\sigma_{2,2}^2$, i.e., $g_{c,2}^{(2)}$ is a Gaussian variable given by
 \begin{equation}\label{eq:reluact17}
g_{c,2}^{(2)} \sim {\mathcal N}(\mu_{2,2},\sigma_{2,2}^2) \quad \mbox{with} \quad \mu_{2,2}=\frac{d}{2}\mE\f^{(2)}\lp \g_{1}^{(2)} \rp, \quad \mbox{and}\quad
\sigma_{2,2}^2=\frac{d}{2} \lp \mE \lp\f^{(2)}\lp \g_{1}^{(2)} \rp \rp^2- \lp \mE\f^{(2)}\lp \g_{1}^{(2)}\rp\rp^2 \rp.
\end{equation}
As $g_{c,1}^{(2)}$  and $g_{c,2}^{(2)}$ are independent one also has
 \begin{eqnarray}\label{eq:reluact18}
-\f^{(2)}\lp \g^{(2)} \rp^T\w = \sum_{j_w=1}^{\frac{d}{2}}\f^{(2)}\lp \g_{j_w}^{(2)}\rp - \sum_{j_w=\frac{d}{2}+1}^{d} \f^{(2)}\lp \g_{j_w}^{(2)}\rp
\longrightarrow \lp g_{c,1}^{(2)} -g_{c,2}^{(2)}\rp \longrightarrow g_{c}^{(2)},
\end{eqnarray}
where $g_{c}^{(2)}$ is a Gaussian variable with mean $\mu_2=\mu_{2,1}-\mu_{2,2}=0$ and variance $\sigma_2^2=\sigma_{2,1}^2+\sigma_{2,2}^2$, i.e., $g_{c}^{(2)}$ is a Gaussian variable given by
 \begin{eqnarray}\label{eq:reluact19}
g_{c}^{(2)} \sim {\mathcal N}(\mu_2,\sigma_2^2) \quad \mbox{with} \quad \mu_2=0, \quad \mbox{and}\quad
\sigma_2^2=d\lp\mE \lp\f^{(2)}\lp \g_{1}^{(2)} \rp \rp^2- \lp \mE\lp \f^{(2)}\lp \g_{1}^{(2)}\rp\rp\rp^2\rp.
\end{eqnarray}
We then also have
 \begin{eqnarray}\label{eq:reluact20}
\mE_{\g^{(2)}} \lp \max\lp - \f^{(2)}\lp \g^{(2)} \rp^T\w,0\rp \rp^2
=\mE_{g_c^{(2)}} \lp \max\lp g_c^{(2)},0\rp \rp^2=\frac{1}{2}\sigma_{2}^2.
\end{eqnarray}
Recalling on (\ref{eq:reluact1}), we first have
\begin{equation}
\mE \lp \f^{(2)}\lp \g_{1}^{(2)}\rp \rp=\mE \max(\g_1^{(2)},0) = \frac{1}{2}\sqrt{\frac{2}{\pi}},
\label{eq:reluact21}
\end{equation}
and then
\begin{equation}
\mE \lp \f^{(2)}\lp \g_{1}^{(2)}\rp \rp^2=\mE \lp \max(\g_1^{(2)},0)\rp^2 = \frac{1}{2}.
\label{eq:reluact22}
\end{equation}
Combining (\ref{eq:reluact19}), (\ref{eq:reluact21}), and (\ref{eq:reluact22}), one obtains
 \begin{eqnarray}\label{eq:reluact23}
 \sigma_2^2=d\lp\mE \lp\f^{(2)}\lp \g_{1}^{(2)} \rp \rp^2- \lp \mE\lp \f^{(2)}\lp \g_{1}^{(2)}\rp\rp\rp^2\rp
 =d\lp\frac{1}{2}-\lp \frac{1}{2}\sqrt{\frac{2}{\pi}}\rp^2\rp=d\lp\frac{\pi-1}{2\pi}\rp.
\end{eqnarray}
From  (\ref{eq:reluact20}) and (\ref{eq:reluact23}), we find
 \begin{eqnarray}\label{eq:reluact24}
\mE_{\g^{(2)}} \lp \max\lp - \f^{(2)}\lp \g^{(2)} \rp^T\w,0\rp \rp^2
 =\frac{1}{2}\sigma_{2}^2=d\lp\frac{\pi-1}{4\pi}\rp.
\end{eqnarray}

\noindent \red{\textbf{\emph{(ii) Handling $\mE_{\g^{(2)}} \left \|  \frac{d\f^{(2)}\lp \g^{(2)} \rp }{d\g^{(2)}} \right \|_2^2$:}}} Recalling again on (\ref{eq:reluact1}), we find
 \begin{eqnarray}\label{eq:reluact25}
\frac{d\f^{(2)}\lp \g_{j_w}^{(2)} \rp }{d\g_{j_w}^{(2)}}=\begin{cases}
                                                     1, & \mbox{if } \g_{j_w}^{(2)} \geq 0 \\
                                                     0, & \mbox{otherwise}.
                                                   \end{cases}
\end{eqnarray}
One then also has
 \begin{eqnarray}\label{eq:reluact26}
\mE_{\g^{(2)}} \left \|  \frac{d\f^{(2)}\lp \g^{(2)} \rp }{d\g^{(2)}} \right \|_2^2
=\mE_{\g^{(2)}} \sum_{j_w=1}^{d}\lp \frac{d\f^{(2)}\lp \g_{j_w}^{(2)} \rp }{d\g_{j_w}^{(2)}}\rp^2
=  \sum_{j_w=1}^{d}\mE_{\g_{j_w}^{(2)}} h_s(\g_{j_w}^{(2)})=\frac{d}{2},
\end{eqnarray}
where $h_s(\cdot)$ is the unit step function.  Moreover, a combination of (\ref{eq:reluact10}), (\ref{eq:reluact20}), and (\ref{eq:reluact26}) gives
 \begin{eqnarray}\label{eq:reluact26a0}
     \nu^{(opt)} =\frac{\max\lp - \f^{(2)}\lp \g^{(2)} \rp^T\w,0\rp }{ \left \|  \frac{d\f^{(2)}\lp \g^{(2)} \rp }{d\g^{(2)}} \right \|_2^2}
     \rightarrow
     \frac{ \max\lp g_c^{(2)},0\rp  }{\frac{d}{2}},
 \end{eqnarray}
which means that for any constant $\epsilon>0$
 \begin{eqnarray}\label{eq:reluact26a0}
   \lim_{d\rightarrow\infty}\mP \lp  \nu^{(opt)} >\epsilon \rp =1.
 \end{eqnarray}
This then confirms the small $\nu$ assumptions utilized earlier and ensures that the above machinery is indeed correct. Also, while the above establishes the inequality in (\ref{eq:reluact12}), it is trivial to check that taking  $ \q^{(net)}$ as  in (\ref{eq:reluact7}) with $\nu^{(opt)}$ from (\ref{eq:reluact10}), one has that the objective in (\ref{eq:reluact3}) is actually equal to the right hand side of (\ref{eq:reluact12}). Moreover, from
(\ref{eq:reluact8a0a0}), one finds
 \begin{eqnarray}\label{eq:reluact26a1}
 \f^{(2)}\lp \q^{(net)} \rp\w  = \f^{(2)}\lp \g^{(2)} \rp\w + \nu \left \|  \frac{d\f^{(2)}\lp \g^{(2)} \rp }{d\g^{(2)}} \right  \|_2^2 +o(\nu),
 \end{eqnarray}
 which for $\nu^{(opt)}$ from (\ref{eq:reluact10}) ensures that, with probability going to 1 as $d\rightarrow\infty$, $ \f^{(2)}\lp \q^{(net)} \rp\w \geq 0$. This also implies that $ \q^{(net)}$ from (\ref{eq:reluact7}) with $\nu^{(opt)}$ from (\ref{eq:reluact10}) is, in the large $d$ limit, with probability 1 feasible in (\ref{eq:reluact3}), which, on the other hand, ensures that the lower bound given by the right hand side of (\ref{eq:reluact12}) is actually attainable. All of this practically means that one has the equality in (\ref{eq:reluact12}). Due to the concentrations, the equality holds not only for the expectations but also with probability going to 1 as $d\rightarrow\infty$.

Utilizing all of the above observations, and combining (\ref{eq:reluact12}), (\ref{eq:reluact24}), and (\ref{eq:reluact26}), we then obtain
 \begin{eqnarray}\label{eq:reluact27}
\lim_{d\rightarrow\infty}\mE_{\g^{(2)}} z_i^{(1)}\lp   \bar{\g}^{(2)};\f^{(2)}\lp \q^{(net)} \rp   \rp
 & =  &  \lim_{d\rightarrow\infty}\frac{\mE_{\g^{(2)}} \lp \max\lp - \f^{(2)}\lp \g^{(2)} \rp^T\w,0\rp \rp^2}{\mE_{\g^{(2)}} \left \|  \frac{d\f^{(2)}\lp \g^{(2)} \rp }{d\g^{(2)}} \right \|_2^2}=\frac{\pi-1}{2\pi}.
 \end{eqnarray}
A further combination of (\ref{eq:negprac20a0}) and (\ref{eq:reluact27}) then gives
\begin{eqnarray}\label{eq:reluact28}
\hspace{-0in}(\mbox{\bl{\textbf{first level:}}}) \qquad  a_c^{(1)}(\infty)
& = &    \lim_{d\rightarrow\infty} a_c^{(1)}(d)  =   \frac{1}{\lim_{d\rightarrow\infty} \mE_{{\mathcal U}_2^{(j_w)}}  \lp z_i^{(1)}\lp  \bar{\g}^{(2)};\f^{(2)}\lp \q^{(net)} \rp\rp\rp}  \nonumber \\
 & = &   \frac{1}{\lim_{d\rightarrow\infty} \mE_{\g^{(2)}}  \lp z_i^{(1)}\lp   \bar{\g}^{(2)};\max\lp\q^{(net)},0\rp\rp\rp}
  = \frac{2\pi}{\pi-1}=\bl{\mathbf{2.9339}}.
  \end{eqnarray}

\subsubsection{$r=2$ -- second level of lifting}
\label{sec:secondlev}

The analysis of the second level of lifting will be split into two separate parts: (i) \emph{partial} second level of lifting; and (ii) \emph{full} second level of lifting.

\subsubsubsection{Partial second level of lifting}
\label{sec:secondlevpar}

 For $r=2$ and the partial lifting, we have (similarly to the first level)  $\hat{\p}_1\rightarrow 1$ and $\hat{\q}_1\rightarrow 1$, $\hat{\p}_{2}=\hat{\q}_{2}=0$, and $\hat{\p}_{r+1}=\hat{\p}_{3}=\hat{\q}_{r+1}=\hat{\q}_{3}=0$. However, now, in general,  $\hat{\c}_{2}\neq 0$. Following discussion of the previous sections, we again start by writing
\begin{align}\label{eq:reluact28a0}
    \bar{\psi}_{rd}^{(d,2)}(\hat{\p},\hat{\q},\c,\gamma_{sq},\gamma_{sq}^{(p)})   & =   \frac{1}{2}
\c_2
 -\gamma_{sq}^{(p)}  - \frac{1}{\c_2}\log\lp \mE_{{\mathcal U}_2^{(j_w)}} e^{\c_2\frac{\lp\sqrt{1-0}\h_1^{(2,1)}\rp^2}{4\gamma_{sq}^{(p)}}}\rp \nonumber \\
 &
\quad  + \gamma_{sq}
- \alpha\frac{1}{\c_2}\log\lp \mE_{\bar{\g}^{(3)}} e^{-\c_2\frac{z_i^{(2)}\lp \bar{\g}^{(3)};\f^{(2)} \rp}{4\gamma_{sq}}}\rp \nonumber \\
& =   \frac{1}{2}
\c_2
      -\gamma_{sq}^{(p)}  +\frac{1}{2\c_2}\log\lp \frac{2\gamma_{sq}^{(p)}-\c_2}{2\gamma_{sq}^{(p)}}\rp \nonumber \\
      &\quad
  + \gamma_{sq}
- \alpha\frac{1}{\c_2}\log\lp \mE_{\bar{\g}^{(3)}} e^{-\c_2\frac{z_i^{(2)}\lp \bar{\g}^{(3)};\max\lp \q^{(net)},0\rp \rp}{4\gamma_{sq}}}\rp.
    \end{align}
From (\ref{eq:reluact12}), (\ref{eq:reluact26a1}), considerations right after (\ref{eq:reluact26a1}), (\ref{eq:reluact27}), (\ref{eq:reluact18}), (\ref{eq:reluact19}), and (\ref{eq:reluact23}), we first have
 \begin{eqnarray}\label{eq:reluact29}
  z_i^{(2)}\lp   \bar{\g}^{(3)};\f^{(2)}\lp \q^{(net)} \rp   \rp & \rightarrow  &
  \frac{\lp \max\lp - \f^{(2)}\lp \g^{(3)} \rp^T\w,0\rp \rp^2}{\left \|  \frac{d\f^{(2)}\lp \g^{(3)} \rp }{d\g^{(3)}} \right \|_2^2},
 \end{eqnarray}
and then
 \begin{eqnarray}\label{eq:reluact30}
  z_i^{(2)}\lp   \bar{\g}^{(3)};\max\lp \q^{(net)},0 \rp   \rp
  & \rightarrow  &
  \frac{\lp \max\lp g_c^{(3)},0\rp \rp^2}{ \frac{d}{2}}
  \rightarrow
  \lp \max\lp \bar{g}_c^{(3)},0\rp \rp^2,
 \end{eqnarray}
with
 \begin{eqnarray}\label{eq:reluact31}
 \bar{g}_c^{(3)} \sim {\mathcal N}\lp 0, \bar{\sigma}_3^2 \rp, \quad \mbox{and}\quad \bar{\sigma}_3^2=\frac{\pi-1}{\pi}.
 \end{eqnarray}
A combination of (\ref{eq:reluact28a0}) and (\ref{eq:reluact31}) further gives
\begin{align}\label{eq:reluact32}
    \bar{\psi}_{rd}^{(d,2)}(\hat{\p},\hat{\q},\c,\gamma_{sq},\gamma_{sq}^{(p)})
& =   \frac{1}{2}
\c_2
      -\gamma_{sq}^{(p)}  +\frac{1}{2\c_2}\log\lp \frac{2\gamma_{sq}^{(p)}-\c_2}{2\gamma_{sq}^{(p)}}\rp
  + \gamma_{sq}
- \alpha\frac{1}{\c_2}\log\lp \mE_{\bar{g}_c^{(2)}} e^{-\c_2\frac{ \lp \max\lp \bar{g}_c^{(3)},0\rp \rp^2}{4\gamma_{sq}}}\rp \nonumber \\
& =   \frac{1}{2}
\c_2
      -\gamma_{sq}^{(p)}  +\frac{1}{2\c_2}\log\lp \frac{2\gamma_{sq}^{(p)}-\c_2}{2\gamma_{sq}^{(p)}}\rp
  + \gamma_{sq}
- \alpha\frac{1}{\c_2}\log\lp \frac{1}{2} + \frac{1}{2\sqrt{\frac{\bar{\sigma}_3^2\c_2}{2\gamma_{sq}}+1}} \rp.
    \end{align}
After computing the derivatives of $ \bar{\psi}_{rd}^{(d,2)}(\hat{\p},\hat{\q},\c,\gamma_{sq},\gamma_{sq}^{(p)})$ with respect to $\gamma_{sq}^{(p)}$, $\gamma_{sq}$, and $\c_2$ and equalling them to zero, one proceeds by solving the obtained system of equations. Denoting the solution of the system by
 $\hat{\gamma}_{sq}^{(p)}$, $\hat{\gamma}_{sq}$, and $\hat{\c}_2$, we first have the following convenient closed form relation
 \begin{equation}\label{eq:reluact33}
 \hat{\gamma}_{sq}^{(p)}=\frac{\hat{\c}_2+\sqrt{\hat{\c}_2^2+4}}{4},
   \end{equation}
and then ultimately from $\bar{\psi}_{rd}^{(d,2)}(\hat{\p},\hat{\q},\hat{\c},\hat{\gamma}_{sq},\hat{\gamma}_{sq}^{(p)})=0$ obtain for
 \begin{equation}\label{eq:reluact34}
\hspace{-2in}(\mbox{\bl{\textbf{\emph{partial} second level:}}}) \qquad \qquad  a_c^{(2,p)}(\infty) =  \lim_{d\rightarrow\infty} a_c^{(2,p)}(d)  \approx \bl{\mathbf{2.8503}}.
  \end{equation}

\subsubsubsection{Full second level of lifting}
\label{sec:secondlevfull}

One can also utilize the above setup for the full lifting on the second level. This time though, one has to be additionally careful. Namely, in addition to $\hat{\c}_{2}\neq 0$, one, in general, also has $\p_2\neq0$ and $\q_2\neq0$. Analogously to (\ref{eq:reluact28a0}), we now write
\begin{eqnarray}\label{eq:reluact35}
    \bar{\psi}_{rd}^{(d,2)}(\p,\q,\c,\gamma_{sq},\gamma_{sq}^{(p)})   & = &  \frac{1}{2}
(1-\p_2\q_2)\c_2
-  \gamma_{sq}^{(p)}  - \frac{1}{\c_2}\mE_{{\mathcal U}_3^{(j_w)}}\log\lp \mE_{{\mathcal U}_2^{(j_w)}} e^{\c_2\frac{\lp\sqrt{1-\q_2}\h_1^{(2,1)} +\sqrt{\q_2}\h_1^{(3,1)} \rp^2}{4 \gamma_{sq}^{(p)}}}\rp \nonumber \\
& &   + \gamma_{sq}
 -\alpha\frac{1}{\c_2}\mE_{{\mathcal U}_3^{(j_w)}} \log\lp \mE_{{\mathcal U}_2^{(j_w)}} e^{-\c_2\frac{z_i^{(2)}\lp \bar{\g}^{(3)};\max\lp \q^{(net)},0\rp\rp }{4\gamma_{sq}}}\rp \nonumber \\
 & = &  \frac{1}{2}
(1-\p_2\q_2)\c_2
 -  \gamma_{sq}^{(p)}
-\Bigg(\Bigg. -\frac{1}{2\c_2} \log \lp \frac{2\gamma_{sq}-\c_2(1-\q_2)}{2\gamma_{sq}} \rp  \nonumber \\
 & & +  \frac{\q_2}{2(2\gamma_{sq}-\c_2(1-\q_2))}   \Bigg.\Bigg)
 \nonumber \\
& &   + \gamma_{sq}
 -\alpha\frac{1}{\c_2}\mE_{{\mathcal U}_3^{(j_w)}} \log\lp \mE_{{\mathcal U}_2^{(j_w)}} e^{-\c_2\frac{z_i^{(2)}\lp \bar{\g}^{(3)};\max\lp \q^{(net)},0\rp\rp }{4\gamma_{sq}}}\rp.
    \end{eqnarray}

We now briefly digress and set
\begin{eqnarray}\label{eq:reluact36}
\g^{(x,r)}\triangleq \sum_{k=2}^{r+1}b_k \g^{(k)},
\end{eqnarray}
and recalling on (\ref{eq:prac10a1}) write
\begin{eqnarray}\label{eq:reluact37}
 z_i^{(r)}(\bar{\g}^{(r+1)};\f^{(2)})
 & = &
 \min_{\f^{(2)}\lp\lp\q^{(net)}\rp^T\rp \w\geq 0}  \sum_{j_w=1}^{d}   \lp \lp\sum_{k=2}^{r+1}b_k \g_{j_w}^{(k)}\rp-\q^{(net)}_{j_w} \rp^2 \nonumber \\
 & = &
 \min_{\f^{(2)}\lp\lp\q^{(net)}\rp^T\rp \w\geq 0}  \sum_{j_w=1}^{d}   \lp \g_{j_w}^{(x,r)}-\q^{(net)}_{j_w} \rp^2.
   \end{eqnarray}
Specializing (\ref{eq:reluact37}) to $r=2$, we further have
\begin{eqnarray}\label{eq:reluact38}
 z_i^{(2)}\lp\bar{\g}^{(3)};\f^{(2)}\lp\q^{(net)}\rp\rp
  & = &
 \min_{\f^{(2)}\lp\lp\q^{(net)}\rp^T\rp \w\geq 0} \sum_{j_w=1}^{d}   \lp \g_{j_w}^{(x,2)}-\q^{(net)}_{j_w} \rp^2.
   \end{eqnarray}
Repeating all the arguments between (\ref{eq:reluact3})   and (\ref{eq:reluact12}) and relying on the discussion between (\ref{eq:reluact26a0}) and (\ref{eq:reluact27}), one obtains that as $d\rightarrow\infty$
\begin{eqnarray}\label{eq:reluact39}
  z_i^{(2)}\lp   \bar{\g}^{(3)};\f^{(2)}\lp \q^{(net)} \rp   \rp
 &\rightarrow  &  \frac{ \lp \max\lp - \f^{(2)}\lp \g^{(x,2)} \rp^T\w,0\rp \rp^2}{\mE_{\bar{\g}^{(3)}} \left \|  \frac{d\f^{(2)}\lp \g^{(x,2)} \rp }{d\g^{(x,2)}} \right \|_2^2}.
 \end{eqnarray}

\noindent \red{\textbf{\emph{(i) Handling $ \lp \max\lp - \f^{(2)}\lp \g^{(x,2)} \rp^T\w,0\rp \rp^2$:}}} We first recall $b_2=\sqrt{1-\p_2}$ and $b_3=\sqrt{\p_2}$ and write
   \begin{eqnarray}\label{eq:reluact40}
  -\f^{(2)}\lp \g^{(x,2)} \rp^T\w
  & = & \sum_{j_w=1}^{\frac{d}{2}}\f^{(2)}\lp \g_{j_w}^{(x,2)}\rp - \sum_{j_w=\frac{d}{2}+1}^{d} \f^{(2)}\lp \g_{j_w}^{(x,2)}\rp.
\end{eqnarray}
Conditioning on $\g^{(3)}$, utilizing concentrations, and relying on the central limit theorem, one further has
 \begin{eqnarray}\label{eq:reluact41}
\sum_{j_w=1}^{\frac{d}{2}}\f^{(2)}\lp \g_{j_w}^{(x,2)}\rp
=\sum_{j_w=1}^{\frac{d}{2}}\f^{(2)}\lp \sum_{k=2}^{3}b_k \g_{j_w}^{(k)}\rp
 \rightarrow   g_{c,1}^{(3,1)},
\end{eqnarray}
where $g_{c,1}^{(3,1)}$ is a Gaussian variable with mean $\mu_{3,1;1}$ and variance $\sigma_{3,1;1}^2$, i.e., $g_{c,1}^{(3,1)}$ is a Gaussian variable given by
 \begin{equation}\label{eq:reluact42}
g_{c,1}^{(3,1)} \sim {\mathcal N}(\mu_{3,1;1},\sigma_{3,1;1}^2),
\end{equation}
 with
 \begin{equation}\label{eq:reluact43}
 \mu_{3,1;1}=\sum_{j_w=1}^{\frac{d}{2}}\mE_{\g^{(2)}}  \f^{(2)}\lp \sum_{k=2}^{r+1}b_k \g_{j_w}^{(k)}\rp,
\end{equation}
and
 \begin{eqnarray}\label{eq:reluact44}
 \sigma_{3,1;1}^2
 & = &\sum_{j_w=1}^{\frac{d}{2}} \lp \mE_{\g^{(3)}} \mE_{\g^{(2)}} \lp \f^{(2)}\lp \sum_{k=2}^{3}b_k \g_{j_w}^{(k)}\rp \rp^2
 - \mE_{\g^{(3)}}\lp\mE_{\g^{(2)}}  \f^{(2)}\lp \sum_{k=2}^{3}b_k \g_{j_w}^{(k)}\rp \rp^2  \rp \nonumber \\
 & = &   \frac{d}{2} \lp \mE_{\g_1^{(3)}} \mE_{\g_1^{(2)}} \lp \f^{(2)}\lp \sum_{k=2}^{3}b_k \g_1^{(k)}\rp \rp^2
 - \mE_{\g_1^{(3)}}\lp\mE_{\g_1^{(2)}}  \f^{(2)}\lp \sum_{k=2}^{3}b_k \g_1^{(k)}\rp \rp^2  \rp.
\end{eqnarray}
Moreover, utilizing again the concentrations, and relying on the central limit theorem, one further has for $\mu_{3,1;1}$ itself
 \begin{eqnarray}\label{eq:reluact45}
 \mu_{3,1;1}=\sum_{j_w=1}^{\frac{d}{2}}\mE_{\g^{(2)}}  \f^{(2)}\lp \sum_{k=2}^{3}b_k \g_{j_w}^{(k)}\rp \rightarrow   g_{c,1}^{(3,2)},
\end{eqnarray}
where $g_{c,1}^{(3,2)}$ is a Gaussian variable with mean $\mu_{3,2;1}$ and variance $\sigma_{3,2;1}^2$, i.e., $g_{c,1}^{(3,2)}$ is a Gaussian variable given by
 \begin{equation}\label{eq:reluact46}
g_{c,1}^{(3,2)} \sim {\mathcal N}(\mu_{3,2;1},\sigma_{3,2;1}^2),
\end{equation}
 with
 \begin{equation}\label{eq:reluact47}
 \mu_{3,2;1}=\sum_{j_w=1}^{\frac{d}{2}} \mE_{\g^{(3)}} \mE_{\g^{(2)}}  \f^{(2)}\lp \sum_{k=2}^{3}b_k \g_{j_w}^{(k)}\rp
 =\frac{d}{2} \mE_{\g_1^{(3)}} \mE_{\g_1^{(2)}}  \f^{(2)}\lp \sum_{k=2}^{3}b_k \g_1^{(k)}\rp,
\end{equation}
and
 \begin{eqnarray}\label{eq:reluact48}
 \sigma_{3,2;1}^2
 & = &\sum_{j_w=1}^{\frac{d}{2}} \lp  \mE_{\g^{(3)}}\lp\mE_{\g^{(2)}}  \f^{(2)}\lp \sum_{k=2}^{3}b_k \g_{j_w}^{(k)}\rp \rp^2
 - \lp \mE_{\g^{(3)}}\mE_{\g^{(2)}}  \f^{(2)}\lp \sum_{k=2}^{3}b_k \g_{j_w}^{(k)}\rp \rp^2 \rp \nonumber \\
 & = &   \frac{d}{2} \lp \mE_{\g_1^{(3)}}\lp \mE_{\g_1^{(2)}}  \f^{(2)}\lp \sum_{k=2}^{3}b_k \g_1^{(k)}\rp \rp^2
 - \lp\mE_{\g_1^{(3)}} \mE_{\g_1^{(2)}}  \f^{(2)}\lp \sum_{k=2}^{3}b_k \g_1^{(k)}\rp \rp^2  \rp.
\end{eqnarray}
A combination of (\ref{eq:reluact41})-(\ref{eq:reluact48}) gives
 \begin{eqnarray}\label{eq:reluact48a0}
\sum_{j_w=1}^{\frac{d}{2}}\f^{(2)}\lp \g_{j_w}^{(x,2)}\rp
=\sum_{j_w=1}^{\frac{d}{2}}\f^{(2)}\lp \sum_{k=2}^{3}b_k \g_{j_w}^{(k)}\rp
 \rightarrow  \sigma_{3,1;1} g_{f,1}^{(3,1)}+\sigma_{3,2;1} g_{f,1}^{(3,2)} +\mu_{3,2;1},
\end{eqnarray}
where $g_{f,1}^{(3,1)}$ and $g_{f,1}^{(3,2)}$ are independent standard normals. Due to symmetry, one analogously also has
 \begin{eqnarray}\label{eq:reluact49}
\sum_{j_w=\frac{d}{2}+1}^{d} \f^{(2)}\lp \g_{j_w}^{(x,2)}\rp  & \rightarrow & \sigma_{3,1;1} g_{f,2}^{(3,1)}+\sigma_{3,2;1} g_{f,2}^{(3,2)} +\mu_{3,2;1},
\end{eqnarray}
where $g_{f,2}^{(3,1)}$ and $g_{f,2}^{(3,2)}$ are independent standard normals (which are also independent of $g_{f,1}^{(3,1)}$ and $g_{f,1}^{(3,2)}$). It is then easy to observe that
 \begin{eqnarray}\label{eq:reluact50}
-\f^{(2)}\lp \g^{(x,2)} \rp^T\w = \sum_{j_w=1}^{\frac{d}{2}}\f^{(2)}\lp \g_{j_w}^{(x,2)}\rp - \sum_{j_w=\frac{d}{2}+1}^{d} \f^{(2)}\lp \g_{j_w}^{(x,2)}\rp
\longrightarrow \sqrt{2}\sigma_{3,1;1} g_{f}^{(3,1)}+\sqrt{2}\sigma_{3,2;1} g_{f}^{(3,2)},
\end{eqnarray}
where $g_{f}^{(3,1)}$ and $g_{f}^{(3,2)}$ are independent standard normals, where $g_{f}^{(3,1)}$ relates to the first part obtained by conditioning on $\g^{(3)}$ and the second part relates to the residual randomness over $\g^{(3)}$.

\noindent \red{\textbf{\emph{(ii) Handling $\mE_{\bar{\g}^{(3)}} \left \|  \frac{d\f^{(2)}\lp \g^{(x,2)} \rp }{d\g^{(x,2)}} \right \|_2^2$:}}}  Due to concentrations one has analogously to (\ref{eq:reluact26})
 \begin{eqnarray}\label{eq:reluact51}
\mE_{\bar{\g}^{(3)}} \left \|  \frac{d\f^{(2)}\lp \g^{(x,2)} \rp }{d\g^{(x,2)}} \right \|_2^2
=\mE_{\bar{\g}^{(3)}} \sum_{j_w=1}^{d}\lp \frac{d\f^{(2)}\lp \g_{j_w}^{(x,2)} \rp }{d\g_{j_w}^{(x,2)}}\rp^2
=  \sum_{j_w=1}^{d}\mE_{\g_{j_w}^{(x,2)}} h_s(\g_{j_w}^{(x,2)})=\frac{d}{2},
\end{eqnarray}
where $h_s(\cdot)$ is the unit step function.

One can now combine (\ref{eq:reluact39}), (\ref{eq:reluact44}), (\ref{eq:reluact48}), (\ref{eq:reluact50}), and (\ref{eq:reluact51}) to write
\begin{eqnarray}\label{eq:reluact52}
  z_i^{(2)}\lp   \bar{\g}^{(3)};\f^{(2)}\lp \q^{(net)} \rp   \rp
 &\rightarrow  &  \frac{ \lp \max\lp - \f^{(2)}\lp \g^{(x,2)} \rp^T\w,0\rp \rp^2}{\mE_{\bar{\g}^{(3)}} \left \|  \frac{d\f^{(2)}\lp \g^{(x,2)} \rp }{d\g^{(x,2)}} \right \|_2^2} \nonumber \\
  &\rightarrow  &  \lp \max\lp  \bar{b}_2g_{f}^{(3,1)}+\bar{b}_3g_{f}^{(3,2)},0\rp \rp^2,
 \end{eqnarray}
where
\begin{eqnarray}\label{eq:reluact53}
\bar{b}_2
& = &
\sqrt{\frac{d}{\mE_{\bar{\g}^{(3)}} \left \|  \frac{d\f^{(2)}\lp \g^{(x,2)} \rp }{d\g^{(x,2)}} \right \|_2^2
} \lp \mE_{\g_1^{(3)}} \mE_{\g_1^{(2)}} \lp \f^{(2)}\lp \sum_{k=2}^{3}b_k \g_1^{(k)}\rp \rp^2
 - \mE_{\g_1^{(3)}}\lp\mE_{\g_1^{(2)}}  \f^{(2)}\lp \sum_{k=2}^{3}b_k \g_1^{(k)}\rp \rp^2  \rp} \nonumber \\
 \nonumber \\
\bar{b}_3
& = & \sqrt{\frac{d}{\mE_{\bar{\g}^{(3)}} \left \|  \frac{d\f^{(2)}\lp \g^{(x,2)} \rp }{d\g^{(x,2)}} \right \|_2^2
} \lp \mE_{\g_1^{(3)}}\lp \mE_{\g_1^{(2)}}  \f^{(2)}\lp \sum_{k=2}^{3}b_k \g_1^{(k)}\rp \rp^2
 - \lp\mE_{\g_1^{(3)}} \mE_{\g_1^{(2)}}  \f^{(2)}\lp \sum_{k=2}^{3}b_k \g_1^{(k)}\rp \rp^2  \rp}, \nonumber \\
 \end{eqnarray}
which, after the utilization of  (\ref{eq:reluact51}), becomes
\begin{eqnarray}\label{eq:reluact53a0}
\bar{b}_2
 & = &
\sqrt{2 \lp \mE_{\g_1^{(3)}} \mE_{\g_1^{(2)}} \lp \f^{(2)}\lp \sum_{k=2}^{3}b_k \g_1^{(k)}\rp \rp^2
 - \mE_{\g_1^{(3)}}\lp\mE_{\g_1^{(2)}}  \f^{(2)}\lp \sum_{k=2}^{3}b_k \g_1^{(k)}\rp \rp^2  \rp}
\nonumber \\
\bar{b}_3
  & = & \sqrt{2 \lp \mE_{\g_1^{(3)}}\lp \mE_{\g_1^{(2)}}  \f^{(2)}\lp \sum_{k=2}^{3}b_k \g_1^{(k)}\rp \rp^2
 - \lp\mE_{\g_1^{(3)}} \mE_{\g_1^{(2)}}  \f^{(2)}\lp \sum_{k=2}^{3}b_k \g_1^{(k)}\rp \rp^2  \rp}.
 \end{eqnarray}
As mentioned earlier, $g_f^{(3,1)}$ relates to the randomness of $\g^{(2)}$ (i.e., ${\mathcal U}_2$) and
$g_f^{(3,2)}$ relates to the randomness of $\g^{(3)}$ (i.e., ${\mathcal U}_3$).

\noindent \red{\textbf{\emph{(iii) Specializing to $\f^{(2)}\lp \q^{(net)}\rp=\max\lp \q^{(net)},0\rp$:}}} We first observe
\begin{eqnarray}\label{eq:reluact54}
\bar{p}_1 =2\mE_{\g_1^{(3)}} \mE_{\g_1^{(2)}} \lp \f^{(2)}\lp \sum_{k=2}^{3}b_k \g_1^{(k)}\rp \rp^2=
 2\mE_{\g_1^{(3)}} \mE_{\g_1^{(2)}} \lp \max \lp \lp \sum_{k=2}^{3}b_k \g_1^{(k)}\rp,0\rp \rp^2=2\frac{1}{2}=1,
    \end{eqnarray}
and
\begin{eqnarray}\label{eq:reluact55}
\bar{p}_3 = 2\lp \mE_{\g_1^{(3)}} \mE_{\g_1^{(2)}}  \f^{(2)}\lp \sum_{k=2}^{3}b_k \g_1^{(k)}\rp \rp^2=
  2 \lp \mE_{\g_1^{(3)}} \mE_{\g_1^{(2)}}  \max \lp\lp \sum_{k=2}^{3}b_k \g_1^{(k)}\rp,0\rp \rp^2=2 \lp \frac{1}{2}\sqrt{\frac{2}{\pi}}\rp^2
  =\frac{1}{\pi}. \nonumber \\
 \end{eqnarray}
Then one also has
\begin{eqnarray}\label{eq:reluact56}
  \mE_{\g_1^{(2)}}  \f^{(2)}\lp \sum_{k=2}^{3}b_k \g_1^{(k)}\rp
&  = &
    \mE_{\g_1^{(2)}}  \max\lp \lp \sum_{k=2}^{3}b_k \g_1^{(k)}\rp,0 \rp
  =
    \mE_{\g_1^{(2)}}  \max\lp \lp \sqrt{1-\p_2} \g_1^{(2)} + \sqrt{\p_2} \g_1^{(3)}\rp,0 \rp \nonumber \\
& = &
\frac{1}{2} \lp \sqrt{\frac{2}{\pi}} \sqrt{1-\p_2} e^{-\frac{\lp \sqrt{\p_2}\g_1^{(3)}\rp.^2}{2 \sqrt{1-\p_2}^2}} + \sqrt{\p_2}\g_1^{(3)}\erf\lp \frac{\sqrt{\p_2}\g_1^{(3)}}{\sqrt{2}\sqrt{1-\p_2}}\rp + \sqrt{\p_2}\g_1^{(3)} \rp, \nonumber \\
  \end{eqnarray}
and
\begin{eqnarray}\label{eq:reluact57}
\bar{p}_2 & = &  2 \mE_{\g_1^{(3)}} \lp \mE_{\g_1^{(2)}}  \f^{(2)}\lp \sum_{k=2}^{3}b_k \g_1^{(k)}\rp  \rp^2 \nonumber \\
 & = &
\int_{-\infty}^{\infty} \lp \sqrt{\frac{2}{\pi}} \sqrt{1-\p_2} e^{-\frac{\lp \sqrt{\p_2}\g_1^{(3)}\rp.^2}{2 \sqrt{1-\p_2}^2}} + \sqrt{\p_2}\g_1^{(3)}\erf\lp \frac{\sqrt{\p_2}\g_1^{(3)}}{\sqrt{2}\sqrt{1-\p_2}}\rp + \sqrt{\p_2}\g_1^{(3)} \rp^2 \frac{e^{-\frac{\lp\g^{(3)}\rp^2}{2}}}{\sqrt{2\pi}}d\g^{(3)}. \nonumber \\
  \end{eqnarray}
Combining (\ref{eq:reluact53}), (\ref{eq:reluact54}), (\ref{eq:reluact55}), and (\ref{eq:reluact57}), we then find
\begin{eqnarray}\label{eq:reluact58}
\bar{b}_2
& = &
\sqrt{\bar{p}_1-\bar{p}_2}
\nonumber \\
\bar{b}_3
& = &
\sqrt{\bar{p}_2-\bar{p}_3},
 \end{eqnarray}
where, for the ReLU activations, $\bar{p}_1$, $\bar{p}_2$, and $\bar{p}_3$ are as in (\ref{eq:reluact54}), (\ref{eq:reluact55}), and (\ref{eq:reluact57}), respectively.

We can now return to the initial considerations, utilize (\ref{eq:reluact52}),(\ref{eq:reluact54}), (\ref{eq:reluact55}), (\ref{eq:reluact57}), and (\ref{eq:reluact58}), and rewrite (\ref{eq:reluact35}) as
\begin{eqnarray}\label{eq:reluact59}
    \bar{\psi}_{rd}^{(d,2)}(\p,\q,\c,\gamma_{sq},\gamma_{sq}^{(p)})
 & = &  \frac{1}{2}
(1-\p_2\q_2)\c_2
 -  \gamma_{sq}^{(p)}
-\Bigg(\Bigg. -\frac{1}{2\c_2} \log \lp \frac{2\gamma_{sq}-\c_2(1-\q_2)}{2\gamma_{sq}} \rp  \nonumber \\
 & & +  \frac{\q_2}{2(2\gamma_{sq}-\c_2(1-\q_2))}   \Bigg.\Bigg)
 \nonumber \\
& &   + \gamma_{sq}
 -\alpha\frac{1}{\c_2}\mE_{{\mathcal U}_3^{(j_w)}} \log\lp \mE_{{\mathcal U}_2^{(j_w)}} e^{-\c_2\frac{z_i^{(2)}\lp \bar{\g}^{(3)};\max\lp \q^{(net)},0\rp\rp }{4\gamma_{sq}}}\rp \nonumber \\
  & = &  \frac{1}{2}
(1-\p_2\q_2)\c_2
 -  \gamma_{sq}^{(p)}
-\Bigg(\Bigg. -\frac{1}{2\c_2} \log \lp \frac{2\gamma_{sq}-\c_2(1-\q_2)}{2\gamma_{sq}} \rp  \nonumber \\
 & & +  \frac{\q_2}{2(2\gamma_{sq}-\c_2(1-\q_2))}   \Bigg.\Bigg)
 \nonumber \\
& &   + \gamma_{sq}
 -\alpha\frac{1}{\c_2}\mE_{g_f^{(3,2)}} \log\lp \mE_{g_f^{(3,1)}} e^{-\c_2\frac{ \lp \max\lp  \bar{b}_2g_{f}^{(3,1)}+\bar{b}_3g_{f}^{(3,2)},0\rp \rp^2 }{4\gamma_{sq}}}\rp.
    \end{eqnarray}
One now observes a remarkable property of the above machinery. Namely, the expression in (\ref{eq:reluact59}) is structurally identical to the corresponding one given in equation (50) in \cite{Stojnicnegsphflrdt23}. The only difference is that now one has the adjusted values $\bar{b}_2=\sqrt{1-\bar{p}_2}$ and $\bar{b}_3=\sqrt{\bar{p}_2-\bar{p}_3}$. One can then solve the remaining integrals as in \cite{Stojnicnegsphflrdt23} and obtain
\begin{eqnarray}\label{eq:reluact60}
\hat{h} & = &  -\frac{\sqrt{\bar{p}_2-\bar{p}_3}g_f^{(3,2)}}{\sqrt{\bar{p}_1-\bar{p}_2}}    \nonumber \\
\hat{B} & = & \frac{\c_2}{4\gamma_{sq}} 
\nonumber \\
\hat{C} & = & \sqrt{\bar{p}_2-\bar{p}_3} g_f^{(3,2)} \nonumber \\
f_{(zd)}^{(2,f)}& = & \frac{e^{-\frac{\hat{B}\hat{C}^2}{2\bar{p}_11-\bar{p}_2)\hat{B} + 1}}}{2\sqrt{2(\bar{p}_1-\bar{p}_2)\hat{B} + 1}}
\erfc\lp\frac{\hat{h}}{\sqrt{4(\bar{p}_1-\bar{p}_2)\hat{B} + 2}}\rp
\nonumber \\
f_{(zu)}^{(2,f)}& = & \frac{1}{2}\erfc\lp-\frac{\hat{h}}{\sqrt{2}}\rp,  
\nonumber \\
f_{(zt)}^{(2,f)}& = & f_{(zd)}^{(2,f)}+f_{(zu)}^{(2,f)}.
   \end{eqnarray}
and
\begin{eqnarray}\label{eq:reluact61}
\mE_{g_f^{(3,2)}} \log\lp \mE_{g_f^{(3,1)}} e^{-\c_2\frac{ \lp \max\lp  \bar{b}_2g_{f}^{(3,1)}+\bar{b}_3g_{f}^{(3,2)},0\rp \rp^2 }{4\gamma_{sq}}}\rp=   \mE_{g_f^{(3,2)}} \log\lp f_{(zt)}^{(2,f)} \rp.
    \end{eqnarray}
As in \cite{Stojnicnegsphflrdt23}, one now needs to compute \emph{five} derivatives with respect to $\q_2$, $\p_2$, $\c_2$, $\gamma_{sq}$, and $\gamma_{sq}^{(p)}$. These are structurally identical (with a very minimal adjustment for $\bar{p}_3\neq 0$) to the corresponding ones computed in \cite{Stojnicnegsphflrdt23}. The only tiny structural difference is that for the $\p_2$ derivative, one needs to additionally trivially account for $\frac{d\bar{p}_2}{d\p_2}$. After computing the derivatives one then solves the following system of equations
\begin{eqnarray}\label{eq:reluact62}
  \frac{d\bar{\psi}_{rd}^{(d,2)}(\p,\q,\c,\gamma_{sq},\gamma_{sq}^{(p)}) }{d\q_2}
 & = &  0\nonumber \\
 \frac{d\bar{\psi}_{rd}^{(d,2)}(\p,\q,\c,\gamma_{sq},\gamma_{sq}^{(p)}) }{d\p_2}
 & = &  0 \nonumber \\
 \frac{d\bar{\psi}_{rd}^{(d,2)}(\p,\q,\c,\gamma_{sq},\gamma_{sq}^{(p)}) }{d\c_2}
 & = &  0\nonumber \\
 \frac{d\bar{\psi}_{rd}^{(d,2)}(\p,\q,\c,\gamma_{sq},\gamma_{sq}^{(p)}) }{d\gamma_{sq}^{(p)}}
 & = &  0\nonumber \\
 \frac{d\bar{\psi}_{rd}^{(d,2)}(\p,\q,\c,\gamma_{sq},\gamma_{sq}^{(p)}) }{d\gamma_{sq}}
 & = &  0,
      \end{eqnarray}
and denotes the obtained solution by $\hat{\q}_2,\hat{\p}_2,\hat{\c}_2,\hat{\gamma}_{sq}^{(p)},\hat{\gamma}_{sq}$. Due to structural identicalness between the derivatives considered here and those considered in \cite{Stojnicnegsphflrdt23}, one actually even has that the following closed form relations, established in \cite{Stojnicnegsphflrdt23}, actually hold here as well
\begin{eqnarray}
 \hat{\gamma}_{sq}^{(p)} &  =  &  \frac{1}{2}\frac{1-\hat{\q}_2}{1-\hat{\p}_2}
 \sqrt{\frac{\hat{\p}_2}{\hat{\q}_2}} \nonumber \\
  \hat{\c}_2 & = &  \frac{1}{1-\hat{\p}_2}
 \sqrt{\frac{\hat{\p}_2}{\hat{\q}_2}}- \frac{1}{1-\hat{\q}_2} \sqrt{\frac{\hat{\q}_2}{\hat{\p}_2}}.
   \label{eq:reluact63}
 \end{eqnarray}
Taking concrete numerical values for all the parameters, we then from  $\bar{\psi}_{rd}^{(d,2)}(\hat{\p},\hat{\q},\hat{\c},\hat{\gamma}_{sq},\hat{\gamma}_{sq}^{(p)})=0$ find for
 \begin{equation}\label{eq:reluact64}
\hspace{-2in}(\mbox{\bl{\textbf{\emph{full} second level:}}}) \qquad \qquad  a_c^{(2,f)}(\infty) =  \lim_{d\rightarrow\infty} a_c^{(2,f)}(d)  \approx \bl{\mathbf{2.6643}}.
  \end{equation}

\noindent \underline{\textbf{\emph{Concrete numerical values:}}}  In  Table \ref{tab:tab1}, the above $a_c^{(2,f)}(\infty)$ is complemented with the concrete values of all the relevant quantities related to the second \emph{full} (2-sfl RDT) level of lifting. Moreover, to enable a systematic view of the lifting progress,  the corresponding  quantities for the first \emph{full} (1-sfl RDT) and the second \emph{partial} (2-spf RDT) level are shown as well.
\begin{table}[h]
\caption{$r$-sfl RDT parameters; \textbf{\emph{ReLU}} activations -- \emph{wide} treelike net capacity;  $\hat{\c}_1\rightarrow 1$; $d\rightarrow\infty$; $n\rightarrow\infty$}\vspace{.1in}
\centering
\def\arraystretch{1.2}
\begin{tabular}{||l||c|c||c|c||c|c||c||c||}\hline\hline
 \hspace{-0in}$r$-sfl RDT                                             & $\hat{\gamma}_{sq}$    & $\hat{\gamma}_{sq}^{(p)}$    &  $\hat{\p}_2$ & $\hat{\p}_1$     & $\hat{\q}_2$  & $\hat{\q}_1$ &  $\hat{\c}_2$    & $\alpha_c^{(r)}(-1.5)$  \\ \hline\hline
$1$-sfl RDT                                      & $0.5$ & $0.5$ &  $0$  & $\rightarrow 1$   & $0$ & $\rightarrow 1$
 &  $\rightarrow 0$  & \bl{$\mathbf{2.9339}$} \\ \hline\hline
 $2$-spl RDT                                      & $0.3339$ & $0.7487$ &  $0$ & $\rightarrow 1$ &  $0$ & $\rightarrow 1$ &   $0.8295$   & \bl{$\mathbf{2.8503}$} \\ \hline
  $2$-sfl RDT                                      & $0.1396$  & $1.7903$  & $0.7571$ & $\rightarrow 1$ &  $0.3822$ & $\rightarrow 1$
 &  $4.6457$   & \bl{$\mathbf{2.6643}$}  \\ \hline\hline
  \end{tabular}
\label{tab:tab1}
\end{table}

We also observe that the capacity results shown in Table \ref{tab:tab1} exactly match the corresponding ones obtained using the statistical physics replica methods relying on the replica symmetry, partial 1rsb, and full 1rsb in \cite{ZavPeh21}. Moreover, the very same replica symmetry and 1rsb predictions were obtained in \cite{BalMalZech19} as well and are also exactly matched.

\subsubsection{$r=3$ -- third level of lifting}
\label{sec:thirdlev}

The main ideas behind the partial lifting were already presented in earlier sections. We here skip repeating such considerations and immediately  look at the \emph{full} third level of lifting. For $r=3$, one has that $\hat{\p}_1\rightarrow 1$ and $\hat{\q}_1\rightarrow 1$  as well as  $\hat{\p}_{r+1}=\hat{\p}_{4}=\hat{\q}_{r+1}=\hat{\q}_{4}=0$.  Analogously to (\ref{eq:negprac19}), (\ref{eq:reluact28a0}), and (\ref{eq:reluact35}), we then write
{\small\begin{align}\label{eq:reluact65}
    \bar{\psi}_{rd}^{(d,3)}(\p,\q,\c,\gamma_{sq},\gamma_{sq}^{(p)})   & =   \frac{1}{2}
(1-\p_2\q_2)\c_2+ \frac{1}{2}
(\p_2\q_2-\p_3\q_3)\c_3 \nonumber \\
& \quad  -  \gamma_{sq}^{(p)}  - \frac{1}{\c_3}\mE_{{\mathcal U}_4^{(j_w)}}\log\lp \mE_{{\mathcal U}_3^{(j_w)}} \lp \mE_{{\mathcal U}_2^{(j_w)}} e^{\c_2\frac{\lp\sqrt{1-\q_2}\h_1^{(2,1)} +\sqrt{\q_2-\q_3}\h_1^{(3,1)}+\sqrt{\q_3}\h_1^{(4,1)} \rp^2}{4 \gamma_{sq}^{(p)}}}\rp^{\frac{\c_3}{\c_2}}\rp \nonumber \\
& \quad   + \gamma_{sq}
 -\frac{\alpha}{\c_3}\mE_{{\mathcal U}_4^{(j_w)}} \log\lp \mE_{{\mathcal U}_3^{(j_w)}} \lp \mE_{{\mathcal U}_2^{(j_w)}} e^{-\c_2\frac{z_i^{(3)}\lp \bar{\g}^{(4)};\max\lp \q^{(net)},0\rp\rp}{4\gamma_{sq}}}\rp^{\frac{\c_3}{\c_2}}\rp \nonumber \\
 & =   \frac{1}{2}
(1-\p_2\q_2)\c_2+ \frac{1}{2}
(\p_2\q_2-\p_3\q_3)\c_3 -  \gamma_{sq}^{(p)} \nonumber \\
&\quad
-\Bigg(\Bigg. -\frac{1}{2\c_2} \log \lp \frac{2\gamma_{sq}^{(p)}-\c_2(1-\q_2)}{2\gamma_{sq}^{(p)}} \rp  -\frac{1}{2\c_3} \log \lp \frac{2\gamma_{sq}^{(p)}-\c_2(1-\q_2)-\c_3(\q_2-\q_3)}{2\gamma_{sq}^{(p)}-\c_2(1-\q_2)} \rp  \nonumber \\
& \quad +  \frac{\q_3}{2(2\gamma_{sq}^{(p)}-\c_2(1-\q_2)-\c_3(\q_2-\q_3))}   \Bigg.\Bigg)
 \nonumber \\
& \quad   + \gamma_{sq}
 -\frac{\alpha}{\c_3}\mE_{{\mathcal U}_4^{(j_w)}} \log\lp \mE_{{\mathcal U}_3^{(j_w)}} \lp \mE_{{\mathcal U}_2^{(j_w)}} e^{-\c_2\frac{z_i^{(3)}\lp \bar{\g}^{(4)};\max\lp \q^{(net)},0\rp\rp}{4\gamma_{sq}}}\rp^{\frac{\c_3}{\c_2}}\rp,
    \end{align}}

\noindent where the first sequence of integrals is handled through the utilization of the closed form solutions obtained in \cite{Stojnichopflrdt23,Stojnicnegsphflrdt23}. To be able to proceed with further analysis of (\ref{eq:reluact65}), we now again digress for a moment and look at the parts of the above expression that turn out to be particularly relevant.

 \noindent \red{\textbf{\emph{(i) Handling  $  z_i^{(3)}\lp   \bar{\g}^{(4)};\max\lp \q^{(net)} \rp   \rp
$:}}}  Repeating the reasoning between  (\ref{eq:reluact36}) and (\ref{eq:reluact52}), one can write the following, third level, analogue to (\ref{eq:reluact52})
\begin{eqnarray}\label{eq:reluact66}
  z_i^{(3)}\lp   \bar{\g}^{(4)};\f^{(2)}\lp \q^{(net)} \rp   \rp
 &\rightarrow  &  \frac{ \lp \max\lp - \f^{(2)}\lp \g^{(x,3)} \rp^T\w,0\rp \rp^2}{\mE_{\bar{\g}^{(4)}} \left \|  \frac{d\f^{(2)}\lp \g^{(x,3)} \rp }{d\g^{(x,3)}} \right \|_2^2} \nonumber \\
  &\rightarrow  &  \lp \max\lp  \bar{b}_2g_{f}^{(4,1)}+\bar{b}_3g_{f}^{(4,2)},0+\bar{b}_4g_{f}^{(4,3)},0\rp \rp^2,
 \end{eqnarray}
with
{\small\begin{eqnarray}\label{eq:reluact67}
\bar{b}_2
& = &
\sqrt{\frac{d}{\mE_{\bar{\g}^{(4)}} \left \|  \frac{d\f^{(2)}\lp \g^{(x,3)} \rp }{d\g^{(x,3)}} \right \|_2^2} \lp \mE_{\g_1^{(4)}} \mE_{\g_1^{(3)}} \mE_{\g_1^{(2)}} \lp \f^{(2)}\lp \sum_{k=2}^{4}b_k \g_1^{(k)}\rp \rp^2
 - \mE_{\g_1^{(4)}} \mE_{\g_1^{(3)}} \lp  \mE_{\g_1^{(2)}}  \f^{(2)}\lp \sum_{k=2}^{4}b_k \g_1^{(k)}\rp \rp^2  \rp}
\nonumber \\
 \bar{b}_3
& = & \sqrt{\frac{d}{\mE_{\bar{\g}^{(4)}} \left \|  \frac{d\f^{(2)}\lp \g^{(x,3)} \rp }{d\g^{(x,3)}} \right \|_2^2} \lp \mE_{\g_1^{(4)}}\mE_{\g_1^{(3)}} \lp  \mE_{\g_1^{(2)}}  \f^{(2)}\lp \sum_{k=2}^{4}b_k \g_1^{(k)}\rp \rp^2
 - \mE_{\g_1^{(4)}} \lp  \mE_{\g_1^{(3)}}  \mE_{\g_1^{(2)}}  \f^{(2)}\lp \sum_{k=2}^{4}b_k \g_1^{(k)}\rp \rp^2   \rp}
 \nonumber \\
 \bar{b}_4
& = & \sqrt{\frac{d}{\mE_{\bar{\g}^{(4)}} \left \|  \frac{d\f^{(2)}\lp \g^{(x,3)} \rp }{d\g^{(x,3)}} \right \|_2^2} \lp \mE_{\g_1^{(4)}}  \lp \mE_{\g_1^{(3)}}  \mE_{\g_1^{(2)}}  \f^{(2)}\lp \sum_{k=2}^{4}b_k \g_1^{(k)}\rp \rp^2
 -  \lp \mE_{\g_1^{(4)}} \mE_{\g_1^{(3)}} \mE_{\g_1^{(2)}}   \f^{(2)}\lp \sum_{k=2}^{4}b_k \g_1^{(k)}\rp \rp^2    \rp}, \nonumber \\
 \end{eqnarray}}
which after the utilization of
\begin{eqnarray}\label{eq:reluact67a0}
  \mE_{\bar{\g}^{(4)}} \left \|  \frac{d\f^{(2)}\lp \g^{(x,3)} \rp }{d\g^{(x,3)}} \right \|_2^2   =  \mE_{\bar{\g}^{(4)}} \left \|  \frac{d\max\lp \g^{(x,3)},0 \rp }{d\g^{(x,3)}} \right \|_2^2   =\frac{d}{2},
 \end{eqnarray}
becomes
{\small\begin{eqnarray}\label{eq:reluact67a1}
\bar{b}_2
 & = &
\sqrt{2 \lp \mE_{\g_1^{(4)}} \mE_{\g_1^{(3)}} \mE_{\g_1^{(2)}} \lp \f^{(2)}\lp \sum_{k=2}^{4}b_k \g_1^{(k)}\rp \rp^2
 - \mE_{\g_1^{(4)}} \mE_{\g_1^{(3)}} \lp  \mE_{\g_1^{(2)}}  \f^{(2)}\lp \sum_{k=2}^{4}b_k \g_1^{(k)}\rp \rp^2  \rp}
\nonumber \\
\bar{b}_3
 & = & \sqrt{2 \lp \mE_{\g_1^{(4)}}\mE_{\g_1^{(3)}} \lp  \mE_{\g_1^{(2)}}  \f^{(2)}\lp \sum_{k=2}^{4}b_k \g_1^{(k)}\rp \rp^2
 - \mE_{\g_1^{(4)}} \lp  \mE_{\g_1^{(3)}}  \mE_{\g_1^{(2)}}  \f^{(2)}\lp \sum_{k=2}^{4}b_k \g_1^{(k)}\rp \rp^2   \rp}
 \nonumber \\
\bar{b}_4
  & = & \sqrt{2 \lp \mE_{\g_1^{(4)}}  \lp \mE_{\g_1^{(3)}}  \mE_{\g_1^{(2)}}  \f^{(2)}\lp \sum_{k=2}^{4}b_k \g_1^{(k)}\rp \rp^2
 -  \lp \mE_{\g_1^{(4)}} \mE_{\g_1^{(3)}} \mE_{\g_1^{(2)}}   \f^{(2)}\lp \sum_{k=2}^{4}b_k \g_1^{(k)}\rp \rp^2    \rp},
 \end{eqnarray}}

 \noindent
where $b_2=\sqrt{1-\p_2}$, $b_3=\sqrt{\p_2-\p_3}$, and $b_4=\sqrt{\p_3}$, and, similarly to what we had earlier, $g_f^{(4,1)}$ relates to the randomness of $\g^{(2)}$ (i.e., ${\mathcal U}_2$),
$g_f^{(4,2)}$ to the randomness of $\g^{(3)}$ (i.e., ${\mathcal U}_3$), and $g_f^{(4,3)}$ to the randomness of $\g^{(4)}$ (i.e., ${\mathcal U}_4$).

\noindent \red{\textbf{\emph{(ii) Further specializing to $\f^{(2)}\lp \q^{(net)}\rp=\max\lp \q^{(net)},0\rp$:}}} We start by observing
\begin{eqnarray}\label{eq:reluact68}
\bar{p}_1 =2\mE_{\g_1^{(4)}} \mE_{\g_1^{(3)}} \mE_{\g_1^{(2)}} \lp \f^{(2)}\lp \sum_{k=2}^{4}b_k \g_1^{(k)}\rp \rp^2=
 2\mE_{\g_1^{(4)}}  \mE_{\g_1^{(3)}} \mE_{\g_1^{(2)}} \lp \max \lp \lp \sum_{k=2}^{4}b_k \g_1^{(k)}\rp,0\rp \rp^2=2\frac{1}{2}=1,\nonumber \\
    \end{eqnarray}
and
\begin{eqnarray}\label{eq:reluact69}
\bar{p}_4 & = & 2\lp \mE_{\g_1^{(4)}}  \mE_{\g_1^{(3)}} \mE_{\g_1^{(2)}}  \f^{(2)}\lp \sum_{k=2}^{4}b_k \g_1^{(k)}\rp \rp^2 \nonumber \\
& = &
  2 \lp \mE_{\g_1^{(4)}}  \mE_{\g_1^{(3)}} \mE_{\g_1^{(2)}}  \max \lp\lp \sum_{k=2}^{4}b_k \g_1^{(k)}\rp,0\rp \rp^2=2 \lp \frac{1}{2}\sqrt{\frac{2}{\pi}}\rp^2
  =\frac{1}{\pi}. \nonumber \\
 \end{eqnarray}
Then one also has
\begin{eqnarray}\label{eq:reluact70}
  \mE_{\g_1^{(3)}}  \mE_{\g_1^{(2)}}  \f^{(2)}\lp \sum_{k=2}^{4}b_k \g_1^{(k)}\rp
&  = &
  \mE_{\g_1^{(3)}}     \mE_{\g_1^{(2)}}  \max\lp \lp \sum_{k=2}^{4}b_k \g_1^{(k)}\rp,0 \rp \nonumber \\
&  = &
    \mE_{\g_1^{(3)}}    \mE_{\g_1^{(2)}}  \max\lp \lp \sqrt{1-\p_2} \g_1^{(2)} + \sqrt{\p_2-\p_3} \g_1^{(3)}+ \sqrt{\p_3} \g_1^{(4)}\rp,0 \rp \nonumber \\
& = &
\frac{1}{2} \lp \sqrt{\frac{2}{\pi}} \sqrt{1-\p_3} e^{-\frac{\lp \sqrt{\p_3}\g_1^{(4)}\rp^2}{2 \sqrt{1-\p_3}^2}} + \sqrt{\p_3}\g_1^{(4)}\erf\lp \frac{\sqrt{\p_3}\g_1^{(4)}}{\sqrt{2}\sqrt{1-\p_3}}\rp + \sqrt{\p_3}\g_1^{(4)} \rp, \nonumber \\
  \end{eqnarray}
and
\begin{eqnarray}\label{eq:reluact71}
\bar{p}_3  & = &  2 \mE_{\g_1^{(4)}} \lp \mE_{\g_1^{(3)}} \mE_{\g_1^{(2)}}  \f^{(2)}\lp \sum_{k=2}^{4}b_k \g_1^{(k)}\rp  \rp^2 \nonumber \\
 & = &
  \frac{1}{2}\int_{-\infty}^{\infty} \lp  \sqrt{\frac{2}{\pi}} \sqrt{1-\p_3} e^{-\frac{\lp \sqrt{\p_3}\g_1^{(4)}\rp^2}{2 \sqrt{1-\p_3}^2}} + \sqrt{\p_3}\g_1^{(4)}\erf\lp \frac{\sqrt{\p_3}\g_1^{(4)}}{\sqrt{2}\sqrt{1-\p_3}}\rp + \sqrt{\p_3}\g_1^{(4)}  \rp^2 \frac{e^{-\frac{\lp\g^{(4)}\rp^2}{2}}}{\sqrt{2\pi}}d\g^{(4)} \nonumber \\
 & \triangleq & \bar{p}_x (\p_3).
  \end{eqnarray}
Moreover, we also find
\begin{eqnarray}\label{eq:reluact72}
     \mE_{\g_1^{(2)}}  \f^{(2)}\lp \sum_{k=2}^{4}b_k \g_1^{(k)}\rp
&  = &
      \mE_{\g_1^{(2)}}  \max\lp \lp \sum_{k=2}^{4}b_k \g_1^{(k)}\rp,0 \rp \nonumber \\
&  = &
      \mE_{\g_1^{(2)}}  \max\lp \lp \sqrt{1-\p_2} \g_1^{(2)} + \sqrt{\p_2-\p_3} \g_1^{(3)}+ \sqrt{\p_3} \g_1^{(4)}\rp,0 \rp, \nonumber \\
  \end{eqnarray}
and
\begin{eqnarray}\label{eq:reluact73}
\bar{p}_2 & = &  2 \mE_{\g_1^{(4)}}  \mE_{\g_1^{(3)}} \lp \mE_{\g_1^{(2)}}  \f^{(2)}\lp \sum_{k=2}^{4}b_k \g_1^{(k)}\rp  \rp^2  =
\bar{p}_x(\p_2). \nonumber \\
  \end{eqnarray}
A combination of (\ref{eq:reluact67}), (\ref{eq:reluact68}), (\ref{eq:reluact69}), (\ref{eq:reluact71}), and (\ref{eq:reluact73}) then also gives
\begin{eqnarray}\label{eq:reluact74}
\bar{b}_2
& = &
\sqrt{\bar{p}_1-\bar{p}_2}
\nonumber \\
\bar{b}_3
& = &
\sqrt{\bar{p}_2-\bar{p}_3}
\nonumber \\
\bar{b}_4
& = &
\sqrt{\bar{p}_3-\bar{p}_4},
 \end{eqnarray}
where, for the ReLU activations, $\bar{p}_1$, $\bar{p}_2$, $\bar{p}_3$, and $\bar{p}_4$ are as in (\ref{eq:reluact68}), (\ref{eq:reluact69}), (\ref{eq:reluact71}), and (\ref{eq:reluact73}), respectively.

One can now return to the analysis of (\ref{eq:reluact65}) and rewrite it as
{\small\begin{align}\label{eq:reluact75}
    \bar{\psi}_{rd}^{(d,3)}(\p,\q,\c,\gamma_{sq},\gamma_{sq}^{(p)})
     & =   \frac{1}{2}
(1-\p_2\q_2)\c_2+ \frac{1}{2}
(\p_2\q_2-\p_3\q_3)\c_3 -  \gamma_{sq}^{(p)} \nonumber \\
&\quad
-\Bigg(\Bigg. -\frac{1}{2\c_2} \log \lp \frac{2\gamma_{sq}^{(p)}-\c_2(1-\q_2)}{2\gamma_{sq}^{(p)}} \rp  -\frac{1}{2\c_3} \log \lp \frac{2\gamma_{sq}^{(p)}-\c_2(1-\q_2)-\c_3(\q_2-\q_3)}{2\gamma_{sq}^{(p)}-\c_2(1-\q_2)} \rp  \nonumber \\
& \quad +  \frac{\q_3}{2(2\gamma_{sq}^{(p)}-\c_2(1-\q_2)-\c_3(\q_2-\q_3))}   \Bigg.\Bigg)
 \nonumber \\
& \quad   + \gamma_{sq}
 -\frac{\alpha}{\c_3}\mE_{{\mathcal U}_4^{(j_w)}} \log\lp \mE_{{\mathcal U}_3^{(j_w)}} \lp \mE_{{\mathcal U}_2^{(j_w)}} e^{-\c_2\frac{z_i^{(3)}\lp \bar{\g}^{(4)};\max\lp \q^{(net)},0\rp\rp}{4\gamma_{sq}}}\rp^{\frac{\c_3}{\c_2}}\rp \nonumber \\
 & =   \frac{1}{2}
(1-\p_2\q_2)\c_2+ \frac{1}{2}
(\p_2\q_2-\p_3\q_3)\c_3 -  \gamma_{sq}^{(p)} \nonumber \\
&\quad
-\Bigg(\Bigg. -\frac{1}{2\c_2} \log \lp \frac{2\gamma_{sq}^{(p)}-\c_2(1-\q_2)}{2\gamma_{sq}^{(p)}} \rp  -\frac{1}{2\c_3} \log \lp \frac{2\gamma_{sq}^{(p)}-\c_2(1-\q_2)-\c_3(\q_2-\q_3)}{2\gamma_{sq}^{(p)}-\c_2(1-\q_2)} \rp  \nonumber \\
& \quad +  \frac{\q_3}{2(2\gamma_{sq}^{(p)}-\c_2(1-\q_2)-\c_3(\q_2-\q_3))}   \Bigg.\Bigg)
 \nonumber \\
& \quad   + \gamma_{sq}
 -\frac{\alpha}{\c_3}
 \mE_{g_{f}^{(4,3)}} \log\lp \mE_{g_{f}^{(4,2)}} \lp \mE_{g_{f}^{(4,1)}} e^{-\c_2\frac{\lp \max\lp  \bar{b}_2g_{f}^{(4,1)}+\bar{b}_3g_{f}^{(4,2)},0+\bar{b}_4g_{f}^{(4,3)},0\rp \rp^2}{4\gamma_{sq}}}\rp^{\frac{\c_3}{\c_2}}\rp.
    \end{align}}

\noindent The very same remarkable property that was observed when considering the second level of lifting right after (\ref{eq:reluact59}), remains in place  and one again observes the structural identicalness between the expression in (\ref{eq:reluact75}) and the corresponding one given in equation (91) in \cite{Stojnicnegsphflrdt23}. As earlier, the adjusted values $\bar{b}_2=\sqrt{\bar{p}_1-\bar{p}_2}$, $\bar{b}_3=\sqrt{\bar{p}_2-\bar{p}_3}$, and  $\bar{b}_4=\sqrt{\bar{p}_3-\bar{p}_4}$ are the only difference. This then enables us to solve the remaining integrals as in \cite{Stojnicnegsphflrdt23} and write
\begin{eqnarray}\label{eq:reluact76}
\tilde{h} & = &  -\frac{\sqrt{\bar{p}_2-\bar{p}_3}g_{f}^{(4,2)}+\sqrt{\bar{p}_3-\bar{p}_4}g_{f}^{(4,3)} }{\sqrt{\bar{p}_1-\bar{p}_2}}    \nonumber \\
\tilde{B} & = & \frac{\c_2}{4\gamma_{sq}} 
\nonumber \\
\tilde{C} & = & \sqrt{\bar{p}_2-\bar{p}_3}g_{f}^{(4,2)}+\sqrt{\bar{p}_3-\bar{p}_4}g_{f}^{(4,3)} \nonumber \\
f_{(zd)}^{(3,f)}& = & \frac{e^{-\frac{\tilde{B}\tilde{C}^2}{2(\bar{p}_1-\bar{p}_2)\tilde{B} + 1}}}{2\sqrt{2(\bar{p}_1-\bar{p}_2)\tilde{B} + 1}}
\erfc\lp\frac{\tilde{h}}{\sqrt{4(\bar{p}_1-\bar{p}_2)\tilde{B} + 2}}\rp
\nonumber \\
f_{(zu)}^{(3,f)}& = & \frac{1}{2}\erfc\lp-\frac{\tilde{h}}{\sqrt{2}}\rp,  
\nonumber \\
f_{(zt)}^{(3,f)}& = & f_{(zd)}^{(3,f)}+f_{(zu)}^{(3,f)}.
   \end{eqnarray}
and
\begin{eqnarray}\label{eq:reluact77}
\mE_{g_{f}^{(4,3)}} \log\lp \mE_{g_{f}^{(4,2)}} \lp \mE_{g_{f}^{(4,1)}} e^{-\c_2\frac{\lp \max\lp  \bar{b}_2g_{f}^{(4,1)}+\bar{b}_3g_{f}^{(4,2)},0+\bar{b}_4g_{f}^{(4,3)},0\rp \rp^2}{4\gamma_{sq}}}\rp^{\frac{\c_3}{\c_2}}\rp
=   \mE_{g_{f}^{(4,3)}} \log\lp \mE_{g_{f}^{(4,2)}} \lp f_{(zt)}^{(3,f)} \rp^{\frac{\c_3}{\c_2}}\rp. \nonumber \\
    \end{eqnarray}
Combining (\ref{eq:reluact75}) and (\ref{eq:reluact77}), we obtain
\begin{align}\label{eq:reluact78}
    \bar{\psi}_{rd}(\p,\q,\c,\gamma_{sq},\gamma_{sq}^{(p)})
 & =    \frac{1}{2}
(1-\p_2\q_2)\c_2+ \frac{1}{2}
(\p_2\q_2-\p_3\q_3)\c_3 -  \gamma_{sq}^{(p)}
-\Bigg(\Bigg. -\frac{1}{2\c_2} \log \lp \frac{2\gamma_{sq}^{(p)}-\c_2(1-\q_2)}{2\gamma_{sq}^{(p)}} \rp
 \nonumber \\
& \quad -\frac{1}{2\c_3} \log \lp \frac{2\gamma_{sq}^{(p)}-\c_2(1-\q_2)-\c_3(\q_2-\q_3)}{2\gamma_{sq}^{(p)}-\c_2(1-\q_2)} \rp
 \nonumber \\
& \quad +  \frac{\q_3}{2(2\gamma_{sq}^{(p)}-\c_2(1-\q_2)-\c_3(\q_2-\q_3))}   \Bigg.\Bigg)
 + \gamma_{sq}
 -\frac{\alpha}{\c_3}   \mE_{g_{f}^{(4,3)}} \log\lp \mE_{g_{f}^{(4,2)}} \lp f_{(zt)}^{(3,f)} \rp^{\frac{\c_3}{\c_2}}\rp.
    \end{align}

Following what we presented in the earlier sections when we discussed the second level of lifting and relying on \cite{Stojnicnegsphflrdt23}, one proceeds by computing the \emph{eight} derivatives with respect to $\q_2$, $\q_3$, $\p_2$, $\p_3$, $\c_2$, $\c_3$, $\gamma_{sq}$, and $\gamma_{sq}^{(p)}$. Keeping in mind very minimal adjustments for $\bar{p}_4\neq 0$, the resulting derivatives are structurally identical to the corresponding ones from \cite{Stojnicnegsphflrdt23}. The only small structural difference is that for the $\p_2$ and $\p_3$ derivatives, a trivial additional accounting for $\frac{d\bar{p}_2}{d\p_2}$, and $\frac{d\bar{p}_3}{d\p_3}$ is needed as well. After computing all the derivatives, one then solves the following system of equations
\begin{eqnarray}\label{eq:reluact79}
  \frac{d\bar{\psi}_{rd}^{(d,3)}(\p,\q,\c,\gamma_{sq},\gamma_{sq}^{(p)}) }{d\q_2}
 & = &
   \frac{d\bar{\psi}_{rd}^{(d,3)}(\p,\q,\c,\gamma_{sq},\gamma_{sq}^{(p)}) }{d\q_3}
=0
\nonumber \\
 \frac{d\bar{\psi}_{rd}^{(d,3)}(\p,\q,\c,\gamma_{sq},\gamma_{sq}^{(p)}) }{d\p_2}
 & = &
  \frac{d\bar{\psi}_{rd}^{(d,3)}(\p,\q,\c,\gamma_{sq},\gamma_{sq}^{(p)}) }{d\p_3}
= 0
\nonumber \\
 \frac{d\bar{\psi}_{rd}^{(d,3)}(\p,\q,\c,\gamma_{sq},\gamma_{sq}^{(p)}) }{d\c_2}
 & = &
  \frac{d\bar{\psi}_{rd}^{(d,3)}(\p,\q,\c,\gamma_{sq},\gamma_{sq}^{(p)}) }{d\c_3}
=0
 \nonumber \\
 \frac{d\bar{\psi}_{rd}^{(d,3)}(\p,\q,\c,\gamma_{sq},\gamma_{sq}^{(p)}) }{d\gamma_{sq}^{(p)}}
 & = &   \frac{d\bar{\psi}_{rd}^{(d,3)}(\p,\q,\c,\gamma_{sq},\gamma_{sq}^{(p)}) }{d\gamma_{sq}}
  = 0,
      \end{eqnarray}
and denotes the obtained solution by $\hat{\q}_2,\hat{\q}_3,\hat{\p}_2,\hat{\p}_3,\hat{\c}_2,\hat{\c}_3,\hat{\gamma}_{sq}^{(p)},\hat{\gamma}_{sq}$. Moreover, the structural identicalness  also ensures that the following closed form relations, established in \cite{Stojnicnegsphflrdt23}, actually hold here as well
\begin{eqnarray}
 \hat{\gamma}_{sq}^{(p)} &  =  &  \frac{1}{2}\frac{1-\hat{\q}_2}{1-\hat{\p}_2}
 \frac{\hat{\p}_2-\hat{\p}_3}{\hat{\q}_2-\hat{\q}_3}\sqrt{\frac{\hat{\q}_3}{\hat{\p}_3}} \nonumber \\
\hat{\c}_3 & = & \frac{1}{\hat{\p}_2-\hat{\p}_3}\sqrt{\frac{\hat{\p}_3}{\hat{\q}_3}} -\frac{1}{\hat{\q}_2-\hat{\q}_3}\sqrt{\frac{\hat{\q}_3}{\hat{\p}_3}} \nonumber \\
 \hat{\c}_2 & = &  \frac{1}{1-\hat{\p}_2}
 \frac{\hat{\p}_2-\hat{\p}_3}{\hat{\q}_2-\hat{\q}_3}\sqrt{\frac{\hat{\q}_3}{\hat{\p}_3}}- \frac{1}{1-\hat{\q}_2}\frac{\hat{\q}_2-\hat{\q}_3}{\hat{\p}_2-\hat{\p}_3}\sqrt{\frac{\hat{\p}_3}{\hat{\q}_3}}.
   \label{eq:reluact80}
 \end{eqnarray}
After, taking concrete numerical values for all the considered parameters, $\q_2$, $\q_3$, $\p_2$, $\p_3$, $\c_2$, $\c_3$, $\gamma_{sq}$, and $\gamma_{sq}^{(p)}$, we then, from  $\bar{\psi}_{rd}^{(d,3)}(\hat{\p},\hat{\q},\hat{\c},\hat{\gamma}_{sq},\hat{\gamma}_{sq}^{(p)})=0$, obtain for
 \begin{equation}\label{eq:reluact81}
\hspace{-2in}(\mbox{\bl{\textbf{\emph{full} third level:}}}) \qquad \qquad  a_c^{(3,f)}(\infty) =  \lim_{d\rightarrow\infty} a_c^{(3,f)}(d)  \approx \bl{\mathbf{2.6534}}.
  \end{equation}

\noindent \underline{\textbf{\emph{Concrete numerical values:}}}  In  Table \ref{tab:tab2}, the concrete values of all the relevant quantities related to the third \emph{full} (3-sfl RDT) level of lifting, complement the above $a_c^{(3,f)}(\infty)$. A systematic view of the lifting progress is also enabled, by showing in parallel the corresponding quantities for the first \emph{full} (1-sfl RDT) and the second \emph{partial} (2-spl RDT) and \emph{full} (2-sfl RDT) levels  as well.
\begin{table}[h]
\caption{$r$-sfl RDT parameters; \textbf{\emph{ReLU}} activations -- \emph{wide} treelike net capacity;  $\hat{\c}_1\rightarrow 1$; $d\rightarrow\infty$; $n\rightarrow\infty$}\vspace{.1in}
\centering
\def\arraystretch{1.2}
{\small
\begin{tabular}{||l||c|c||c|c|c||c|c|c||c|c||c||}\hline\hline
 \hspace{-0in}$r$-sfl RDT                                             & $\hat{\gamma}_{sq}$    & $\hat{\gamma}_{sq}^{(p)}$    &  $\hat{\p}_3$  &  $\hat{\p}_2$ & $\hat{\p}_1$    &  $\hat{\q}_3$   & $\hat{\q}_2$  & $\hat{\q}_1$ &  $\hat{\c}_3$ &  $\hat{\c}_2$    & $\alpha_c^{(r)}(\infty)$  \\ \hline\hline
$1$-sfl RDT                                      & $0.5$ & $0.5$ &  $0$ &  $0$  & $\rightarrow 1$  &  $0$ & $0$ & $\rightarrow 1$ &  $0$
 &  $\rightarrow 0$  & \bl{$\mathbf{2.9339}$} \\ \hline\hline
 $2$-spl RDT                                      & $0.3339$ & $0.7487$ &  $0$ &  $0$ & $\rightarrow 1$ &  $0$ &  $0$ & $\rightarrow 1$ &  $0$ &  $0.8295$   & \bl{$\mathbf{2.8503}$} \\ \hline
  $2$-sfl RDT                                      & $0.1396$  & $1.7903$ &  $0$ & $0.7571$ & $\rightarrow 1$ &  $0$ &  $0.3822$ & $\rightarrow 1$
 &  $0$ &  $4.6457$   & \bl{$\mathbf{2.6643}$}  \\ \hline\hline
   $3$-sfl RDT                                      & $0.0786$  & $3.1858$ &  $0.6961$ & $0.9756$ & $\rightarrow 1$ &  $0.3331$ &  $0.7026$ & $\rightarrow 1$
&  $3.3$ &  $15$   & \bl{$\mathbf{2.6534}$}  \\ \hline\hline
  \end{tabular}
  }
\label{tab:tab2}
\end{table}

\subsubsection{General $r$-th level of lifting}
\label{sec:rlev}

It is clear from Table \ref{tab:tab2} that the convergence of the lifting mechanism is rather rapid with the concrete results showing, already on the third level, relative improvements no better than $\sim 0.1\% $. Doing further evaluations on higher levels is therefore practically not necessarily needed. For the completeness, we however, formalize below the general $r$-level ($r\geq 2$) results. In particular, analogously to (\ref{eq:reluact59}) and (\ref{eq:reluact75}), we have
 \begin{equation}\label{eq:reluact82}
    \bar{\psi}_{rd}^{(d,r)}(\p,\q,\c,\gamma_{sq},\gamma_{sq}^{(p)})
   =    \frac{1}{2}
\sum_{k=2}^{r}(\p_{k-1}\q_{k-1}-\p_k\q_k)\c_k  - {\mathcal I}_{sph}^{(r)}  +{\mathcal I}_{net}^{(r)},
\end{equation}
 where
 {\small\begin{align}\label{eq:reluact83}
 {\mathcal I}_{sph}^{(r)} & = \gamma_{sq}^{(p)}
+\Bigg(\Bigg. -\sum_{k=2}^{r}\frac{1}{2\c_k} \log \lp \frac{\Theta_{k}}{\Theta_{k-1}}\rp +\frac{\p_r}{2\Theta_{r}}  \Bigg.\Bigg) \nonumber \\
{\mathcal I}_{net}^{(r)} & = \gamma_{sq}
\nonumber \\
& \quad  -\frac{\alpha}{\c_r}\mE_{g_{f}^{(r+1,r)}}\log\lp \mE_{g_{f}^{(r+1,r-1)}} \lp \dots \lp \mE_{g_{f}^{(r+1,2)}} \lp  \mE_{g_{f}^{(r+1,1)}}
   e^{-\frac{\c_2\lp \max\lp  \sum_{k=2}^{r+1} \bar{b}_kg_{f}^{(r+1,k-1)},0\rp \rp^2}{4\gamma_{sq}}}
   \rp^{\frac{\c_3}{\c_2}} \rp^{\frac{\c_4}{\c_{3}}} \dots \rp^{\frac{\c_r}{\c_{r-1}}} \rp,\nonumber \\
    \end{align}}
and
\begin{eqnarray}\label{eq:reluact85}
 \Theta_{1} & = & 2\gamma_{sq}^{(p)} \nonumber \\
 \Theta_{k} & = & \Theta_{k-1}-\c_k(\p_{k-1}-\p_k), k\in\{2,3,\dots,r\} \nonumber \\
 \bar{b}_k & = &\sqrt{\bar{p}_{k-1}-\bar{p}_k}, k\in\{2,3,\dots,r\}  \nonumber \\.
\bar{p}_k & = &\bar{p}_x(\p_k), k\in\{1,2,3,\dots,r+1\},
 \end{eqnarray}
with additionally noting that $\bar{p}_x(\cdot)$ is as given in (\ref{eq:reluact71}) and $\bar{p}_1=1 = \bar{p}_x(1)= \bar{p}_x(\p_1)$  and $\bar{p}_{r+1}  = \frac{1}{\pi}=\bar{p}_x(0)=\bar{p}_x(\p_{r+1})$. One then solves the following system with $(3(r-1)+2)$ unknowns
\begin{eqnarray}\label{eq:reluact86}
  \frac{d\bar{\psi}_{rd}^{(d,r)}(\p,\q,\c,\gamma_{sq},\gamma_{sq}^{(p)}) }{d\q}
 & = &
 0
\nonumber \\
 \frac{d\bar{\psi}_{rd}^{(d,r)}(\p,\q,\c,\gamma_{sq},\gamma_{sq}^{(p)}) }{d\p}
 & = &
  0
\nonumber \\
 \frac{d\bar{\psi}_{rd}^{(d,r)}(\p,\q,\c,\gamma_{sq},\gamma_{sq}^{(p)}) }{d\c}
 & = &
 0
 \nonumber \\
 \frac{d\bar{\psi}_{rd}^{(d,r)}(\p,\q,\c,\gamma_{sq},\gamma_{sq}^{(p)}) }{d\gamma_{sq}^{(p)}}
 & = &  0 \nonumber \\
   \frac{d\bar{\psi}_{rd}^{(d,r)}(\p,\q,\c,\gamma_{sq},\gamma_{sq}^{(p)}) }{d\gamma_{sq}}
 & = & 0,
      \end{eqnarray}
and denotes by $\hat{\q},\hat{\p},\hat{\c},\hat{\gamma}_{sq}^{(p)},\hat{\gamma}_{sq}$ the obtained solution. Moreover, from Theorem 3 in \cite{Stojnicnegsphflrdt23}, one has the following remarkable, closed form relations among the obtained parameters
 \begin{eqnarray}
 \hat{\gamma}_{sq}^{(p)}&  = & \frac{1}{2}\frac{\hat{\q}_1-\hat{\q}_2}{\hat{\p}_{1}-\hat{\p}_2}  \prod_{k=2:2:r-1} \frac{ \hat{\p}_{k}-\hat{\p}_{k+1} }{ \hat{\q}_{k}-\hat{\q}_{k+1}} \prod_{k=2:2:r-2}  \frac{\hat{\q}_{k+1}-\hat{\q}_{k+2} }{ \hat{\p}_{k+1}-\hat{\p}_{k+2}}
 \sqrt{\lp \frac{\hat{\q}_r}{\hat{\p}_r}\rp^{(-1)^{r+1}}} \nonumber \\
 \hat{\c}_i &  = &
\frac{1}{\hat{\p}_{i-1}-\hat{\p}_i}  \prod_{k=i:2:r-1} \frac{ \hat{\p}_{k}-\hat{\p}_{k+1} }{ \hat{\q}_{k}-\hat{\q}_{k+1}} \prod_{k=i:2:r-2}  \frac{ \hat{\q}_{k+1}-\hat{\q}_{k+2} }{ \hat{\p}_{k+1}-\hat{\p}_{k+2}}
  \sqrt{\lp \frac{\hat{\q}_r}{\hat{\p}_r}\rp^{(-1)^i}}
\nonumber \\
& &  -
\frac{1}{\hat{\q}_{i-1}-\hat{\q}_i}  \prod_{k=i:2:r-1} \frac{ \hat{\q}_{k}-\hat{\q}_{k+1} }{ \hat{\p}_{k}-\hat{\p}_{k+1}} \prod_{k=i:2:r-2}  \frac{ \hat{\p}_{k+1}-\hat{\p}_{k+2} }{ \hat{\q}_{k+1}-\hat{\q}_{k+2}}
  \sqrt{\lp \frac{\hat{\p}_r}{\hat{\q}_r}\rp^{(-1)^i}}, \quad \mbox{with}\quad i\in\{2,3,\dots,r\}. \nonumber \\
   \label{eq:reluact87}
 \end{eqnarray}
Finally, from $\bar{\psi}_{rd}^{(d,r)}(\hat{\p},\hat{\q},\hat{\c},\hat{\gamma}_{sq},\hat{\gamma}_{sq}^{(p)})=0$, one determines $\alpha_c^{(r,f)}(\infty)$.

\subsubsection{Modulo-$\m$ sfl RDT}
\label{sec:posmodm}

We should also add that everything presented above can be repeated  while utilizing the modulo-m sfl RDT as discussed in \cite{Stojnicnegsphflrdt23,Stojnicbinperflrdt23,Stojnichopflrdt23}. Instead of Theorem \ref{thm:thm2}, one then basically has the following theorem.

 \begin{theorem}
  \label{thm:thm2modm}
    Assume the setup of  Lemma \ref{lemma:lemma1} and Theorems \ref{thm:thm1} and Theorem \ref{thm:thm2} and instead of the complete, assume the modulo-$\m$ sfl RDT setup of \cite{Stojnicsflgscompyx23}. Let the ``fixed'' parts of $\hat{\p}$, $\hat{\q}$, and $\hat{\c}$ satisfy $\hat{\p}_1\rightarrow 1$, $\hat{\q}_1\rightarrow 1$, $\hat{\c}_1\rightarrow 1$, $\hat{\p}_{r+1}=\hat{\q}_{r+1}=\hat{\c}_{r+1}=0$, and let the ``non-fixed'' parts of $\hat{\p}_k$, $\hat{\q}_k$, and $\hat{\c}_k$ ($k\in\{2,3,\dots,r\}$) be the solutions of the following system of equations
  \begin{eqnarray}\label{eq:modmnegthmprac1eq1}
   \frac{d \bar{\psi}_{rd}^{(d)}(\p,\q,\c,\gamma_{sq},\gamma_{sq}^{(p)})}{d\p} & = &  0 \nonumber \\
   \frac{d \bar{\psi}_{rd}^{(d)}(\p,\q,\c,\gamma_{sq},\gamma_{sq}^{(p)})}{d\q} & = &  0 \nonumber \\
   \frac{d \bar{\psi}_{rd}^{(d)}(\p,\q,\c,\gamma_{sq},\gamma_{sq}^{(p)})}{d\gamma_{sq}} & = &  0\nonumber \\
   \frac{d \bar{\psi}_{rd}^{(d)}(\p,\q,\c,\gamma_{sq},\gamma_{sq}^{(p)})}{d\gamma_{sq}^{(p)}} & = &  0,
 \end{eqnarray}
 and, consequently, let
\begin{eqnarray}\label{eq:modmprac17}
c_k(\hat{\p},\hat{\q})  & = & \sqrt{\hat{\q}_{k-1}-\hat{\q}_k} \nonumber \\
b_k(\hat{\p},\hat{\q})  & = & \sqrt{\hat{\p}_{k-1}-\hat{\p}_k}.
 \end{eqnarray}
 Then
\begin{eqnarray}\label{eq:modmthm2negprac13}
    \bar{\psi}_{rd}^{(d)}(\hat{\p},\hat{\q},\c,\gamma_{sq},\gamma_{sq}^{(p)})   & \leq &  \max_{\c}\frac{1}{2}    \sum_{k=2}^{r+1}\Bigg(\Bigg.
   \hat{\p}_{k-1}\hat{\q}_{k-1}
   -\hat{\p}_{k}\hat{\q}_{k}
  \Bigg.\Bigg)
\c_k
 -\gamma_{sq}^{(p)} - \varphi(D_{1,1}^{(per)}(c_k(\hat{\p},\hat{\q})),\c) \nonumber \\
& &
+\gamma_{sq}- \alpha\varphi(-D_1^{(net)}(b_k(\hat{\p},\hat{\q})),\c)\triangleq      \bar{\psi}_{rd,m}^{(d)}(\hat{\p},\hat{\q},\hat{\c},\hat{\gamma}_{sq},\hat{\gamma}_{sq}^{(p)}) ,\nonumber \\
  \end{eqnarray}
and
\begin{align}
\hspace{-.0in} \bar{\psi}_{rd,m}^{(d)}(\hat{\p},\hat{\q},\hat{\c},\hat{\gamma}_{sq},\hat{\gamma}_{sq}^{(p)}) \leq 0 & \Longrightarrow  \lp \lim_{n\rightarrow\infty}\mP_{X}(f_{rp}(X)>0)\longrightarrow 0 \rp
 \nonumber \\
& \Longleftrightarrow  \lp \lim_{n\rightarrow\infty}\mP_{X}(A([n,d,1];\f^{(2)}) \quad \mbox{fails to memorize data set} \quad (X,\1))\longrightarrow 0\rp. \nonumber \\
\label{eq:modmthm2ta17}
\end{align}
\end{theorem}
\begin{proof}
Follows directly from Lemma \ref{lemma:lemma1}, Theorems \ref{thm:thm1} and \ref{thm:thm2}, and the sfl RDT machinery presented in \cite{Stojnicnflgscompyx23,Stojnicsflgscompyx23,Stojnicflrdt23}.
\end{proof}

We have done the numerical evaluations utilizing the above theorem as well and obtained exactly the same results as in Table \ref{tab:tab2}. This basically indicates that the  above mentioned \emph{stationarity} over $\c$ is of the \emph{maximization} type, precisely as observed in \cite{Stojnicnegsphflrdt23,Stojnicbinperflrdt23,Stojnichopflrdt23}. Moreover, we maintained this practice in all of the calculations related to different activation functions that we present below and observed the very same outcome.

\subsection{Quadratic activations}
\label{sec:quad}

We now consider the well known \emph{quadratic}  activation. This means that we now assume that the neuronal function in the hidden layer is the following
\begin{equation}
\hspace{-1.5in} \mbox{\textbf{\bl{\emph{Quadratic}} activation:}} \hspace{1in} \f^{(2)}(\x)=\x^2.
\label{eq:quadact1}
\end{equation}
We will again start by looking at the first level of lifting. As all the key results obtained in the previous sections hold for generic activations, we are here in position to heavily utilize them and consequently proceed at a much faster pace.

\subsubsection{$r=1$ -- first level of lifting}
\label{sec:firstlevquad}

We first observe that from (\ref{eq:negprac20a0}) one has
 \begin{eqnarray}\label{eq:quadact0a0}
a_c^{(1)}(d)
& = &   \frac{1}{\mE_{{\mathcal U}_2^{(j_w)}}  \lp z_i^{(1)}\lp\g^{(2)};\f^{(2)}\lp \q^{(net)} \rp\rp\rp}
  =   \frac{1}{\mE_{\g^{(2)}}  \lp z_i^{(1)}\lp \bar{\g}^{(2)}; \lp\q^{(net)}\rp^2\rp\rp}.
  \end{eqnarray}
  Analogously to (\ref{eq:reluact12}), one now also finds
    \begin{eqnarray}\label{eq:quadact12}
\lim_{d\rightarrow \infty}\mE_{\g^{(2)}} z_i^{(1)}\lp   \bar{\g}^{(2)};\f^{(2)}\lp \q^{(net)} \rp   \rp
 & \rightarrow &
\frac{\lim_{d\rightarrow \infty}\mE_{\g^{(2)}} \lp \max\lp - \f^{(2)}\lp \g^{(2)} \rp^T\w,0\rp \rp^2}{\lim_{d\rightarrow \infty}\mE_{\g^{(2)}} \left \|  \frac{d\f^{(2)}\lp \g^{(2)} \rp }{d\g^{(2)}} \right \|_2^2}.
 \end{eqnarray}

\noindent \red{\textbf{\emph{(i) Handling $\mE_{\g^{(2)}} \lp \max\lp - \f^{(2)}\lp \g^{(2)} \rp^T\w,0\rp \rp^2$:}}} From (\ref{eq:reluact2})), we also have
 \begin{eqnarray}\label{eq:quadact20}
\mE_{\g^{(2)}} \lp \max\lp - \f^{(2)}\lp \g^{(2)} \rp^T\w,0\rp \rp^2
=\mE_{g_c^{(2)}} \lp \max\lp g_c^{(2)},0\rp \rp^2=\frac{1}{2}\sigma_{2}^2,
\end{eqnarray}
where recalling on (\ref{eq:quadact1}) and analogously to   (\ref{eq:reluact21}), (\ref{eq:reluact22}), and (\ref{eq:reluact23})
\begin{eqnarray}
\mE \lp \f^{(2)}\lp \g_{1}^{(2)}\rp \rp &= &\mE \lp\g_1^{(2)}\rp^2 =1 \nonumber \\
 \mE \lp \f^{(2)}\lp \g_{1}^{(2)}\rp \rp^2 &= &\mE \lp  \g_1^{(2)}\rp^4 = 3,
\label{eq:quadact22}
\end{eqnarray}
and
 \begin{eqnarray}\label{eq:quadact23}
 \sigma_2^2=d\lp\mE \lp\f^{(2)}\lp \g_{1}^{(2)} \rp \rp^2- \lp \mE\lp \f^{(2)}\lp \g_{1}^{(2)}\rp\rp\rp^2\rp
 =2d.
\end{eqnarray}
From  (\ref{eq:quadact20}) and (\ref{eq:quadact23}), we find
 \begin{eqnarray}\label{eq:quadact24}
\mE_{\g^{(2)}} \lp \max\lp - \f^{(2)}\lp \g^{(2)} \rp^T\w,0\rp \rp^2
 =\frac{1}{2}\sigma_{2}^2=d.
\end{eqnarray}

\noindent \red{\textbf{\emph{(ii) Handling $\mE_{\g^{(2)}} \left \|  \frac{d\f^{(2)}\lp \g^{(2)} \rp }{d\g^{(2)}} \right \|_2^2$:}}} Recalling again on (\ref{eq:quadact1}), we find
 \begin{eqnarray}\label{eq:quadact25}
\frac{d\f^{(2)}\lp \g_{j_w}^{(2)} \rp }{d\g_{j_w}^{(2)}}= 2\g_{j_w}^{(2)}.
\end{eqnarray}
One then also has
 \begin{eqnarray}\label{eq:quadact26}
\mE_{\g^{(2)}} \left \|  \frac{d\f^{(2)}\lp \g^{(2)} \rp }{d\g^{(2)}} \right \|_2^2
=\mE_{\g^{(2)}} \sum_{j_w=1}^{d}\lp \frac{d\f^{(2)}\lp \g_{j_w}^{(2)} \rp }{d\g_{j_w}^{(2)}}\rp^2
=  4\sum_{j_w=1}^{d}\mE_{\g_{j_w}^{(2)}} \lp\g_{j_w}^{(2)}\rp^2=4d.
\end{eqnarray}
 Utilizing the above observations, and combining (\ref{eq:quadact12}), (\ref{eq:quadact24}), and (\ref{eq:quadact26}), one finds
 \begin{eqnarray}\label{eq:quadact27}
\lim_{d\rightarrow\infty}\mE_{\g^{(2)}} z_i^{(1)}\lp   \bar{\g}^{(2)};\f^{(2)}\lp \q^{(net)} \rp   \rp
 & =  &  \lim_{d\rightarrow\infty}\frac{\mE_{\g^{(2)}} \lp \max\lp - \f^{(2)}\lp \g^{(2)} \rp^T\w,0\rp \rp^2}{\mE_{\g^{(2)}} \left \|  \frac{d\f^{(2)}\lp \g^{(2)} \rp }{d\g^{(2)}} \right \|_2^2}=\frac{1}{4}.
 \end{eqnarray}
A further combination of (\ref{eq:negprac20a0}), (\ref{eq:quadact0a0}), and (\ref{eq:quadact27}) then gives
\begin{eqnarray}\label{eq:quadact28}
\hspace{-0in}(\mbox{\bl{\textbf{first level:}}}) \qquad  a_c^{(1)}(\infty)
& = &    \lim_{d\rightarrow\infty} a_c^{(1)}(d) =  \frac{1}{\lim_{d\rightarrow\infty} \mE_{{\mathcal U}_2^{(j_w)}}  \lp z_i^{(1)}\lp  \bar{\g}^{(2)};\f^{(2)}\lp \q^{(net)} \rp\rp\rp}  \nonumber \\
 & = &   \frac{1}{\lim_{d\rightarrow\infty} \mE_{\g^{(2)}}  \lp z_i^{(1)}\lp   \bar{\g}^{(2)};\lp\q^{(net)}\rp^2\rp\rp}
  = \bl{\mathbf{4}}.
  \end{eqnarray}

\subsubsection{$r=2$ -- second level of lifting}
\label{sec:secondlev}

As was the case when we considered the ReLU activations in the previous sections, we again split the analysis of the second level of lifting  into two separate parts: (i) \emph{partial} second level of lifting; and (ii) \emph{full} second level of lifting.

\subsubsubsection{Partial second level of lifting}
\label{sec:secondlevparquad}

Analogously to (\ref{eq:reluact28a0}), we first have
 \begin{align}\label{eq:quadact28a0}
    \bar{\psi}_{rd}^{(d,2)}(\hat{\p},\hat{\q},\c,\gamma_{sq},\gamma_{sq}^{(p)})
     & =   \frac{1}{2}
\c_2
      -\gamma_{sq}^{(p)}  +\frac{1}{2\c_2}\log\lp \frac{2\gamma_{sq}^{(p)}-\c_2}{2\gamma_{sq}^{(p)}}\rp \nonumber \\
      &\quad
  + \gamma_{sq}
- \alpha\frac{1}{\c_2}\log\lp \mE_{\bar{\g}^{(3)}} e^{-\c_2\frac{z_i^{(2)}\lp \bar{\g}^{(3)};\lp \q^{(net)}\rp^2 \rp}{4\gamma_{sq}}}\rp,
    \end{align}
 and then analogously to (\ref{eq:reluact30})
 \begin{eqnarray}\label{eq:quadact30}
  z_i^{(2)}\lp   \bar{\g}^{(3)};\max\lp \q^{(net)},0 \rp   \rp
   \rightarrow
  \lp \max\lp \bar{g}_c^{(3)},0\rp \rp^2,
 \end{eqnarray}
with
 \begin{eqnarray}\label{eq:quadact31}
 \bar{g}_c^{(3)} \sim {\mathcal N}\lp 0, \bar{\sigma}_3^2 \rp, \quad \mbox{and}\quad \bar{\sigma}_3^2=\frac{1}{2}.
 \end{eqnarray}
A combination of (\ref{eq:quadact28a0}) and (\ref{eq:quadact31}) further gives
\begin{align}\label{eq:quadact32}
    \bar{\psi}_{rd}^{(d,2)}(\hat{\p},\hat{\q},\c,\gamma_{sq},\gamma_{sq}^{(p)})
& =   \frac{1}{2}
\c_2
      -\gamma_{sq}^{(p)}  +\frac{1}{2\c_2}\log\lp \frac{2\gamma_{sq}^{(p)}-\c_2}{2\gamma_{sq}^{(p)}}\rp
  + \gamma_{sq}
- \alpha\frac{1}{\c_2}\log\lp \mE_{\bar{g}_c^{(2)}} e^{-\c_2\frac{ \lp \max\lp \bar{g}_c^{(3)},0\rp \rp^2}{4\gamma_{sq}}}\rp \nonumber \\
& =   \frac{1}{2}
\c_2
      -\gamma_{sq}^{(p)}  +\frac{1}{2\c_2}\log\lp \frac{2\gamma_{sq}^{(p)}-\c_2}{2\gamma_{sq}^{(p)}}\rp
  + \gamma_{sq}
- \alpha\frac{1}{\c_2}\log\lp \frac{1}{2} + \frac{1}{2\sqrt{\frac{\bar{\sigma}_3^2\c_2}{2\gamma_{sq}}+1}} \rp.
    \end{align}
 One then computes the derivatives of $ \bar{\psi}_{rd}^{(d,2)}(\hat{\p},\hat{\q},\c,\gamma_{sq},\gamma_{sq}^{(p)})$ with respect to $\gamma_{sq}^{(p)}$, $\gamma_{sq}$, and $\c_2$, equals them to zero, and proceeds by solving the obtained system of equations. Denoting the solution of the system by
 $\hat{\gamma}_{sq}^{(p)}$, $\hat{\gamma}_{sq}$, and $\hat{\c}_2$, the convenient closed form relation (analogous to the one given in (\label{eq:reluact33})) holds
 \begin{equation}\label{eq:quadact33}
 \hat{\gamma}_{sq}^{(p)}=\frac{\hat{\c}_2+\sqrt{\hat{\c}_2^2+4}}{4}.
   \end{equation}
One then from $\bar{\psi}_{rd}^{(d,2)}(\hat{\p},\hat{\q},\hat{\c},\hat{\gamma}_{sq},\hat{\gamma}_{sq}^{(p)})=0$, ultimately finds for
 \begin{equation}\label{eq:quadact34}
\hspace{-2in}(\mbox{\bl{\textbf{\emph{partial} second level:}}}) \qquad \qquad  a_c^{(2,p)}(\infty) =  \lim_{d\rightarrow\infty} a_c^{(2,p)}(d)  \approx \bl{\mathbf{3.3811}}.
  \end{equation}

\subsubsubsection{Full second level of lifting}
\label{sec:secondlevfullquad}

Paralleling (\ref{eq:reluact35}) and keeping in mind (\ref{eq:quadact1}), we then write
\begin{eqnarray}\label{eq:quadact35}
    \bar{\psi}_{rd}^{(d,2)}(\p,\q,\c,\gamma_{sq},\gamma_{sq}^{(p)})
  & = &  \frac{1}{2}
(1-\p_2\q_2)\c_2
 -  \gamma_{sq}^{(p)}
-\Bigg(\Bigg. -\frac{1}{2\c_2} \log \lp \frac{2\gamma_{sq}-\c_2(1-\q_2)}{2\gamma_{sq}} \rp  \nonumber \\
 & & +  \frac{\q_2}{2(2\gamma_{sq}-\c_2(1-\q_2))}   \Bigg.\Bigg)
 \nonumber \\
& &   + \gamma_{sq}
 -\alpha\frac{1}{\c_2}\mE_{{\mathcal U}_3^{(j_w)}} \log\lp \mE_{{\mathcal U}_2^{(j_w)}} e^{-\c_2\frac{z_i^{(2)}\lp \bar{\g}^{(3)};\lp \q^{(net)}\rp^2\rp }{4\gamma_{sq}}}\rp.
    \end{eqnarray}
Moreover, analogously to (\ref{eq:reluact52}), we also have
\begin{eqnarray}\label{eq:quadact52}
  z_i^{(2)}\lp   \bar{\g}^{(3)}; \lp \q^{(net)} \rp^2   \rp
   &\rightarrow  &  \lp \max\lp  \bar{b}_2g_{f}^{(3,1)}+\bar{b}_3g_{f}^{(3,2)},0\rp \rp^2,
 \end{eqnarray}
where $\bar{b}_2$ and $\bar{b}_3$ are as in (\ref{eq:reluact53}), and, as mentioned earlier, $g_f^{(3,1)}$ relates to the randomness of $\g^{(2)}$ (i.e., ${\mathcal U}_2$) and $g_f^{(3,2)}$ relates to the randomness of $\g^{(3)}$ (i.e., ${\mathcal U}_3$).

\noindent \red{\textbf{\emph{(i) Handling $\mE_{\bar{\g}^{(3)}} \left \|  \frac{d\f^{(2)}\lp \g^{(x,2)} \rp }{d\g^{(x,2)}} \right \|_2^2$:}}} Analogously to (\ref{eq:quadact26}), we, keeping in mind (\ref{eq:quadact1}), find
\begin{eqnarray}\label{eq:quadact53a0}
\mE_{\bar{\g}^{(3)}} \left \|  \frac{d\f^{(2)}\lp \g^{(x,2)} \rp }{d\g^{(x,2)}} \right \|_2^2=4d.
\end{eqnarray}

\noindent \red{\textbf{\emph{(ii) Further specializing to $\f^{(2)}\lp \q^{(net)}\rp=\lp \q^{(net)}\rp^2$:}}} We first observe
\begin{eqnarray}\label{eq:quadact54}
\bar{p}_1 =\frac{1}{4}\mE_{\g_1^{(3)}} \mE_{\g_1^{(2)}} \lp \f^{(2)}\lp \sum_{k=2}^{3}b_k \g_1^{(k)}\rp \rp^2=
 \frac{1}{4}\mE_{\g_1^{(3)}} \mE_{\g_1^{(2)}} \lp \sum_{k=2}^{3}b_k \g_1^{(k)}\rp^4=\frac{3}{4},
    \end{eqnarray}
and
\begin{eqnarray}\label{eq:quadact55}
\bar{p}_3 =  \frac{1}{4}\lp \mE_{\g_1^{(3)}} \mE_{\g_1^{(2)}}  \f^{(2)}\lp \sum_{k=2}^{3}b_k \g_1^{(k)}\rp \rp^2=
   \frac{1}{4} \lp \mE_{\g_1^{(3)}} \mE_{\g_1^{(2)}}   \lp \sum_{k=2}^{3}b_k \g_1^{(k)}\rp^2 \rp^2
  =\frac{1}{4}. \nonumber \\
 \end{eqnarray}
Then one also has
\begin{eqnarray}\label{eq:quadact56}
  \mE_{\g_1^{(2)}}  \f^{(2)}\lp \sum_{k=2}^{3}b_k \g_1^{(k)}\rp
&  = &
    \mE_{\g_1^{(2)}}   \lp \sum_{k=2}^{3}b_k \g_1^{(k)} \rp^2
  =
    \mE_{\g_1^{(2)}}   \lp \sqrt{1-\p_2} \g_1^{(2)} + \sqrt{\p_2} \g_1^{(3)} \rp^2
    = 1-\p_2+\p_2 \lp\g_1^{(3)} \rp^2, \nonumber \\
   \end{eqnarray}
and
\begin{eqnarray}\label{eq:quadact57}
\bar{p}_2 & = &  \frac{1}{4} \mE_{\g_1^{(3)}} \lp \mE_{\g_1^{(2)}}  \f^{(2)}\lp \sum_{k=2}^{3}b_k \g_1^{(k)}\rp  \rp^2
= \frac{1}{4} \mE_{\g_1^{(3)}} \lp  1-\p_2+\p_2 \lp\g_1^{(3)} \rp^2 \rp^2
\nonumber \\
 & = &
 \frac{1}{4} \lp ( 1-\p_2)^2+2(1-\p_2)\p_2+3\p_2^2\rp
 =  \frac{1}{4} \lp  1+2\p_2^2\rp.
 \nonumber \\
   \end{eqnarray}
One then has
\begin{eqnarray}\label{eq:quadact58}
\bar{b}_2
& = &
\sqrt{\bar{p}_1-\bar{p}_2}
\nonumber \\
\bar{b}_3
& = &
\sqrt{\bar{p}_2-\bar{p}_3},
 \end{eqnarray}
where, for the \emph{quadratic} activations, $\bar{p}_1$, $\bar{p}_2$, and $\bar{p}_3$ are as in (\ref{eq:quadact54}), (\ref{eq:quadact55}), and (\ref{eq:quadact57}), respectively.

We can now  rewrite (\ref{eq:quadact35}) as
\begin{eqnarray}\label{eq:quadact59}
    \bar{\psi}_{rd}^{(d,2)}(\p,\q,\c,\gamma_{sq},\gamma_{sq}^{(p)})
   & = &  \frac{1}{2}
(1-\p_2\q_2)\c_2
 -  \gamma_{sq}^{(p)}
-\Bigg(\Bigg. -\frac{1}{2\c_2} \log \lp \frac{2\gamma_{sq}-\c_2(1-\q_2)}{2\gamma_{sq}} \rp  \nonumber \\
 & & +  \frac{\q_2}{2(2\gamma_{sq}-\c_2(1-\q_2))}   \Bigg.\Bigg)
   + \gamma_{sq}
 -\alpha\frac{1}{\c_2}   \mE_{g_f^{(3,2)}} \log\lp f_{(zt)}^{(2,f)} \rp,
    \end{eqnarray}
where $f_{(zt)}^{(2,f)}$ is as in (\ref{eq:reluact60}). One then computes the derivatives as in \cite{Stojnicnegsphflrdt23}, equals them to zero to obtain the system as in (\ref{eq:reluact62})and denotes the
obtained solution by $\hat{\q}_2,\hat{\p}_2,\hat{\c}_2,\hat{\gamma}_{sq}^{(p)},\hat{\gamma}_{sq}$. Keeping in mind that the closed form relations from (\ref{eq:reluact63}) continue to hold, one,
after taking the concrete numerical values for all the relevant parameters, from $\bar{\psi}_{rd}^{(d,2)}(\hat{\p},\hat{\q},\hat{\c},\hat{\gamma}_{sq},\hat{\gamma}_{sq}^{(p)})=0$ finds for
 \begin{equation}\label{eq:quadact64}
\hspace{-2in}(\mbox{\bl{\textbf{\emph{full} second level:}}}) \qquad \qquad  a_c^{(2,f)}(\infty) =  \lim_{d\rightarrow\infty} a_c^{(2,f)}(d)  \approx \bl{\mathbf{3.3750}}.
  \end{equation}

\noindent \underline{\textbf{\emph{Concrete numerical values:}}}  In  Table \ref{tab:tab3}, the $a_c^{(2,f)}(\infty)$ obtained above is complemented with the concrete values of all the parameters relevant to the second \emph{full} (2-sfl RDT) level of lifting. To enusre a systematic view of the lifting progress,  the corresponding  quantities for the first \emph{full} (1-sfl RDT) and the second \emph{partial} (2-spl RDT) level are included in the table as well.
\begin{table}[h]
\caption{$r$-sfl RDT parameters; \textbf{\emph{Quadratic}} activations -- \emph{wide} treelike net capacity;  $\hat{\c}_1\rightarrow 1$; $d\rightarrow\infty$; $n\rightarrow\infty$}\vspace{.1in}
\centering
\def\arraystretch{1.2}
\begin{tabular}{||l||c|c||c|c||c|c||c||c||}\hline\hline
 \hspace{-0in}$r$-sfl RDT                                             & $\hat{\gamma}_{sq}$    & $\hat{\gamma}_{sq}^{(p)}$    &  $\hat{\p}_2$ & $\hat{\p}_1$     & $\hat{\q}_2$  & $\hat{\q}_1$ &  $\hat{\c}_2$    & $\alpha_c^{(r)}(\infty)$  \\ \hline\hline
$1$-sfl RDT                                      & $0.5$ & $0.5$ &  $0$  & $\rightarrow 1$   & $0$ & $\rightarrow 1$
 &  $\rightarrow 0$  & \bl{$\mathbf{4}$} \\ \hline\hline
 $2$-spl RDT                                      & $ 0.1975$ & $1.2657$ &  $0$ & $\rightarrow 1$ &  $0$ & $\rightarrow 1$ &   $2.1364$   & \bl{$\mathbf{3.3811}$} \\ \hline
  $2$-sfl RDT                                      & $0.1855$  & $ 1.3482$  & $0.2845$ & $\rightarrow 1$ &  $0.0666$ & $\rightarrow 1$
 &  $2.3705$   & \bl{$\mathbf{3.3750}$}  \\ \hline\hline
  \end{tabular}
\label{tab:tab3}
\end{table}

As was the case for the \emph{ReLU} activations, the capacity results shown in Table \ref{tab:tab3} exactly match the corresponding ones obtained using the statistical physics replica methods relying on the replica symmetry, partial 1rsb, and full 1rsb in \cite{ZavPeh21}.

\subsubsection{$r=3$ -- third level of lifting}
\label{sec:thirdlevquad}

Analogously to  (\ref{eq:reluact65}) (and ultimately (\ref{eq:negprac19}) and (\ref{eq:quadact35})), we then write
{\small\begin{align}\label{eq:quadact65}
    \bar{\psi}_{rd}^{(d,3)}(\p,\q,\c,\gamma_{sq},\gamma_{sq}^{(p)})
  & =   \frac{1}{2}
(1-\p_2\q_2)\c_2+ \frac{1}{2}
(\p_2\q_2-\p_3\q_3)\c_3 -  \gamma_{sq}^{(p)} \nonumber \\
&\quad
-\Bigg(\Bigg. -\frac{1}{2\c_2} \log \lp \frac{2\gamma_{sq}^{(p)}-\c_2(1-\q_2)}{2\gamma_{sq}^{(p)}} \rp  -\frac{1}{2\c_3} \log \lp \frac{2\gamma_{sq}^{(p)}-\c_2(1-\q_2)-\c_3(\q_2-\q_3)}{2\gamma_{sq}^{(p)}-\c_2(1-\q_2)} \rp  \nonumber \\
& \quad +  \frac{\q_3}{2(2\gamma_{sq}^{(p)}-\c_2(1-\q_2)-\c_3(\q_2-\q_3))}   \Bigg.\Bigg)
 \nonumber \\
& \quad   + \gamma_{sq}
 -\frac{\alpha}{\c_3}\mE_{{\mathcal U}_4^{(j_w)}} \log\lp \mE_{{\mathcal U}_3^{(j_w)}} \lp \mE_{{\mathcal U}_2^{(j_w)}} e^{-\c_2\frac{z_i^{(3)}\lp \bar{\g}^{(4)};\lp \q^{(net)}\rp^2\rp}{4\gamma_{sq}}}\rp^{\frac{\c_3}{\c_2}}\rp,
    \end{align}}
 where
 \begin{eqnarray}\label{eq:quadact66}
  z_i^{(3)}\lp   \bar{\g}^{(4)};\f^{(2)}\lp \q^{(net)} \rp   \rp
    &\rightarrow  &  \lp \max\lp  \bar{b}_2g_{f}^{(4,1)}+\bar{b}_3g_{f}^{(4,2)}+\bar{b}_4g_{f}^{(4,3)},0\rp \rp^2,
 \end{eqnarray}
with $\bar{b}_2$, $\bar{b}_3$, and $\bar{b}_4$  as in (\ref{eq:reluact67}), and $b_2=\sqrt{1-\p_2}$, $b_3=\sqrt{\p_2-\p_3}$, and $b_4=\sqrt{\p_3}$. Similarly to what we had earlier, $g_f^{(4,1)}$ relates to the randomness of $\g^{(2)}$ (i.e., ${\mathcal U}_2$),
$g_f^{(4,2)}$ to the randomness of $\g^{(3)}$ (i.e., ${\mathcal U}_3$), and $g_f^{(4,3)}$ to the randomness of $\g^{(4)}$ (i.e., ${\mathcal U}_4$).

\noindent \red{\textbf{\emph{(i) Handling $\mE_{\bar{\g}^{(4)}} \left \|  \frac{d\f^{(2)}\lp \g^{(x,3)} \rp }{d\g^{(x,3)}} \right \|_2^2$:}}} As in (\ref{eq:quadact53a0}), we, keeping in mind (\ref{eq:quadact1}), find
\begin{eqnarray}\label{eq:quadact65a0}
\mE_{\bar{\g}^{(4)}} \left \|  \frac{d\f^{(2)}\lp \g^{(x,3)} \rp }{d\g^{(x,3)}} \right \|_2^2=
\mE_{\bar{\g}^{(4)}} \left \|  \frac{d\lp \g^{(x,3)} \rp^2 }{d\g^{(x,3)}} \right \|_2^2=4d.
\end{eqnarray}

\noindent \red{\textbf{\emph{(ii) Further specializing to $\f^{(2)}\lp \q^{(net)}\rp=\max\lp \q^{(net)},0\rp$:}}} We start by observing
\begin{eqnarray}\label{eq:quadact68}
\bar{p}_1 =\frac{1}{4}\mE_{\g_1^{(4)}} \mE_{\g_1^{(3)}} \mE_{\g_1^{(2)}} \lp \f^{(2)}\lp \sum_{k=2}^{4}b_k \g_1^{(k)}\rp \rp^2=
 2\mE_{\g_1^{(4)}}  \mE_{\g_1^{(3)}} \mE_{\g_1^{(2)}} \lp \sum_{k=2}^{4}b_k \g_1^{(k)}\rp^4=\frac{3}{4},\nonumber \\
    \end{eqnarray}
and
\begin{eqnarray}\label{eq:quadact69}
\bar{p}_4 & = & \frac{1}{4}\lp \mE_{\g_1^{(4)}}  \mE_{\g_1^{(3)}} \mE_{\g_1^{(2)}}  \f^{(2)}\lp \sum_{k=2}^{4}b_k \g_1^{(k)}\rp \rp^2 =
   \frac{1}{4}\lp \mE_{\g_1^{(4)}}  \mE_{\g_1^{(3)}} \mE_{\g_1^{(2)}}   \lp\sum_{k=2}^{4}b_k \g_1^{(k)}\rp^2 \rp^2= \frac{1}{4}. \nonumber \\
 \end{eqnarray}
Then one also has
\begin{eqnarray}\label{eq:quadact70}
  \mE_{\g_1^{(3)}}  \mE_{\g_1^{(2)}}  \f^{(2)}\lp \sum_{k=2}^{4}b_k \g_1^{(k)}\rp
&  = &
  \mE_{\g_1^{(3)}}     \mE_{\g_1^{(2)}}   \lp \sum_{k=2}^{4}b_k \g_1^{(k)}\rp^2 \nonumber \\
&  = &
    \mE_{\g_1^{(3)}}    \mE_{\g_1^{(2)}}   \lp \sqrt{1-\p_2} \g_1^{(2)} + \sqrt{\p_2-\p_3} \g_1^{(3)}+ \sqrt{\p_3} \g_1^{(4)} \rp^2 \nonumber \\
& = &
1-\p_2+\p_2 \lp\g_1^{(3)} \rp^2, \nonumber \\
  \end{eqnarray}
and
\begin{eqnarray}\label{eq:quadact71}
\bar{p}_3  & = &  \frac{1}{4} \mE_{\g_1^{(4)}} \lp \mE_{\g_1^{(3)}} \mE_{\g_1^{(2)}}  \f^{(2)}\lp \sum_{k=2}^{4}b_k \g_1^{(k)}\rp  \rp^2
=
\frac{1}{4} \mE_{\g_1^{(4)}} \lp 1-\p_2+\p_2 \lp\g_1^{(3)} \rp^2  \rp^2
\nonumber \\
 & = &
 \frac{1}{4} \lp  1+2\p_3^2\rp  \triangleq  \bar{p}_x (\p_3).
  \end{eqnarray}
We then also quickly find
\begin{eqnarray}\label{eq:quadact72}
     \mE_{\g_1^{(2)}}  \f^{(2)}\lp \sum_{k=2}^{4}b_k \g_1^{(k)}\rp
&  = &
      \mE_{\g_1^{(2)}}   \lp \sum_{k=2}^{4}b_k \g_1^{(k)} \rp^2 \nonumber \\
&  = &
      \mE_{\g_1^{(2)}}   \lp \sqrt{1-\p_2} \g_1^{(2)} + \sqrt{\p_2-\p_3} \g_1^{(3)}+ \sqrt{\p_3} \g_1^{(4)} \rp^2, \nonumber \\
  \end{eqnarray}
and
\begin{eqnarray}\label{eq:quadact73}
\bar{p}_2 & = &  \frac{1}{4} \mE_{\g_1^{(4)}}  \mE_{\g_1^{(3)}} \lp \mE_{\g_1^{(2)}}  \f^{(2)}\lp \sum_{k=2}^{4}b_k \g_1^{(k)}\rp  \rp^2  =
\bar{p}_x(\p_2). \nonumber \\
  \end{eqnarray}
A combination of (\ref{eq:reluact67}), (\ref{eq:quadact68}), (\ref{eq:quadact69}), (\ref{eq:quadact71}), and (\ref{eq:quadact73}) then also gives
\begin{eqnarray}\label{eq:quadact74}
\bar{b}_2
& = &
\sqrt{\bar{p}_1-\bar{p}_2}
\nonumber \\
\bar{b}_3
& = &
\sqrt{\bar{p}_2-\bar{p}_3}
\nonumber \\
\bar{b}_4
& = &
\sqrt{\bar{p}_3-\bar{p}_4},
 \end{eqnarray}
where, for the \emph{quadratic} activations, $\bar{p}_1$, $\bar{p}_2$, $\bar{p}_3$, and $\bar{p}_4$ are as in (\ref{eq:quadact68}), (\ref{eq:quadact69}), (\ref{eq:quadact71}), and (\ref{eq:quadact73}), respectively.

Analogously to (\ref{eq:reluact78}), we then have
\begin{align}\label{eq:quadact78}
    \bar{\psi}_{rd}(\p,\q,\c,\gamma_{sq},\gamma_{sq}^{(p)})
 & =    \frac{1}{2}
(1-\p_2\q_2)\c_2+ \frac{1}{2}
(\p_2\q_2-\p_3\q_3)\c_3 -  \gamma_{sq}^{(p)}
-\Bigg(\Bigg. -\frac{1}{2\c_2} \log \lp \frac{2\gamma_{sq}^{(p)}-\c_2(1-\q_2)}{2\gamma_{sq}^{(p)}} \rp
 \nonumber \\
& \quad -\frac{1}{2\c_3} \log \lp \frac{2\gamma_{sq}^{(p)}-\c_2(1-\q_2)-\c_3(\q_2-\q_3)}{2\gamma_{sq}^{(p)}-\c_2(1-\q_2)} \rp
 \nonumber \\
& \quad +  \frac{\q_3}{2(2\gamma_{sq}^{(p)}-\c_2(1-\q_2)-\c_3(\q_2-\q_3))}   \Bigg.\Bigg)
 + \gamma_{sq}
 -\frac{\alpha}{\c_3}   \mE_{g_{f}^{(4,3)}} \log\lp \mE_{g_{f}^{(4,2)}} \lp f_{(zt)}^{(3,f)} \rp^{\frac{\c_3}{\c_2}}\rp,
    \end{align}
where $f_{(zt)}^{(3,f)} $ is as in (\ref{eq:reluact76}) with $\bar{b}_2$, $\bar{b}_3$, and $\bar{b}_4$ as in (\ref{eq:quadact74}). One then proceeds by solving the system in (\ref{eq:reluact79}) and observing that the closed form relations (\ref{eq:reluact80}) continue to hold for the obtained solutions. Utilizing the obtained concrete numerical values for all the considered parameters, one then, from  $\bar{\psi}_{rd}^{(d,3)}(\hat{\p},\hat{\q},\hat{\c},\hat{\gamma}_{sq},\hat{\gamma}_{sq}^{(p)})=0$, obtains for
 \begin{equation}\label{eq:quadact81}
\hspace{-2in}(\mbox{\bl{\textbf{\emph{full} third level:}}}) \qquad \qquad  a_c^{(3,f)}(\infty) =  \lim_{d\rightarrow\infty} a_c^{(3,f)}(d)  \approx \bl{\mathbf{3.3669}}.
  \end{equation}

\noindent \underline{\textbf{\emph{Concrete numerical values:}}}  In  Table \ref{tab:tab2}, the above $a_c^{(3,f)}(\infty)$ is complemented by the concrete values of all the relevant quantities related to the third \emph{full} (3-sfl RDT) level of lifting. As earlier, the corresponding quantities for the first \emph{full} (1-sfl RDT) and the second \emph{partial} (2-spl RDT) and \emph{full} (2-sfl RDT) levels  are shown in parallel to enable a systematic view of the lifting mechanism progressing.
\begin{table}[h]
\caption{$r$-sfl RDT parameters; \textbf{\emph{Quadratic}} activations  -- \emph{wide} treelike net capacity;  $\hat{\c}_1\rightarrow 1$; $d\rightarrow\infty$; $n\rightarrow\infty$}\vspace{.1in}
\centering
\def\arraystretch{1.2}
{\small
\begin{tabular}{||l||c|c||c|c|c||c|c|c||c|c||c||}\hline\hline
 \hspace{-0in}$r$-sfl RDT                                             & $\hat{\gamma}_{sq}$    & $\hat{\gamma}_{sq}^{(p)}$    &  $\hat{\p}_3$  &  $\hat{\p}_2$ & $\hat{\p}_1$    &  $\hat{\q}_3$   & $\hat{\q}_2$  & $\hat{\q}_1$ &  $\hat{\c}_3$ &  $\hat{\c}_2$    & $\alpha_c^{(r)}(\infty)$  \\ \hline\hline
$1$-sfl RDT                                      & $0.5$ & $0.5$ &  $0$  &  $0$  & $\rightarrow 1$  &  $0$   & $0$ & $\rightarrow 1$
 &  $0$  &  $\rightarrow 0$  & \bl{$\mathbf{4}$} \\ \hline\hline
 $2$-spl RDT                                      & $ 0.1975$ & $1.2657$  &  $0$  &  $0$ & $\rightarrow 1$ &  $0$  &  $0$ & $\rightarrow 1$  &  $0$ &   $2.1364$   & \bl{$\mathbf{3.3811}$} \\ \hline
  $2$-sfl RDT                                      & $0.1855$  & $ 1.3482$   &  $0$ & $0.2845$ & $\rightarrow 1$  &  $0$  &  $0.0666$ & $\rightarrow 1$
 &  $0$  &  $2.3705$   & \bl{$\mathbf{3.3750}$}  \\ \hline\hline
   $3$-sfl RDT                                      & $0.1110$  & $2.2673$ & $0.2014$ & $0.9557$ &  $\rightarrow 1$ &  $0.0432$ &  $0.6504$ & $\rightarrow 1$
&  $2.1$ &  $8$   & \bl{$\mathbf{3.3669}$}  \\ \hline\hline
  \end{tabular}
  }
\label{tab:tab4}
\end{table}

\subsubsection{General $r$-th level of lifting}
\label{sec:rlev}

All general $r$ lifting considerations are exactly the same as stated between  (\ref{eq:reluact82})-(\ref{eq:reluact87}) with $\bar{p}_x(\cdot)$ as in (\ref{eq:quadact71}).

\subsection{Error function ($\erf$) activations}
\label{sec:erf}

We now consider the well known $\erf$  activation. This basically means that the neuronal function in the hidden layer is now assumed as
\begin{equation}
\hspace{-1.5in} \mbox{\textbf{\bl{\erf} activation:}} \hspace{1in} \f^{(2)}(\x)=\erf\lp \x\rp.
\label{eq:erfact1}
\end{equation}
As usual, we start by looking at the first level of lifting, utilize the results obtained in the previous sections for generic activations, and specialize them to the $\erf$ scenario of interest here.

\subsubsection{$r=1$ -- first level of lifting}
\label{sec:firstleverf}

We first observe that (\ref{eq:negprac20a0}) can now be rewritten as
 \begin{eqnarray}\label{eq:erfact0a0}
a_c^{(1)}(d)
& = &   \frac{1}{\mE_{{\mathcal U}_2^{(j_w)}}  \lp z_i^{(1)}\lp\g^{(2)};\f^{(2)}\lp \q^{(net)} \rp\rp\rp}
  =   \frac{1}{\mE_{\g^{(2)}}  \lp z_i^{(1)}\lp \bar{\g}^{(2)}; \erf\lp\q^{(net)}\rp\rp\rp}.
  \end{eqnarray}
where one also recalls on (\ref{eq:reluact12})
    \begin{eqnarray}\label{eq:erfact12}
\lim_{d\rightarrow \infty}\mE_{\g^{(2)}} z_i^{(1)}\lp   \bar{\g}^{(2)};\f^{(2)}\lp \q^{(net)} \rp   \rp
 & \rightarrow &
\frac{\lim_{d\rightarrow \infty}\mE_{\g^{(2)}} \lp \max\lp - \f^{(2)}\lp \g^{(2)} \rp^T\w,0\rp \rp^2}{\lim_{d\rightarrow \infty}\mE_{\g^{(2)}} \left \|  \frac{d\f^{(2)}\lp \g^{(2)} \rp }{d\g^{(2)}} \right \|_2^2}.
 \end{eqnarray}

\noindent \red{\textbf{\emph{(i) Handling $\mE_{\g^{(2)}} \lp \max\lp - \f^{(2)}\lp \g^{(2)} \rp^T\w,0\rp \rp^2$:}}} From (\ref{eq:reluact2})), we further have
 \begin{eqnarray}\label{eq:erfact20}
\mE_{\g^{(2)}} \lp \max\lp - \f^{(2)}\lp \g^{(2)} \rp^T\w,0\rp \rp^2
=\mE_{g_c^{(2)}} \lp \max\lp g_c^{(2)},0\rp \rp^2=\frac{1}{2}\sigma_{2}^2,
\end{eqnarray}
where recalling on (\ref{eq:erfact1}) and analogously to   (\ref{eq:reluact21}), (\ref{eq:reluact22}), and (\ref{eq:reluact23})
\begin{eqnarray}
\mE \lp \f^{(2)}\lp \g_{1}^{(2)}\rp \rp &= &\mE \erf\lp\g_1^{(2)}\rp =0 \nonumber \\
 \mE \lp \f^{(2)}\lp \g_{1}^{(2)}\rp \rp^2 &= &\mE \lp \erf\lp \g_1^{(2)}\rp\rp^2 = \frac{2}{\pi}\asin\lp \frac{2}{3}\rp\approx 0.4646,
\label{eq:erfact22}
\end{eqnarray}
and
 \begin{eqnarray}\label{eq:erfact23}
 \sigma_2^2=d\lp\mE \lp\f^{(2)}\lp \g_{1}^{(2)} \rp \rp^2- \lp \mE\lp \f^{(2)}\lp \g_{1}^{(2)}\rp\rp\rp^2\rp
 =\frac{2}{\pi}\asin\lp \frac{2}{3}\rp d\approx 0.4646d.
\end{eqnarray}
From  (\ref{eq:erfact20}) and (\ref{eq:erfact23}), we also find
 \begin{eqnarray}\label{eq:erfact24}
\mE_{\g^{(2)}} \lp \max\lp - \f^{(2)}\lp \g^{(2)} \rp^T\w,0\rp \rp^2
 =\frac{1}{2}\sigma_{2}^2=\frac{1}{\pi}\asin\lp \frac{2}{3}\rp d\approx 0.2323d.
\end{eqnarray}

\noindent \red{\textbf{\emph{(ii) Handling $\mE_{\g^{(2)}} \left \|  \frac{d\f^{(2)}\lp \g^{(2)} \rp }{d\g^{(2)}} \right \|_2^2$:}}} Recalling again on (\ref{eq:erfact1}), we find
 \begin{eqnarray}\label{eq:erfact25}
\frac{d\f^{(2)}\lp \g_{j_w}^{(2)} \rp }{d\g_{j_w}^{(2)}}= \frac{d\erf\lp \g_{j_w}^{(2)} \rp }{d\g_{j_w}^{(2)}}= \frac{2}{\sqrt{\pi}}e^{-\lp\g_{j_w}^{(2)}\rp^2}.
\end{eqnarray}
One then further finds
 \begin{eqnarray}\label{eq:erfact26}
\mE_{\g^{(2)}} \left \|  \frac{d\f^{(2)}\lp \g^{(2)} \rp }{d\g^{(2)}} \right \|_2^2
=\mE_{\g^{(2)}} \sum_{j_w=1}^{d}\lp \frac{d\f^{(2)}\lp \g_{j_w}^{(2)} \rp }{d\g_{j_w}^{(2)}}\rp^2
=  \frac{4}{\pi}\sum_{j_w=1}^{d}\mE_{\g_{j_w}^{(2)}} e^{-2\lp\g_{j_w}^{(2)}\rp^2}=\frac{4}{\sqrt{5}\pi}d.
\end{eqnarray}
Relying on the above observations, and combining (\ref{eq:erfact12}), (\ref{eq:erfact24}), and (\ref{eq:erfact26}), one obtains
 \begin{eqnarray}\label{eq:erfact27}
\lim_{d\rightarrow\infty}\mE_{\g^{(2)}} z_i^{(1)}\lp   \bar{\g}^{(2)};\f^{(2)}\lp \q^{(net)} \rp   \rp
 & =  &  \lim_{d\rightarrow\infty}\frac{\mE_{\g^{(2)}} \lp \max\lp - \f^{(2)}\lp \g^{(2)} \rp^T\w,0\rp \rp^2}{\mE_{\g^{(2)}} \left \|  \frac{d\f^{(2)}\lp \g^{(2)} \rp }{d\g^{(2)}} \right \|_2^2}=
 \frac{\frac{1}{\pi}\asin\lp \frac{2}{3}\rp }{\frac{4}{\sqrt{5}\pi}}
 =0.4079.\nonumber \\
 \end{eqnarray}
A further combination of (\ref{eq:negprac20a0}), (\ref{eq:erfact0a0}), and (\ref{eq:erfact27}) then gives
\begin{eqnarray}\label{eq:erfact28}
\hspace{-0in}(\mbox{\bl{\textbf{first level:}}}) \qquad  a_c^{(1)}(\infty)
& = &    \lim_{d\rightarrow\infty} a_c^{(1)}(d)   =   \frac{1}{\lim_{d\rightarrow\infty} \mE_{{\mathcal U}_2^{(j_w)}}  \lp z_i^{(1)}\lp  \bar{\g}^{(2)};\f^{(2)}\lp \q^{(net)} \rp\rp\rp}  \nonumber \\
 & = &   \frac{1}{\lim_{d\rightarrow\infty} \mE_{\g^{(2)}}  \lp z_i^{(1)}\lp   \bar{\g}^{(2)}; \erf\lp\q^{(net)}\rp\rp\rp} = \frac{4}{\sqrt{5}\asin\lp \frac{2}{3}\rp}
  \approx \bl{\mathbf{2.4514}}.
  \end{eqnarray}

\subsubsection{Second level of lifting}
\label{sec:secondleverf}

Analogously to (\ref{eq:reluact35}) and keeping in mind (\ref{eq:erfact1}), we first have
\begin{eqnarray}\label{eq:erfact35}
    \bar{\psi}_{rd}^{(d,2)}(\p,\q,\c,\gamma_{sq},\gamma_{sq}^{(p)})
  & = &  \frac{1}{2}
(1-\p_2\q_2)\c_2
 -  \gamma_{sq}^{(p)}
-\Bigg(\Bigg. -\frac{1}{2\c_2} \log \lp \frac{2\gamma_{sq}-\c_2(1-\q_2)}{2\gamma_{sq}} \rp  \nonumber \\
 & & +  \frac{\q_2}{2(2\gamma_{sq}-\c_2(1-\q_2))}   \Bigg.\Bigg)
 \nonumber \\
& &   + \gamma_{sq}
 -\alpha\frac{1}{\c_2}\mE_{{\mathcal U}_3^{(j_w)}} \log\lp \mE_{{\mathcal U}_2^{(j_w)}} e^{-\c_2\frac{z_i^{(2)}\lp \bar{\g}^{(3)}; \erf\lp \q^{(net)}\rp \rp }{4\gamma_{sq}}}\rp,
    \end{eqnarray}
where, analogously to (\ref{eq:reluact52}),
\begin{eqnarray}\label{eq:erfact52}
  z_i^{(2)}\lp   \bar{\g}^{(3)}; \erf\lp \q^{(net)} \rp   \rp
   &\rightarrow  &  \lp \max\lp  \bar{b}_2g_{f}^{(3,1)}+\bar{b}_3g_{f}^{(3,2)},0\rp \rp^2,
 \end{eqnarray}
with $\bar{b}_2$ and $\bar{b}_3$ as in (\ref{eq:reluact53}). As mentioned on multiple occasions earlier, $g_f^{(3,1)}$ relates to the randomness of $\g^{(2)}$ (i.e., ${\mathcal U}_2$) and $g_f^{(3,2)}$ relates to the randomness of $\g^{(3)}$ (i.e., ${\mathcal U}_3$).

\noindent \red{\textbf{\emph{(i) Handling $\mE_{\bar{\g}^{(3)}} \left \|  \frac{d\f^{(2)}\lp \g^{(x,2)} \rp }{d\g^{(x,2)}} \right \|_2^2$:}}} Analogously to (\ref{eq:erfact26}), and keeping in mind (\ref{eq:erfact1}), we find
\begin{eqnarray}\label{eq:erfact53a0}
\mE_{\bar{\g}^{(3)}} \left \|  \frac{d\f^{(2)}\lp \g^{(x,2)} \rp }{d\g^{(x,2)}} \right \|_2^2=\frac{4}{\sqrt{5}\pi}d.
\end{eqnarray}

\noindent \red{\textbf{\emph{(ii) Further specializing to $\f^{(2)}\lp \q^{(net)}\rp=\erf\lp \q^{(net)}\rp$:}}} We first observe
\begin{eqnarray}\label{eq:erfact54}
\bar{p}_1  &  = & \frac{\sqrt{5}\pi}{4}\mE_{\g_1^{(3)}} \mE_{\g_1^{(2)}} \lp \f^{(2)}\lp \sum_{k=2}^{3}b_k \g_1^{(k)}\rp \rp^2=
 \frac{\sqrt{5}\pi}{4}\mE_{\g_1^{(3)}} \mE_{\g_1^{(2)}} \lp\erf \lp \sum_{k=2}^{3}b_k \g_1^{(k)}\rp\rp^2 \nonumber \\
& = & \frac{\sqrt{5}}{2}\asin \lp \frac{2}{3}\rp \approx 0.8159,
    \end{eqnarray}
and
\begin{eqnarray}\label{eq:erfact55}
\bar{p}_3 =  \frac{\sqrt{5}\pi}{4}\lp \mE_{\g_1^{(3)}} \mE_{\g_1^{(2)}}  \f^{(2)}\lp \sum_{k=2}^{3}b_k \g_1^{(k)}\rp \rp^2=
   \frac{\sqrt{5}\pi}{4} \lp \mE_{\g_1^{(3)}} \mE_{\g_1^{(2)}}  \erf  \lp \sum_{k=2}^{3}b_k \g_1^{(k)}\rp \rp^2
  =0. \nonumber \\
 \end{eqnarray}
Then one also has
\begin{eqnarray}\label{eq:erfact56}
  \mE_{\g_1^{(2)}}  \f^{(2)}\lp \sum_{k=2}^{3}b_k \g_1^{(k)}\rp
&  = &
    \mE_{\g_1^{(2)}}   \erf\lp \sum_{k=2}^{3}b_k \g_1^{(k)} \rp
  =
    \mE_{\g_1^{(2)}} \erf \lp \sqrt{1-\p_2} \g_1^{(2)} + \sqrt{\p_2} \g_1^{(3)} \rp, \nonumber \\
   \end{eqnarray}
and
\begin{eqnarray}\label{eq:erfact57}
\bar{p}_2 & = &  \frac{\sqrt{5}\pi}{4} \mE_{\g_1^{(3)}} \lp \mE_{\g_1^{(2)}}  \f^{(2)}\lp \sum_{k=2}^{3}b_k \g_1^{(k)}\rp  \rp^2
= \frac{\sqrt{5}\pi}{4} \mE_{\g_1^{(3)}} \lp   \mE_{\g_1^{(2)}} \erf \lp \sqrt{1-\p_2} \g_1^{(2)} + \sqrt{\p_2} \g_1^{(3)} \rp  \rp^2. \nonumber \\
   \end{eqnarray}
One then has
\begin{eqnarray}\label{eq:erfact58}
\bar{b}_2
& = &
\sqrt{\bar{p}_1-\bar{p}_2}
\nonumber \\
\bar{b}_3
& = &
\sqrt{\bar{p}_2-\bar{p}_3},
 \end{eqnarray}
where, for the $\erf$ activations, $\bar{p}_1$, $\bar{p}_2$, and $\bar{p}_3$ are as in (\ref{eq:erfact54}), (\ref{eq:erfact55}), and (\ref{eq:erfact57}), respectively.

It is then easy to  rewrite (\ref{eq:erfact35}) as
\begin{eqnarray}\label{eq:erfact59}
    \bar{\psi}_{rd}^{(d,2)}(\p,\q,\c,\gamma_{sq},\gamma_{sq}^{(p)})
   & = &  \frac{1}{2}
(1-\p_2\q_2)\c_2
 -  \gamma_{sq}^{(p)}
-\Bigg(\Bigg. -\frac{1}{2\c_2} \log \lp \frac{2\gamma_{sq}-\c_2(1-\q_2)}{2\gamma_{sq}} \rp  \nonumber \\
 & & +  \frac{\q_2}{2(2\gamma_{sq}-\c_2(1-\q_2))}   \Bigg.\Bigg)
    + \gamma_{sq}
 -\alpha\frac{1}{\c_2}   \mE_{g_f^{(3,2)}} \log\lp f_{(zt)}^{(2,f)} \rp,
    \end{eqnarray}
with $f_{(zt)}^{(2,f)}$ as in (\ref{eq:reluact60}) and $\bar{p}_1$, $\bar{p}_2$, and $\bar{p}_3$ as in (\ref{eq:erfact54}), (\ref{eq:erfact55}), and (\ref{eq:erfact57}), respectively. One then solves the system in (\ref{eq:reluact62}) and observes that the closed form relations from (\ref{eq:reluact63}) continue to hold. After taking concrete numerical values for all the parameters, from $\bar{\psi}_{rd}^{(d,2)}(\hat{\p},\hat{\q},\hat{\c},\hat{\gamma}_{sq},\hat{\gamma}_{sq}^{(p)})=0$, we then find for
 \begin{equation}\label{eq:erfact64}
\hspace{-2in}(\mbox{\bl{\textbf{\emph{full} second level:}}}) \qquad \qquad  a_c^{(2,f)}(\infty) =  \lim_{d\rightarrow\infty} a_c^{(2,f)}(d)  \approx \bl{\mathbf{2.3750}}.
  \end{equation}

\noindent \underline{\textbf{\emph{Concrete numerical values:}}}  As expected by standards set earlier, we, in  Table \ref{tab:tab5}, complement the $a_c^{(2,f)}(\infty)$ obtained above with the concrete values of all the parameters relevant to the second \emph{full} (2-sfl RDT) level of lifting.  The corresponding  quantities for the first \emph{full} (1-sfl RDT)  are included as well so that the progress of the lifting mechanism can be systematically viewed. Differently from earlier sections though, the second partial level makes no progress and is therefore not included in the table.
\begin{table}[h]
\caption{$r$-sfl RDT parameters; \textbf{$\erf$} activations -- \emph{wide} treelike net capacity;  $\hat{\c}_1\rightarrow 1$; $d\rightarrow\infty$; $n\rightarrow\infty$}\vspace{.1in}
\centering
\def\arraystretch{1.2}
\begin{tabular}{||l||c|c||c|c||c|c||c||c||}\hline\hline
 \hspace{-0in}$r$-sfl RDT                                             & $\hat{\gamma}_{sq}$    & $\hat{\gamma}_{sq}^{(p)}$    &  $\hat{\p}_2$ & $\hat{\p}_1$     & $\hat{\q}_2$  & $\hat{\q}_1$ &  $\hat{\c}_2$    & $\alpha_c^{(r)}(\infty)$  \\ \hline\hline
$1$-sfl RDT                                      & $0.5$ & $0.5$ &  $0$  & $\rightarrow 1$   & $0$ & $\rightarrow 1$
 &  $\rightarrow 0$  & \bl{$\mathbf{2.4514}$} \\ \hline\hline
   $2$-sfl RDT                                      & $0.2110$  & $ 1.1847$  & $ 0.7547$ & $\rightarrow 1$ &  $0.5183$ & $\rightarrow 1$
 &  $3.1984$   & \bl{$\mathbf{2.3750}$}  \\ \hline\hline
  \end{tabular}
\label{tab:tab5}
\end{table}

Similarly to what we observed earlier when discussing the \emph{ReLU} and \emph{quadratic} activations, we here again note that the capacity results shown in Table \ref{tab:tab5} exactly match the corresponding ones obtained using the statistical physics replica methods relying on the replica symmetry and full 1rsb in \cite{ZavPeh21}.

\subsubsection{$r=3$ -- third level of lifting}
\label{sec:thirdleverf}

Analogously to  (\ref{eq:reluact65}) (and ultimately (\ref{eq:negprac19}) and (\ref{eq:erfact35})), we then write
{\small\begin{align}\label{eq:erfact65}
    \bar{\psi}_{rd}^{(d,3)}(\p,\q,\c,\gamma_{sq},\gamma_{sq}^{(p)})
  & =   \frac{1}{2}
(1-\p_2\q_2)\c_2+ \frac{1}{2}
(\p_2\q_2-\p_3\q_3)\c_3 -  \gamma_{sq}^{(p)} \nonumber \\
&\quad
-\Bigg(\Bigg. -\frac{1}{2\c_2} \log \lp \frac{2\gamma_{sq}^{(p)}-\c_2(1-\q_2)}{2\gamma_{sq}^{(p)}} \rp  -\frac{1}{2\c_3} \log \lp \frac{2\gamma_{sq}^{(p)}-\c_2(1-\q_2)-\c_3(\q_2-\q_3)}{2\gamma_{sq}^{(p)}-\c_2(1-\q_2)} \rp  \nonumber \\
& \quad +  \frac{\q_3}{2(2\gamma_{sq}^{(p)}-\c_2(1-\q_2)-\c_3(\q_2-\q_3))}   \Bigg.\Bigg)
 \nonumber \\
& \quad   + \gamma_{sq}
 -\frac{\alpha}{\c_3}\mE_{{\mathcal U}_4^{(j_w)}} \log\lp \mE_{{\mathcal U}_3^{(j_w)}} \lp \mE_{{\mathcal U}_2^{(j_w)}} e^{-\c_2\frac{z_i^{(3)}\lp \bar{\g}^{(4)}; \erf\lp \q^{(net)}\rp\rp}{4\gamma_{sq}}}\rp^{\frac{\c_3}{\c_2}}\rp,
    \end{align}}
 where
 \begin{eqnarray}\label{eq:erfact66}
  z_i^{(3)}\lp   \bar{\g}^{(4)};\f^{(2)}\lp \q^{(net)} \rp   \rp
 =
  z_i^{(3)}\lp   \bar{\g}^{(4)};\erf\lp \q^{(net)} \rp   \rp
    &\rightarrow  &  \lp \max\lp  \bar{b}_2g_{f}^{(4,1)}+\bar{b}_3g_{f}^{(4,2)}+\bar{b}_4g_{f}^{(4,3)},0\rp \rp^2,\nonumber \\
 \end{eqnarray}
with $\bar{b}_2$, $\bar{b}_3$, and $\bar{b}_4$  as in (\ref{eq:reluact67}), and $b_2=\sqrt{1-\p_2}$, $b_3=\sqrt{\p_2-\p_3}$, and $b_4=\sqrt{\p_3}$. As usual, $g_f^{(4,k)},k=1,2,3$, relates to the randomness of $\g^{(k+1)}$ (i.e., ${\mathcal U}_{k+1}$).

\noindent \red{\textbf{\emph{(i) Handling $\mE_{\bar{\g}^{(4)}} \left \|  \frac{d\f^{(2)}\lp \g^{(x,3)} \rp }{d\g^{(x,3)}} \right \|_2^2$:}}} As in (\ref{eq:erfact53a0}), we, keeping in mind (\ref{eq:erfact1}), find
\begin{eqnarray}\label{eq:erfact65a0}
\mE_{\bar{\g}^{(4)}} \left \|  \frac{d\f^{(2)}\lp \g^{(x,3)} \rp }{d\g^{(x,3)}} \right \|_2^2=
\mE_{\bar{\g}^{(4)}} \left \|  \frac{d\lp \g^{(x,3)} \rp^2 }{d\g^{(x,3)}} \right \|_2^2=\frac{4}{\sqrt{5}\pi}d.
\end{eqnarray}

\noindent \red{\textbf{\emph{(ii) Further specializing to $\f^{(2)}\lp \q^{(net)}\rp=\erf\lp \q^{(net)}\rp$:}}} We start by observing
\begin{eqnarray}\label{eq:erfact68}
\bar{p}_1 & = & \frac{\sqrt{5}\pi}{4}\mE_{\g_1^{(4)}} \mE_{\g_1^{(3)}} \mE_{\g_1^{(2)}} \lp \f^{(2)}\lp \sum_{k=2}^{4}b_k \g_1^{(k)}\rp \rp^2=
\frac{\sqrt{5}\pi}{4} \mE_{\g_1^{(4)}}  \mE_{\g_1^{(3)}} \mE_{\g_1^{(2)}} \lp \erf\lp \sum_{k=2}^{4}b_k \g_1^{(k)}\rp\rp^2 \nonumber \\
& = &\frac{\sqrt{5}}{2}\asin\lp \frac{2}{3}\rp=0.8159,\nonumber \\
    \end{eqnarray}
and
\begin{eqnarray}\label{eq:erfact69}
\bar{p}_4 & = & \frac{\sqrt{5}\pi}{4} \lp \mE_{\g_1^{(4)}}  \mE_{\g_1^{(3)}} \mE_{\g_1^{(2)}}  \f^{(2)}\lp \sum_{k=2}^{4}b_k \g_1^{(k)}\rp \rp^2 =
   \frac{\sqrt{5}\pi}{4} \lp \mE_{\g_1^{(4)}}  \mE_{\g_1^{(3)}} \mE_{\g_1^{(2)}}   \erf\lp\sum_{k=2}^{4}b_k \g_1^{(k)}\rp \rp^2= 0. \nonumber \\
 \end{eqnarray}
Then one also has
\begin{eqnarray}\label{eq:erfact70}
  \mE_{\g_1^{(3)}}  \mE_{\g_1^{(2)}}  \f^{(2)}\lp \sum_{k=2}^{4}b_k \g_1^{(k)}\rp
&  = &
  \mE_{\g_1^{(3)}}     \mE_{\g_1^{(2)}}   \erf\lp \sum_{k=2}^{4}b_k \g_1^{(k)}\rp \nonumber \\
&  = &
    \mE_{\g_1^{(3)}}    \mE_{\g_1^{(2)}}   \erf\lp \sqrt{1-\p_2} \g_1^{(2)} + \sqrt{\p_2-\p_3} \g_1^{(3)}+ \sqrt{\p_3} \g_1^{(4)} \rp,
  \end{eqnarray}
and
\begin{eqnarray}\label{eq:erfact71}
\bar{p}_3  & = &  \frac{\sqrt{5}\pi}{4} \mE_{\g_1^{(4)}} \lp \mE_{\g_1^{(3)}} \mE_{\g_1^{(2)}}  \f^{(2)}\lp \sum_{k=2}^{4}b_k \g_1^{(k)}\rp  \rp^2
\nonumber \\
& = &
\frac{\sqrt{5}\pi}{4} \mE_{\g_1^{(4)}} \lp     \mE_{\g_1^{(3)}}    \mE_{\g_1^{(2)}}   \erf\lp \sqrt{1-\p_2} \g_1^{(2)} + \sqrt{\p_2-\p_3} \g_1^{(3)}+ \sqrt{\p_3} \g_1^{(4)} \rp \rp^2
  \triangleq  \bar{p}_x (\p_3).
  \end{eqnarray}
One then immediately also has
\begin{eqnarray}\label{eq:erfact72}
     \mE_{\g_1^{(2)}}  \f^{(2)}\lp \sum_{k=2}^{4}b_k \g_1^{(k)}\rp
&  = &
      \mE_{\g_1^{(2)}} \erf  \lp \sum_{k=2}^{4}b_k \g_1^{(k)} \rp \nonumber \\
&  = &
      \mE_{\g_1^{(2)}}  \erf \lp \sqrt{1-\p_2} \g_1^{(2)} + \sqrt{\p_2-\p_3} \g_1^{(3)}+ \sqrt{\p_3} \g_1^{(4)} \rp, \nonumber \\
  \end{eqnarray}
and
\begin{eqnarray}\label{eq:erfact73}
\bar{p}_2 & = &  \frac{\sqrt{5}\pi}{4} \mE_{\g_1^{(4)}}  \mE_{\g_1^{(3)}} \lp \mE_{\g_1^{(2)}}  \f^{(2)}\lp \sum_{k=2}^{4}b_k \g_1^{(k)}\rp  \rp^2  =
\bar{p}_x(\p_2).
  \end{eqnarray}
One then further finds
\begin{eqnarray}\label{eq:erfact74}
\bar{b}_2
& = &
\sqrt{\bar{p}_1-\bar{p}_2}
\nonumber \\
\bar{b}_3
& = &
\sqrt{\bar{p}_2-\bar{p}_3}
\nonumber \\
\bar{b}_4
& = &
\sqrt{\bar{p}_3-\bar{p}_4},
 \end{eqnarray}
where $\bar{p}_1$, $\bar{p}_2$, $\bar{p}_3$, and $\bar{p}_4$ are as in (\ref{eq:erfact68}), (\ref{eq:erfact69}), (\ref{eq:erfact71}), and (\ref{eq:erfact73}), respectively.

Analogously to (\ref{eq:reluact78}), we then have
\begin{align}\label{eq:erfact78}
    \bar{\psi}_{rd}(\p,\q,\c,\gamma_{sq},\gamma_{sq}^{(p)})
 & =    \frac{1}{2}
(1-\p_2\q_2)\c_2+ \frac{1}{2}
(\p_2\q_2-\p_3\q_3)\c_3 -  \gamma_{sq}^{(p)}
-\Bigg(\Bigg. -\frac{1}{2\c_2} \log \lp \frac{2\gamma_{sq}^{(p)}-\c_2(1-\q_2)}{2\gamma_{sq}^{(p)}} \rp
 \nonumber \\
& \quad -\frac{1}{2\c_3} \log \lp \frac{2\gamma_{sq}^{(p)}-\c_2(1-\q_2)-\c_3(\q_2-\q_3)}{2\gamma_{sq}^{(p)}-\c_2(1-\q_2)} \rp
 \nonumber \\
& \quad +  \frac{\q_3}{2(2\gamma_{sq}^{(p)}-\c_2(1-\q_2)-\c_3(\q_2-\q_3))}   \Bigg.\Bigg)
 + \gamma_{sq}
 -\frac{\alpha}{\c_3}   \mE_{g_{f}^{(4,3)}} \log\lp \mE_{g_{f}^{(4,2)}} \lp f_{(zt)}^{(3,f)} \rp^{\frac{\c_3}{\c_2}}\rp,
    \end{align}
where $f_{(zt)}^{(3,f)} $ is as in (\ref{eq:reluact76}) with $\bar{b}_2$, $\bar{b}_3$, and $\bar{b}_4$ as in (\ref{eq:erfact74}). After  observing that the closed form relations (\ref{eq:reluact80}) continue to hold, one solves the system in (\ref{eq:reluact79})  to obtain the concrete numerical values for all the considered parameters, $\hat{\p},\hat{\q},\hat{\c},\hat{\gamma}_{sq},\hat{\gamma}_{sq}^{(p)}$. Plugging these values in  $\bar{\psi}_{rd}^{(d,3)}(\hat{\p},\hat{\q},\hat{\c},\hat{\gamma}_{sq},\hat{\gamma}_{sq}^{(p)})=0$, gives
 \begin{equation}\label{eq:erfact81}
\hspace{-2in}(\mbox{\bl{\textbf{\emph{full} third level:}}}) \qquad \qquad  a_c^{(3,f)}(\infty) =  \lim_{d\rightarrow\infty} a_c^{(3,f)}(d)  \approx \bl{\mathbf{2.3744}}.
  \end{equation}

\noindent \underline{\textbf{\emph{Concrete numerical values:}}}  Table \ref{tab:tab6} contains the concrete values of all the relevant quantities related to the third \emph{full} (3-sfl RDT) level of lifting that complement the above $a_c^{(3,f)}(\infty)$. The corresponding quantities for the first \emph{full} (1-sfl RDT) and the second \emph{full} (2-sfl RDT) levels  are shown in parallel as well, allowing a systematic following of the lifting mechanism progressing.
\begin{table}[h]
\caption{$r$-sfl RDT parameters; \textbf{$\erf$} activations  -- \emph{wide} treelike net capacity;  $\hat{\c}_1\rightarrow 1$; $d\rightarrow\infty$; $n\rightarrow\infty$}\vspace{.1in}
\centering
\def\arraystretch{1.2}
{\small
\begin{tabular}{||l||c|c||c|c|c||c|c|c||c|c||c||}\hline\hline
 \hspace{-0in}$r$-sfl RDT                                             & $\hat{\gamma}_{sq}$    & $\hat{\gamma}_{sq}^{(p)}$    &  $\hat{\p}_3$  &  $\hat{\p}_2$ & $\hat{\p}_1$    &  $\hat{\q}_3$   & $\hat{\q}_2$  & $\hat{\q}_1$ &  $\hat{\c}_3$ &  $\hat{\c}_2$    & $\alpha_c^{(r)}(\infty)$  \\ \hline\hline
$1$-sfl RDT                                      & $0.5$ & $0.5$ &  $0$  &  $0$  & $\rightarrow 1$  &  $0$   & $0$ & $\rightarrow 1$
 &  $0$  &  $\rightarrow 0$  & \bl{$\mathbf{2.4514}$} \\ \hline\hline
   $2$-sfl RDT                                      & $0.2110$  & $ 1.1847$   &  $0$ & $ 0.7547$ & $\rightarrow 1$  &  $0$  &  $0.5183$ & $\rightarrow 1$
 &  $0$  &  $3.1984$   & \bl{$\mathbf{2.3750}$}  \\ \hline\hline
   $3$-sfl RDT                                      & $0.1706$  & $1.4666$ &  $0.7389$ & $0.9657$ & $\rightarrow 1$ &  $0.5023$ &  $0.8256$ & $\rightarrow 1$
&  $2.8$ &  $6.9$   & \bl{$\mathbf{2.3744}$}  \\ \hline\hline
  \end{tabular}
  }
\label{tab:tab6}
\end{table}

\subsubsection{General $r$-th level of lifting}
\label{sec:rleverf}

General $r$ lifting considerations between  (\ref{eq:reluact82})-(\ref{eq:reluact87}) continue to hold with $\bar{p}_x(\cdot)$ as in (\ref{eq:erfact71}).

\subsection{$\tanh$ activations}
\label{sec:tanh}

We now consider as the neuronal functions in the hidden layer the following
\begin{equation}
\hspace{-1.5in} \mbox{\textbf{\bl{\tanh} activation:}} \hspace{1in} \f^{(2)}(\x)=\tanh\lp \x\rp.
\label{eq:tanhact1}
\end{equation}
As in previous sections, we start by looking at the first level of lifting and specialize the generic results obtained earlier to the $\tanh$ scenario of interest here.

\subsubsection{$r=1$ -- first level of lifting}
\label{sec:firstlevtanh}

We follow again the procedure outlined above on multiple occasions and start by rewriting  (\ref{eq:negprac20a0})  as
 \begin{eqnarray}\label{eq:tanhact0a0}
a_c^{(1)}(d)
& = &   \frac{1}{\mE_{{\mathcal U}_2^{(j_w)}}  \lp z_i^{(1)}\lp\g^{(2)};\f^{(2)}\lp \q^{(net)} \rp\rp\rp}
  =   \frac{1}{\mE_{\g^{(2)}}  \lp z_i^{(1)}\lp \bar{\g}^{(2)}; \tanh\lp\q^{(net)}\rp\rp\rp},
  \end{eqnarray}
where recalling on (\ref{eq:reluact12}) also gives
    \begin{eqnarray}\label{eq:tanhact12}
\lim_{d\rightarrow \infty}\mE_{\g^{(2)}} z_i^{(1)}\lp   \bar{\g}^{(2)};\f^{(2)}\lp \q^{(net)} \rp   \rp
 & \rightarrow &
\frac{\lim_{d\rightarrow \infty}\mE_{\g^{(2)}} \lp \max\lp - \f^{(2)}\lp \g^{(2)} \rp^T\w,0\rp \rp^2}{\lim_{d\rightarrow \infty}\mE_{\g^{(2)}} \left \|  \frac{d\f^{(2)}\lp \g^{(2)} \rp }{d\g^{(2)}} \right \|_2^2}.
 \end{eqnarray}

\noindent \red{\textbf{\emph{(i) Handling $\mE_{\g^{(2)}} \lp \max\lp - \f^{(2)}\lp \g^{(2)} \rp^T\w,0\rp \rp^2$:}}} Recalling further on (\ref{eq:reluact2}), we also have
 \begin{eqnarray}\label{eq:tanhact20}
\mE_{\g^{(2)}} \lp \max\lp - \f^{(2)}\lp \g^{(2)} \rp^T\w,0\rp \rp^2
=\mE_{g_c^{(2)}} \lp \max\lp g_c^{(2)},0\rp \rp^2=\frac{1}{2}\sigma_{2}^2,
\end{eqnarray}
which, together with (\ref{eq:tanhact1}), allows to write (analogously to   (\ref{eq:reluact21}), (\ref{eq:reluact22}), and (\ref{eq:reluact23}))
\begin{eqnarray}
\mE \lp \f^{(2)}\lp \g_{1}^{(2)}\rp \rp &= &\mE \tanh\lp\g_1^{(2)}\rp =0 \nonumber \\
 \mE \lp \f^{(2)}\lp \g_{1}^{(2)}\rp \rp^2 &= &\mE \lp \tanh\lp \g_1^{(2)}\rp\rp^2 \approx 0.3943,
\label{eq:tanhact22}
\end{eqnarray}
and
 \begin{eqnarray}\label{eq:tanhact23}
 \sigma_2^2=d\lp\mE \lp\f^{(2)}\lp \g_{1}^{(2)} \rp \rp^2- \lp \mE\lp \f^{(2)}\lp \g_{1}^{(2)}\rp\rp\rp^2\rp
  \approx 0.3943 d.
\end{eqnarray}
From  (\ref{eq:tanhact20}) and (\ref{eq:tanhact23}), one  also easily finds
 \begin{eqnarray}\label{eq:tanhact24}
\mE_{\g^{(2)}} \lp \max\lp - \f^{(2)}\lp \g^{(2)} \rp^T\w,0\rp \rp^2
 =\frac{1}{2}\sigma_{2}^2  \approx 0.1971 d.
\end{eqnarray}

\noindent \red{\textbf{\emph{(ii) Handling $\mE_{\g^{(2)}} \left \|  \frac{d\f^{(2)}\lp \g^{(2)} \rp }{d\g^{(2)}} \right \|_2^2$:}}} Keeping in mind (\ref{eq:tanhact1}), we find
 \begin{eqnarray}\label{eq:tanhact25}
\frac{d\f^{(2)}\lp \g_{j_w}^{(2)} \rp }{d\g_{j_w}^{(2)}}= \frac{d\tanh\lp \g_{j_w}^{(2)} \rp }{d\g_{j_w}^{(2)}}= 1-\tanh\lp\g_{j_w}^{(2)}\rp^2,
\end{eqnarray}
and
 \begin{eqnarray}\label{eq:tanhact26}
\mE_{\g^{(2)}} \left \|  \frac{d\f^{(2)}\lp \g^{(2)} \rp }{d\g^{(2)}} \right \|_2^2
=\mE_{\g^{(2)}} \sum_{j_w=1}^{d}\lp \frac{d\f^{(2)}\lp \g_{j_w}^{(2)} \rp }{d\g_{j_w}^{(2)}}\rp^2
\approx 0.4644 d.
\end{eqnarray}
A combination of (\ref{eq:tanhact12}), (\ref{eq:tanhact24}), and (\ref{eq:tanhact26}) gives
 \begin{eqnarray}\label{eq:tanhact27}
\lim_{d\rightarrow\infty}\mE_{\g^{(2)}} z_i^{(1)}\lp   \bar{\g}^{(2)};\f^{(2)}\lp \q^{(net)} \rp   \rp
 & =  &  \lim_{d\rightarrow\infty}\frac{\mE_{\g^{(2)}} \lp \max\lp - \f^{(2)}\lp \g^{(2)} \rp^T\w,0\rp \rp^2}{\mE_{\g^{(2)}} \left \|  \frac{d\f^{(2)}\lp \g^{(2)} \rp }{d\g^{(2)}} \right \|_2^2}
 =0.4244.\nonumber \\
 \end{eqnarray}
Utilizing (\ref{eq:negprac20a0}), (\ref{eq:tanhact0a0}), and (\ref{eq:tanhact27}), one then obtains
\begin{eqnarray}\label{eq:tanhact28}
\hspace{-0in}(\mbox{\bl{\textbf{first level:}}}) \qquad  a_c^{(1)}(\infty)
& = &    \lim_{d\rightarrow\infty} a_c^{(1)}(d)  =    \frac{1}{\lim_{d\rightarrow\infty} \mE_{{\mathcal U}_2^{(j_w)}}  \lp z_i^{(1)}\lp  \bar{\g}^{(2)};\f^{(2)}\lp \q^{(net)} \rp\rp\rp}  \nonumber \\
 & = &   \frac{1}{\lim_{d\rightarrow\infty} \mE_{\g^{(2)}}  \lp z_i^{(1)}\lp   \bar{\g}^{(2)}; \tanh\lp\q^{(net)}\rp\rp\rp}   \approx \bl{\mathbf{2.3556}}.
  \end{eqnarray}

We here again observe that the above capacity result exactly matches the corresponding one obtained using the statistical physics replica symmetry methods in \cite{ZavPeh21}.

\subsubsection{Second level of lifting}
\label{sec:secondleverf}

Relying on (\ref{eq:tanhact1}), we write analogously to (\ref{eq:reluact35})
\begin{eqnarray}\label{eq:tanhact35}
    \bar{\psi}_{rd}^{(d,2)}(\p,\q,\c,\gamma_{sq},\gamma_{sq}^{(p)})
  & = &  \frac{1}{2}
(1-\p_2\q_2)\c_2
 -  \gamma_{sq}^{(p)}
-\Bigg(\Bigg. -\frac{1}{2\c_2} \log \lp \frac{2\gamma_{sq}-\c_2(1-\q_2)}{2\gamma_{sq}} \rp  \nonumber \\
 & & +  \frac{\q_2}{2(2\gamma_{sq}-\c_2(1-\q_2))}   \Bigg.\Bigg)
 \nonumber \\
& &   + \gamma_{sq}
 -\alpha\frac{1}{\c_2}\mE_{{\mathcal U}_3^{(j_w)}} \log\lp \mE_{{\mathcal U}_2^{(j_w)}} e^{-\c_2\frac{z_i^{(2)}\lp \bar{\g}^{(3)}; \tanh\lp \q^{(net)}\rp \rp }{4\gamma_{sq}}}\rp,
    \end{eqnarray}
where, as in (\ref{eq:reluact52}),
\begin{eqnarray}\label{eq:tanhact52}
  z_i^{(2)}\lp   \bar{\g}^{(3)}; \tanh\lp \q^{(net)} \rp   \rp
   &\rightarrow  &  \lp \max\lp  \bar{b}_2g_{f}^{(3,1)}+\bar{b}_3g_{f}^{(3,2)},0\rp \rp^2,
 \end{eqnarray}
with $\bar{b}_2$ and $\bar{b}_3$ as in (\ref{eq:reluact53}). As usual, $g_f^{(3,1)}$ relates to the randomness of $\g^{(2)}$ (i.e., ${\mathcal U}_2$) and $g_f^{(3,2)}$ relates to the randomness of $\g^{(3)}$ (i.e., ${\mathcal U}_3$).

\noindent \red{\textbf{\emph{(i) Handling $\mE_{\bar{\g}^{(3)}} \left \|  \frac{d\f^{(2)}\lp \g^{(x,2)} \rp }{d\g^{(x,2)}} \right \|_2^2$:}}} Analogously to (\ref{eq:tanhact26})
\begin{eqnarray}\label{eq:tanhact53a0}
\mE_{\bar{\g}^{(3)}} \left \|  \frac{d\f^{(2)}\lp \g^{(x,2)} \rp }{d\g^{(x,2)}} \right \|_2^2=0.4644 d.
\end{eqnarray}

\noindent \red{\textbf{\emph{(ii) Further specializing to $\f^{(2)}\lp \q^{(net)}\rp=\tanh\lp \q^{(net)}\rp$:}}} We first write
\begin{eqnarray}\label{eq:tanhact54}
\bar{p}_1  &  \approx &  2.1533 \mE_{\g_1^{(3)}} \mE_{\g_1^{(2)}} \lp \f^{(2)}\lp \sum_{k=2}^{3}b_k \g_1^{(k)}\rp \rp^2=
  2.1533 \mE_{\g_1^{(3)}} \mE_{\g_1^{(2)}} \lp\tanh \lp \sum_{k=2}^{3}b_k \g_1^{(k)}\rp\rp^2 \nonumber \\
& \approx & 0.8490,
    \end{eqnarray}
and
\begin{eqnarray}\label{eq:tanhact55}
\bar{p}_3 \approx  2.1533 \lp \mE_{\g_1^{(3)}} \mE_{\g_1^{(2)}}  \f^{(2)}\lp \sum_{k=2}^{3}b_k \g_1^{(k)}\rp \rp^2=
    2.1533 \lp \mE_{\g_1^{(3)}} \mE_{\g_1^{(2)}}  \tanh  \lp \sum_{k=2}^{3}b_k \g_1^{(k)}\rp \rp^2
  =0. \nonumber \\
 \end{eqnarray}
Then one also has
\begin{eqnarray}\label{eq:tanhact56}
  \mE_{\g_1^{(2)}}  \f^{(2)}\lp \sum_{k=2}^{3}b_k \g_1^{(k)}\rp
&  = &
    \mE_{\g_1^{(2)}}   \tanh\lp \sum_{k=2}^{3}b_k \g_1^{(k)} \rp
  =
    \mE_{\g_1^{(2)}} \tanh \lp \sqrt{1-\p_2} \g_1^{(2)} + \sqrt{\p_2} \g_1^{(3)} \rp, \nonumber \\
   \end{eqnarray}
and
\begin{eqnarray}\label{eq:tanhact57}
\bar{p}_2 & \approx &  2.1533\mE_{\g_1^{(3)}} \lp \mE_{\g_1^{(2)}}  \f^{(2)}\lp \sum_{k=2}^{3}b_k \g_1^{(k)}\rp  \rp^2
= 2.1533 \mE_{\g_1^{(3)}} \lp   \mE_{\g_1^{(2)}} \tanh \lp \sqrt{1-\p_2} \g_1^{(2)} + \sqrt{\p_2} \g_1^{(3)} \rp  \rp^2. \nonumber \\
   \end{eqnarray}
One then has
\begin{eqnarray}\label{eq:tanhact58}
\bar{b}_2
& = &
\sqrt{\bar{p}_1-\bar{p}_2}
\nonumber \\
\bar{b}_3
& = &
\sqrt{\bar{p}_2-\bar{p}_3},
 \end{eqnarray}
where, for the $\tanh$ activations, $\bar{p}_1$, $\bar{p}_2$, and $\bar{p}_3$ are as in (\ref{eq:tanhact54}), (\ref{eq:tanhact55}), and (\ref{eq:tanhact57}), respectively.

Moreover, one can now easily rewrite (\ref{eq:tanhact35}) as
\begin{eqnarray}\label{eq:tanhact59}
    \bar{\psi}_{rd}^{(d,2)}(\p,\q,\c,\gamma_{sq},\gamma_{sq}^{(p)})
   & = &  \frac{1}{2}
(1-\p_2\q_2)\c_2
 -  \gamma_{sq}^{(p)}
-\Bigg(\Bigg. -\frac{1}{2\c_2} \log \lp \frac{2\gamma_{sq}-\c_2(1-\q_2)}{2\gamma_{sq}} \rp  \nonumber \\
 & & +  \frac{\q_2}{2(2\gamma_{sq}-\c_2(1-\q_2))}   \Bigg.\Bigg)
    + \gamma_{sq}
 -\alpha\frac{1}{\c_2}   \mE_{g_f^{(3,2)}} \log\lp f_{(zt)}^{(2,f)} \rp,
    \end{eqnarray}
with $f_{(zt)}^{(2,f)}$ as in (\ref{eq:reluact60}) and $\bar{p}_1$, $\bar{p}_2$, and $\bar{p}_3$ as in (\ref{eq:tanhact54}), (\ref{eq:tanhact55}), and (\ref{eq:tanhact57}), respectively. Solving the system in (\ref{eq:reluact62}) again gives the solution that satisfies the closed form relations from (\ref{eq:reluact63}). Taking the obtained concrete numerical parameters values and utilizing $\bar{\psi}_{rd}^{(d,2)}(\hat{\p},\hat{\q},\hat{\c},\hat{\gamma}_{sq},\hat{\gamma}_{sq}^{(p)})=0$, we obtain for
 \begin{equation}\label{eq:tanhact64}
\hspace{-2in}(\mbox{\bl{\textbf{\emph{full} second level:}}}) \qquad \qquad  a_c^{(2,f)}(\infty) =  \lim_{d\rightarrow\infty} a_c^{(2,f)}(d)  \approx \bl{\mathbf{2.3063}}.
  \end{equation}

\noindent \underline{\textbf{\emph{Concrete numerical values:}}} Following the common practice, we, in  Table \ref{tab:tab7}, again complement the above $a_c^{(2,f)}(\infty)$ with the concrete values of all the parameters relevant to the second \emph{full} (2-sfl RDT) level of lifting.  The corresponding  quantities for the first \emph{full} (1-sfl RDT) level are shown in the table as well. As was the case for the $\erf$ activation, the second partial level of lifting makes no progress and is therefore not included in the table.
\begin{table}[h]
\caption{$r$-sfl RDT parameters; \textbf{$\tanh$} activations -- \emph{wide} treelike net capacity;  $\hat{\c}_1\rightarrow 1$; $d\rightarrow\infty$; $n\rightarrow\infty$}\vspace{.1in}
\centering
\def\arraystretch{1.2}
\begin{tabular}{||l||c|c||c|c||c|c||c||c||}\hline\hline
 \hspace{-0in}$r$-sfl RDT                                             & $\hat{\gamma}_{sq}$    & $\hat{\gamma}_{sq}^{(p)}$    &  $\hat{\p}_2$ & $\hat{\p}_1$     & $\hat{\q}_2$  & $\hat{\q}_1$ &  $\hat{\c}_2$    & $\alpha_c^{(r)}(\infty)$  \\ \hline\hline
$1$-sfl RDT                                      & $0.5$ & $0.5$ &  $0$  & $\rightarrow 1$   & $0$ & $\rightarrow 1$
 &  $\rightarrow 0$  & \bl{$\mathbf{2.3556}$} \\ \hline\hline
   $2$-sfl RDT                                      & $0.2235$  & $ 1.1186$  & $ 0.7855$ & $\rightarrow 1$ &  $0.5857$ & $\rightarrow 1$
 &  $3.3157$   & \bl{$\mathbf{2.3063}$}  \\ \hline\hline
  \end{tabular}
\label{tab:tab7}
\end{table}

\subsubsection{$r=3$ -- third level of lifting}
\label{sec:thirdleverf}

Analogously to  (\ref{eq:reluact65}) (and ultimately (\ref{eq:negprac19}) and (\ref{eq:tanhact35})), we first write
{\small\begin{align}\label{eq:tanhact65}
    \bar{\psi}_{rd}^{(d,3)}(\p,\q,\c,\gamma_{sq},\gamma_{sq}^{(p)})
  & =   \frac{1}{2}
(1-\p_2\q_2)\c_2+ \frac{1}{2}
(\p_2\q_2-\p_3\q_3)\c_3 -  \gamma_{sq}^{(p)} \nonumber \\
&\quad
-\Bigg(\Bigg. -\frac{1}{2\c_2} \log \lp \frac{2\gamma_{sq}^{(p)}-\c_2(1-\q_2)}{2\gamma_{sq}^{(p)}} \rp  -\frac{1}{2\c_3} \log \lp \frac{2\gamma_{sq}^{(p)}-\c_2(1-\q_2)-\c_3(\q_2-\q_3)}{2\gamma_{sq}^{(p)}-\c_2(1-\q_2)} \rp  \nonumber \\
& \quad +  \frac{\q_3}{2(2\gamma_{sq}^{(p)}-\c_2(1-\q_2)-\c_3(\q_2-\q_3))}   \Bigg.\Bigg)
 \nonumber \\
& \quad   + \gamma_{sq}
 -\frac{\alpha}{\c_3}\mE_{{\mathcal U}_4^{(j_w)}} \log\lp \mE_{{\mathcal U}_3^{(j_w)}} \lp \mE_{{\mathcal U}_2^{(j_w)}} e^{-\c_2\frac{z_i^{(3)}\lp \bar{\g}^{(4)}; \tanh\lp \q^{(net)}\rp\rp}{4\gamma_{sq}}}\rp^{\frac{\c_3}{\c_2}}\rp,
    \end{align}}
 where
 \begin{eqnarray}\label{eq:tanhact66}
  z_i^{(3)}\lp   \bar{\g}^{(4)};\f^{(2)}\lp \q^{(net)} \rp   \rp
 =
  z_i^{(3)}\lp   \bar{\g}^{(4)};\tanh\lp \q^{(net)} \rp   \rp
    &\rightarrow  &  \lp \max\lp  \bar{b}_2g_{f}^{(4,1)}+\bar{b}_3g_{f}^{(4,2)}+\bar{b}_4g_{f}^{(4,3)},0\rp \rp^2,\nonumber \\
 \end{eqnarray}
with $\bar{b}_2$, $\bar{b}_3$, and $\bar{b}_4$  as in (\ref{eq:reluact67}), and $b_2=\sqrt{1-\p_2}$, $b_3=\sqrt{\p_2-\p_3}$, and $b_4=\sqrt{\p_3}$ (as earlier, $g_f^{(4,1)}$ relates to the randomness of $\g^{(2)}$ (i.e., ${\mathcal U}_2$),
$g_f^{(4,2)}$ to the randomness of $\g^{(3)}$ (i.e., ${\mathcal U}_3$), and $g_f^{(4,3)}$ to the randomness of $\g^{(4)}$ (i.e., ${\mathcal U}_4$)).

\noindent \red{\textbf{\emph{(i) Handling $\mE_{\bar{\g}^{(4)}} \left \|  \frac{d\f^{(2)}\lp \g^{(x,3)} \rp }{d\g^{(x,3)}} \right \|_2^2$:}}} As in (\ref{eq:tanhact53a0}),
\begin{eqnarray}\label{eq:tanhact65a0}
\mE_{\bar{\g}^{(4)}} \left \|  \frac{d\f^{(2)}\lp \g^{(x,3)} \rp }{d\g^{(x,3)}} \right \|_2^2=
\mE_{\bar{\g}^{(4)}} \left \|  \frac{d\lp \g^{(x,3)} \rp^2 }{d\g^{(x,3)}} \right \|_2^2=0.4644 d.
\end{eqnarray}

\noindent \red{\textbf{\emph{(ii) Further specializing to $\f^{(2)}\lp \q^{(net)}\rp=\tanh\lp \q^{(net)}\rp$:}}} We start by observing
\begin{eqnarray}\label{eq:tanhact68}
\bar{p}_1 & \approx & 2.1533 \mE_{\g_1^{(4)}} \mE_{\g_1^{(3)}} \mE_{\g_1^{(2)}} \lp \f^{(2)}\lp \sum_{k=2}^{4}b_k \g_1^{(k)}\rp \rp^2=
2.1533 \mE_{\g_1^{(4)}}  \mE_{\g_1^{(3)}} \mE_{\g_1^{(2)}} \lp \tanh\lp \sum_{k=2}^{4}b_k \g_1^{(k)}\rp\rp^2 \nonumber \\
& = & 0.8490,\nonumber \\
    \end{eqnarray}
and
\begin{eqnarray}\label{eq:tanhact69}
\bar{p}_4 & \approx & 2.1533 \lp \mE_{\g_1^{(4)}}  \mE_{\g_1^{(3)}} \mE_{\g_1^{(2)}}  \f^{(2)}\lp \sum_{k=2}^{4}b_k \g_1^{(k)}\rp \rp^2 =
   2.1533 \lp \mE_{\g_1^{(4)}}  \mE_{\g_1^{(3)}} \mE_{\g_1^{(2)}}   \tanh\lp\sum_{k=2}^{4}b_k \g_1^{(k)}\rp \rp^2= 0. \nonumber \\
 \end{eqnarray}
Then one also has
\begin{eqnarray}\label{eq:tanhact70}
  \mE_{\g_1^{(3)}}  \mE_{\g_1^{(2)}}  \f^{(2)}\lp \sum_{k=2}^{4}b_k \g_1^{(k)}\rp
&  = &
  \mE_{\g_1^{(3)}}     \mE_{\g_1^{(2)}}   \tanh\lp \sum_{k=2}^{4}b_k \g_1^{(k)}\rp \nonumber \\
&  = &
    \mE_{\g_1^{(3)}}    \mE_{\g_1^{(2)}}   \tanh\lp \sqrt{1-\p_2} \g_1^{(2)} + \sqrt{\p_2-\p_3} \g_1^{(3)}+ \sqrt{\p_3} \g_1^{(4)} \rp,
  \end{eqnarray}
and
\begin{eqnarray}\label{eq:tanhact71}
\bar{p}_3  & \approx &  2.1533 \mE_{\g_1^{(4)}} \lp \mE_{\g_1^{(3)}} \mE_{\g_1^{(2)}}  \f^{(2)}\lp \sum_{k=2}^{4}b_k \g_1^{(k)}\rp  \rp^2
\nonumber \\
& = &
2.1533 \mE_{\g_1^{(4)}} \lp     \mE_{\g_1^{(3)}}    \mE_{\g_1^{(2)}}   \tanh\lp \sqrt{1-\p_2} \g_1^{(2)} + \sqrt{\p_2-\p_3} \g_1^{(3)}+ \sqrt{\p_3} \g_1^{(4)} \rp \rp^2
  \triangleq  \bar{p}_x (\p_3).
  \end{eqnarray}
One then immediately also has
\begin{eqnarray}\label{eq:tanhact72}
     \mE_{\g_1^{(2)}}  \f^{(2)}\lp \sum_{k=2}^{4}b_k \g_1^{(k)}\rp
&  = &
      \mE_{\g_1^{(2)}} \tanh  \lp \sum_{k=2}^{4}b_k \g_1^{(k)} \rp \nonumber \\
&  = &
      \mE_{\g_1^{(2)}}  \tanh \lp \sqrt{1-\p_2} \g_1^{(2)} + \sqrt{\p_2-\p_3} \g_1^{(3)}+ \sqrt{\p_3} \g_1^{(4)} \rp, \nonumber \\
  \end{eqnarray}
and
\begin{eqnarray}\label{eq:tanhact73}
\bar{p}_2 & \approx & 2.1533 \mE_{\g_1^{(4)}}  \mE_{\g_1^{(3)}} \lp \mE_{\g_1^{(2)}}  \f^{(2)}\lp \sum_{k=2}^{4}b_k \g_1^{(k)}\rp  \rp^2  =
\bar{p}_x(\p_2).
  \end{eqnarray}
The above then gives
\begin{eqnarray}\label{eq:tanhact74}
\bar{b}_2
& = &
\sqrt{\bar{p}_1-\bar{p}_2}
\nonumber \\
\bar{b}_3
& = &
\sqrt{\bar{p}_2-\bar{p}_3}
\nonumber \\
\bar{b}_4
& = &
\sqrt{\bar{p}_3-\bar{p}_4},
 \end{eqnarray}
where $\bar{p}_1$, $\bar{p}_2$, $\bar{p}_3$, and $\bar{p}_4$ are as in (\ref{eq:tanhact68}), (\ref{eq:tanhact69}), (\ref{eq:tanhact71}), and (\ref{eq:tanhact73}), respectively.

Analogously to (\ref{eq:reluact78}), we then write
\begin{align}\label{eq:tanhact78}
    \bar{\psi}_{rd}(\p,\q,\c,\gamma_{sq},\gamma_{sq}^{(p)})
 & =    \frac{1}{2}
(1-\p_2\q_2)\c_2+ \frac{1}{2}
(\p_2\q_2-\p_3\q_3)\c_3 -  \gamma_{sq}^{(p)}
-\Bigg(\Bigg. -\frac{1}{2\c_2} \log \lp \frac{2\gamma_{sq}^{(p)}-\c_2(1-\q_2)}{2\gamma_{sq}^{(p)}} \rp
 \nonumber \\
& \quad -\frac{1}{2\c_3} \log \lp \frac{2\gamma_{sq}^{(p)}-\c_2(1-\q_2)-\c_3(\q_2-\q_3)}{2\gamma_{sq}^{(p)}-\c_2(1-\q_2)} \rp
 \nonumber \\
& \quad +  \frac{\q_3}{2(2\gamma_{sq}^{(p)}-\c_2(1-\q_2)-\c_3(\q_2-\q_3))}   \Bigg.\Bigg)
 + \gamma_{sq}
 -\frac{\alpha}{\c_3}   \mE_{g_{f}^{(4,3)}} \log\lp \mE_{g_{f}^{(4,2)}} \lp f_{(zt)}^{(3,f)} \rp^{\frac{\c_3}{\c_2}}\rp,
    \end{align}
where $f_{(zt)}^{(3,f)} $ is as in (\ref{eq:reluact76}) with $\bar{b}_2$, $\bar{b}_3$, and $\bar{b}_4$ as in (\ref{eq:tanhact74}). Utilizing the closed form relations (\ref{eq:reluact80}), one solves the system in (\ref{eq:reluact79})  and denotes the concrete numerical values for all the considered parameters by $\hat{\p},\hat{\q},\hat{\c},\hat{\gamma}_{sq},\hat{\gamma}_{sq}^{(p)}$. Plugging these values in  $\bar{\psi}_{rd}^{(d,3)}(\hat{\p},\hat{\q},\hat{\c},\hat{\gamma}_{sq},\hat{\gamma}_{sq}^{(p)})=0$, allows to determine
 \begin{equation}\label{eq:tanhact81}
\hspace{-2in}(\mbox{\bl{\textbf{\emph{full} third level:}}}) \qquad \qquad  a_c^{(3,f)}(\infty) =  \lim_{d\rightarrow\infty} a_c^{(3,f)}(d)  \approx \bl{\mathbf{2.3058}}.
  \end{equation}

\noindent \underline{\textbf{\emph{Concrete numerical values:}}}  In  Table \ref{tab:tab8},  all the relevant concrete parameters values related to the third \emph{full} (3-sfl RDT) level of lifting are added as complements to the above $a_c^{(3,f)}(\infty)$. The corresponding quantities for the first and second \emph{full} (1-sfl RDT and 2-sfl RDT) levels  are shown in parallel as well.
\begin{table}[h]
\caption{$r$-sfl RDT parameters; \textbf{$\tanh$} activations  -- \emph{wide} treelike net capacity;  $\hat{\c}_1\rightarrow 1$; $d\rightarrow\infty$; $n\rightarrow\infty$}\vspace{.1in}
\centering
\def\arraystretch{1.2}
{\small
\begin{tabular}{||l||c|c||c|c|c||c|c|c||c|c||c||}\hline\hline
 \hspace{-0in}$r$-sfl RDT                                             & $\hat{\gamma}_{sq}$    & $\hat{\gamma}_{sq}^{(p)}$    &  $\hat{\p}_3$  &  $\hat{\p}_2$ & $\hat{\p}_1$    &  $\hat{\q}_3$   & $\hat{\q}_2$  & $\hat{\q}_1$ &  $\hat{\c}_3$ &  $\hat{\c}_2$    & $\alpha_c^{(r)}(\infty)$  \\ \hline\hline
$1$-sfl RDT                                      & $0.5$ & $0.5$ &  $0$  &  $0$  & $\rightarrow 1$  &  $0$   & $0$ & $\rightarrow 1$
 &  $0$  &  $\rightarrow 0$  & \bl{$\mathbf{2.3556}$} \\ \hline\hline
   $2$-sfl RDT                                     & $0.2235$  & $ 1.1186$   &  $0$ & $ 0.7855$ & $\rightarrow 1$  &  $0$  &  $0.5857$ & $\rightarrow 1$
 &  $0$  &  $3.3157$   & \bl{$\mathbf{2.3063}$}  \\ \hline\hline
   $3$-sfl RDT                                      & $0.1824$  & $1.3711$ &  $0.7685$ & $0.9660$ & $\rightarrow 1$ &  $0.5682$ &  $0.8479$ & $\rightarrow 1$
&  $2.8$ &  $7.2$   & \bl{$\mathbf{2.3058}$}  \\ \hline\hline
  \end{tabular}
  }
\label{tab:tab8}
\end{table}

\subsubsection{General $r$-th level of lifting}
\label{sec:rleverf}

As earlier, general $r$ lifting considerations between  (\ref{eq:reluact82})-(\ref{eq:reluact87}) are applicable again with $\bar{p}_x(\cdot)$ as in (\ref{eq:tanhact71}).

\subsection{Summary}
\label{sec:summary}

In  Table \ref{tab:tab9},  we summarize the above results obtained for different hidden layer activations.

\begin{table}[h]
\caption{\emph{Wide} treelike net capacity -- lifting mechanism progress for different activations;   $d\rightarrow\infty$; $n\rightarrow\infty$}\vspace{.1in}
\centering
\def\arraystretch{1.2}
{\small
\begin{tabular}{c||c|c|c|c}\hline\hline
  \hspace{.2in}   \textbf{Memory capacity} \hspace{.2in}   &  \multicolumn{4}{c}{\textbf{Activation}}                   \\
    \cline{2-5}
 \hspace{.0in} ($r$ -- lifting level)  &  ReLU & Quad & $\erf$ & $\tanh$  \\  \hline\hline
\hspace{.0in}$\alpha_c^{(1,f)}(\infty)$  ($1$-sfl RDT) \hspace{.0in}   & $\quad \mathbf{2.9339}\quad $ & $\quad \mathbf{4}\quad $ &  $\quad \mathbf{ 2.4514}\quad $  &  $\quad \mathbf{ 2.3556}\quad $    \\ \hline\hline
\hspace{.0in}$\alpha_c^{(2,f)}(\infty)$  ($2$-sfl RDT)    & $\mathbf{2.6643}$  & $\mathbf{3.3750}$   &  $\mathbf{2.3750}$ & $\mathbf{2.3063}$   \\ \hline\hline
\hspace{.0in}$\alpha_c^{(3,f)}(\infty)$  ($3$-sfl RDT)    & $\mathbf{2.6534}$  & $\mathbf{3.3669}$ &  $\mathbf{2.3744}$ & $\mathbf{2.3058}$   \\ \hline\hline
  \end{tabular}
  }
\label{tab:tab9}
\end{table}

The results from Table \ref{tab:tab9} are also visualized in Figure \ref{fig:fig4}. The benefit of the lifted RDT is fairly strong. Moreover, the lifting mechanism converges rather rapidly with relative improvement no better than $\sim 0.1\%$ achieved already on the third level.

\begin{figure}[h]
\centering
\centerline{\includegraphics[width=1\linewidth]{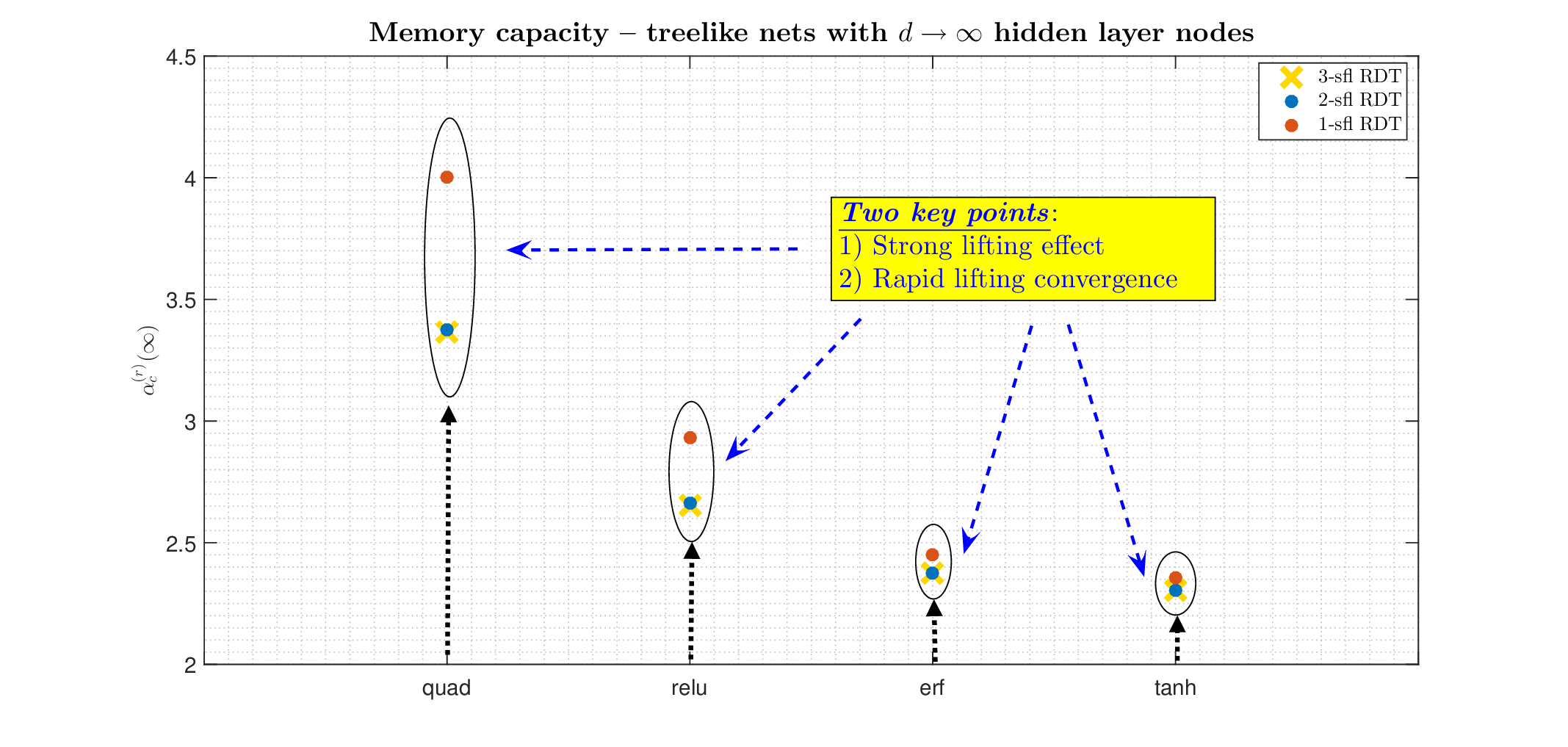}}
\caption{Memory capacity -- treelike nets with $d\rightarrow\infty$ hidden layer neurons; different activations}
\label{fig:fig4}
\end{figure}

\section{Conclusion}
\label{sec:conc}

We studied the memory capacity of the \emph{wide} treelike committee machines (TCM) neural networks with generic hidden layer activations. \cite{Stojnictcmspnncaprdt23,Stojnictcmspnncapliftedrdt23} recently showed that the Random Duality Theory (RDT) and its a \emph{partially lifted} (pl RDT) variant can be used to create very powerful frameworks for precise networks capacity analysis. The \emph{sign} activations considerations from \cite{Stojnictcmspnncaprdt23,Stojnictcmspnncapliftedrdt23} were then extended  to more general ones in \cite{Stojnictcmspnncapdiffactrdt23}, where particularly elegant results were uncovered for \emph{any} even number of the \emph{quadratically} and \emph{ReLU} activated hidden layer neurons, $d$. The machinery of  \cite{Stojnictcmspnncapdiffactrdt23} was designed to work for any type of activations. However, it often requires a significant amount of numerical work to become practically usable. Here, we consider very popular \emph{wide} hidden layer networks and uncover that a strong portion of the numerical difficulties magically disappears.

In particular, we employ recently developed \emph{fully lifted} (fl) RDT to characterize the \emph{wide} ($d\rightarrow \infty$) TCM nets capacity. We obtain explicit, closed form, capacity characterizations for generic hidden layer activations functions. Even though the obtained forms allow  significantly less computationally intensive evaluations, they ultimately still requite one to perform  a substantial residual numerical work to ensure that the whole machinery can indeed be made practically relevant and useful. We successfully conducted all of it. Moreover,  for four very famous activations, ReLU, quadratic, erf, and tanh, we uncovered that the whole lifting mechanism exhibits a remarkably fast convergence with the relative improvements no better than $\sim 0.1\%$ happening already on the 3-rd level of lifting. As an additional bonus, we also observe that the capacity characterizations obtained on the first and second level of lifting exactly match those obtained through the statistical physics replica theory methods in \cite{ZavPeh21} for the generic and in \cite{BalMalZech19} for the ReLU activations.

Various extensions are possible as well. We here discussed in details some of the most famous NN activations functions. Depending on the context of application, many other well known ones are of interest and can be handled as well. Also, more complex network architectures including both  multi-layered TCM and  FCM or PM based ones are as relevant. We will discuss all of these in detail in separate papers.

\begin{singlespace}
\bibliographystyle{plain}
\bibliography{nflgscompyxRefs}

\begin{thebibliography}{10}

\bibitem{ADHLW19}
S.~Arora, S.~S. Du, W.~Hu, Z.~Li, and R.~Wang.
\newblock Fine-grained analysis of optimiza\-tion and generalization for
  overparameterized two-layer neural networks.
\newblock 2019.
\newblock available online at
  {\small\bl{\url{http://arxiv.org/abs/1901.08584}}}.

\bibitem{BalMalZech19}
C.~Baldassi, E.~M. Malatesta, and R.~Zecchina.
\newblock Properties of the geometry of solutions and capacity of multilayer
  neural networks with rectified linear unit activations.
\newblock {\em Phys. Rev. Lett.}, 123:170602, October 2019.

\bibitem{BalVen87}
P.~Baldi and S.~Venkatesh.
\newblock Number od stable points for spin-glasses and neural networks of
  higher orders.
\newblock {\em Phys. Rev. Letters}, 58(9):913--916, Mar. 1987.

\bibitem{BHK90}
E.~Barkai, D.~Hansel, and I.~Kanter.
\newblock Statistical mechanics of a multilayered neural network.
\newblock {\em Phys. Rev. Lett.}, 65(18):2312--2315, Oct 1990.

\bibitem{BHS92}
E.~Barkai, D.~Hansel, and H.~Sompolinsky.
\newblock Broken symmetries in multilayered perceptrons.
\newblock {\em Phys. Rev. A}, 45(6):4146, March 1992.

\bibitem{BarKan91}
E.~Barkai and I.~Kanter.
\newblock Storage capacity of a multilayer neural network with binary weights.
\newblock {\em Europhys. Lett.}, 14(2):107, 1991.

\bibitem{Baum88}
E.~B. Baum.
\newblock On the capabilities of multilayer perceptrons.
\newblock {\em Journal of complexity}, 4(3):193--215, 1988.

\bibitem{Cameron60}
S.~H. Cameron.
\newblock Tech-report 60-600.
\newblock {\em Proceedings of the bionics symposium}, pages 197--212, 1960.
\newblock Wright air development division, {D}ayton, {O}hio.

\bibitem{Cover65}
T.~Cover.
\newblock Geomretrical and statistical properties of systems of linear
  inequalities with applications in pattern recognition.
\newblock {\em IEEE Transactions on Electronic Computers}, (EC-14):326--334,
  1965.

\bibitem{DT}
D.~Donoho and J.~Tanner.
\newblock Neighborliness of randomly-projected simplices in high dimensions.
\newblock {\em Proc. National Academy of Sciences}, 102(27):9452--9457, 2005.

\bibitem{DonTan09Univ}
D.~Donoho and J.~Tanner.
\newblock Observed universality of phase transitions in high-dimensional
  geometry, with implications for modern data analysis and signal processing.
\newblock {\em Phylosophical transactions of the royal society A: mathematical,
  physical and engineering sciences}, 367, November 2009.

\bibitem{DTbern}
D.~Donoho and J.~Tanner.
\newblock Counting the face of randomly projected hypercubes and orthants, with
  application.
\newblock {\em Discrete and Computational Geometry}, 43:522--541, 2010.

\bibitem{DuZhaiPoc18}
S.~S. Du, X.~Zhai, B.~Poczos, and A.~Singh.
\newblock Gradient descent provably optimizes over\-parameterized neural
  networks.
\newblock 2018.
\newblock available online at
  {\small\bl{\url{http://arxiv.org/abs/1810.02054}}}.

\bibitem{EKTVZ92}
A.~Engel, H.~M. Kohler, F.~Tschepke, H.~Vollmayr, and A.~Zippelius.
\newblock Storage capacity and learning algorithms for two-layer neural
  networks.
\newblock {\em Phys. Rev. A}, 45(10):7590, May 1992.

\bibitem{MitchDurb89}
R.~M.~Durbin G.~J.~Mitchison.
\newblock Bounds on the learning capacity of some multi-layer networks.
\newblock {\em Biological Cybernetics}, 60:345--365, 1989.

\bibitem{Gar88}
E.~Gardner.
\newblock The space of interactions in neural networks models.
\newblock {\em J. Phys. A: Math. Gen.}, 21:257--270, 1988.

\bibitem{GarDer88}
E.~Gardner and B.~Derrida.
\newblock Optimal storage properties of neural networks models.
\newblock {\em J. Phys. A: Math. Gen.}, 21:271--284, 1988.

\bibitem{GeWangZhao19}
R.~Ge, R.~Wang, and H.~Zhao.
\newblock Mildly overparametrized neural nets can memorize training data
  efficiently.
\newblock 2019.
\newblock available online at
  {\small\bl{\url{http://arxiv.org/abs/1909.11837}}}.

\bibitem{HardrtMa16}
M.~Hardt and T.~Ma.
\newblock Identity matters in deep learning.
\newblock 2016.
\newblock available online at
  {\small\bl{\url{http://arxiv.org/abs/1611.04231}}}.

\bibitem{GBHuang03}
G.~B. Huang.
\newblock Learning capability and storage capacity of two-hidden-layer
  feedforward networks.
\newblock {\em IEEE Transactions on Neural Networks}, 14(2):274--281, 2003.

\bibitem{JiTel19}
Z.~Ji and M.~Telgarsky.
\newblock Polylogarithmic width suffices for gradient descent to achieve
  arbitrarily small test error with shallow relu networks.
\newblock 2019.
\newblock available online at
  {\small\bl{\url{http://arxiv.org/abs/1909.12292}}}.

\bibitem{Joseph60}
R.~D. Joseph.
\newblock The number of orthants in $n$-space instersected by an
  $s$-dimensional subspace.
\newblock {\em Tech. memo 8, project {PARA}}, 1960.
\newblock Cornel aeronautical lab., Buffalo, N.Y.

\bibitem{LiLiang18}
Y.~Li and Y.~Liang.
\newblock Learning overparameterized neural networks via stochastic gradient
  descent on structured data.
\newblock {\em In Advances in Neural Information Processing Systems}, pages
  8157--8166, 2018.

\bibitem{MonZech95}
R.~Monasson and R.~Zecchina.
\newblock Weight space structure and internal representations: A direct
  approach to learning and generalization in multilayer neural networks.
\newblock {\em Phys. Rev. Lett.}, 75:2432, September 1995.

\bibitem{OymSol19}
S.~Oymak and M.~Soltanolkotabi.
\newblock Towards moderate overparameterization: global convergence guarantees
  for training shallow neural networks.
\newblock 2019.
\newblock available online at
  {\small\bl{\url{http://arxiv.org/abs/1902.04674}}}.

\bibitem{Schlafli}
L.~Schlafli.
\newblock {\em Gesammelte Mathematische AbhandLungen I}.
\newblock Basel, Switzerland: Verlag Birkhauser, 1950.

\bibitem{SongYang19}
Z.~Song and X.~Yang.
\newblock Quadratic suffices for over-parametrization via matrix {C}hernoff
  bound.
\newblock 2019.
\newblock available online at
  {\small\bl{\url{http://arxiv.org/abs/1906.03593}}}.

\bibitem{StojnicISIT2010binary}
M.~Stojnic.
\newblock Recovery thresholds for $\ell_1$ optimization in binary compressed
  sensing.
\newblock {\em ISIT, IEEE International Symposium on Information Theory}, pages
  1593 -- 1597, 13-18 June 2010.
\newblock Austin, TX.

\bibitem{StojnicGardGen13}
M.~Stojnic.
\newblock Another look at the {G}ardner problem.
\newblock 2013.
\newblock available online at \bl{\url{http://arxiv.org/abs/1306.3979}}.

\bibitem{StojnicLiftStrSec13}
M.~Stojnic.
\newblock Lifting $\ell_1$-optimization strong and sectional thresholds.
\newblock 2013.
\newblock available online at \bl{\url{http://arxiv.org/abs/1306.3770}}.

\bibitem{StojnicMoreSophHopBnds10}
M.~Stojnic.
\newblock Lifting/lowering {H}opfield models ground state energies.
\newblock 2013.
\newblock available online at \bl{\url{http://arxiv.org/abs/1306.3975}}.

\bibitem{StojnicGardSphNeg13}
M.~Stojnic.
\newblock Negative spherical perceptron.
\newblock 2013.
\newblock available online at \bl{\url{http://arxiv.org/abs/1306.3980}}.

\bibitem{StojnicGardSphErr13}
M.~Stojnic.
\newblock Spherical perceptron as a storage memory with limited errors.
\newblock 2013.
\newblock available online at \bl{\url{http://arxiv.org/abs/1306.3809}}.

\bibitem{Stojnicsflgscompyx23}
M.~Stojnic.
\newblock Bilinearly indexed random processes -- {\emph{stationarization}} of
  fully lifted interpolation.
\newblock 2023.
\newblock available online at \bl{\url{http://arxiv.org/abs/2311.18097}}.

\bibitem{Stojnicbinperflrdt23}
M.~Stojnic.
\newblock Binary perceptrons capacity via fully lifted random duality theory.
\newblock 2023.
\newblock available online at \bl{\url{http://arxiv.org/abs/2312.00073}}.

\bibitem{Stojnictcmspnncaprdt23}
M.~Stojnic.
\newblock Capacity of the treelike sign perceptrons neural networks with one
  hidden layer -- rdt based upper bounds.
\newblock 2023.
\newblock available online at \bl{\url{http://arxiv.org/abs/2312.08244}}.

\bibitem{Stojnicnegsphflrdt23}
M.~Stojnic.
\newblock Fl rdt based ultimate lowering of the negative spherical perceptron
  capacity.
\newblock 2023.
\newblock available online at \bl{\url{http://arxiv.org/abs/2312.16531}}.

\bibitem{Stojnicnflgscompyx23}
M.~Stojnic.
\newblock Fully lifted interpolating comparisons of bilinearly indexed random
  processes.
\newblock 2023.
\newblock available online at \bl{\url{http://arxiv.org/abs/2311.18092}}.

\bibitem{Stojnicflrdt23}
M.~Stojnic.
\newblock Fully lifted random duality theory.
\newblock 2023.
\newblock available online at \bl{\url{http://arxiv.org/abs/2312.00070}}.

\bibitem{Stojnictcmspnncapliftedrdt23}
M.~Stojnic.
\newblock {\emph{Lifted}} rdt based capacity analysis of the 1-hidden layer
  treelike \emph{sign} perceptrons neural networks.
\newblock 2023.
\newblock available online at \bl{\url{http://arxiv.org/abs/2312.08257}}.

\bibitem{Stojnichopflrdt23}
M.~Stojnic.
\newblock Studying {H}opfield models via fully lifted random duality theory.
\newblock 2023.
\newblock available online at \bl{\url{http://arxiv.org/abs/2312.00071}}.

\bibitem{Stojnictcmspnncapdiffactrdt23}
M.~Stojnic.
\newblock Fixed width treelike neural networks capacity analysis -- generic
  activations.
\newblock 2024.
\newblock available online at arxiv.

\bibitem{RuoyuSun19}
R.~Sun.
\newblock Optimization for deep learning: theory and algorithms.
\newblock 2019.
\newblock available online at
  {\small\bl{\url{http://arxiv.org/abs/1912.08957}}}.

\bibitem{Urban97}
R~Urbanczik.
\newblock Storage capacity of the fully-connected committee machine.
\newblock {\em J. Phys. A: Math. Gen.}, 30, 1997.

\bibitem{Ven86}
S.~Venkatesh.
\newblock Epsilon capacity of neural networks.
\newblock {\em Proc. Conf. on Neural Networks for Computing, Snowbird, UT},
  1986.

\bibitem{Vershynin20}
R.~Vershynin.
\newblock Memory capacity of neural networks with threshold and {R}e{LU}
  activations.
\newblock 2019.
\newblock available online at
  {\small\bl{\url{http://arxiv.org/abs/2001.06938}}}.

\bibitem{Wendel62}
J.~G. Wendel.
\newblock A problem in geometric probability.
\newblock {\em Mathematica Scandinavica}, 1:109--111, 1962.

\bibitem{Winder61}
R.~O. Winder.
\newblock Single stage threshold logic.
\newblock {\em Switching circuit theory and logical design}, pages 321--332,
  Sep. 1961.
\newblock AIEE Special publications S-134.

\bibitem{Winder}
R.~O. Winder.
\newblock {\em Threshold logic}.
\newblock Ph. D. dissertation, Princetoin University, 1962.

\bibitem{XiongKwonOh97}
Y.~Xiong and J.~H.~Oh C.~Kwon.
\newblock The storage capacity of a fully-connected committee machine.
\newblock {\em NIPS}, 1997.

\bibitem{Yama93}
M.~Yamasaki.
\newblock The lower bound of the capacity for a neural network with multiple
  hidden layers.
\newblock {\em In International Conference on Artificial Neural Networks},
  pages 546--549, 1993.

\bibitem{YunSuJad19}
C.~Yun, S.~Sra, and A.~Jadbabaie.
\newblock Small relu networks are powerful memorizers: a tight analysis of
  memorization capacity.
\newblock {\em In Advances in Neural Information Processing Systems}, pages
  15532--15543, 2019.

\bibitem{ZavPeh21}
J.~A. Zavatone-Veth and C.~Pehlevan.
\newblock Activation function dependence of the storage capacity of treelike
  neural networks.
\newblock {\em Phys. Rev. E}, 103:L020301, February 2021.

\bibitem{ZBHRV17}
C.~Zhang, S.~Bengio, M.~Hardt, B.~Recht, and O.~Vinyals.
\newblock Understanding deep learning requires rethinking generalization.
\newblock {\em ICLR}, 2017.

\bibitem{ZCZG18}
D.~Zou, Y.~Cao, D.~Zhou, and Q.~Gu.
\newblock Stochastic gradient descent optimizes over\-parameterized deep relu
  networks.
\newblock 2018.
\newblock available online at
  {\small\bl{\url{http://arxiv.org/abs/1811.08888}}}.

\end{thebibliography}
\end{singlespace}

\end{document}